\documentclass[accepted]{article}

\usepackage[utf8]{inputenc}
\usepackage[english]{babel}

\usepackage{bbm}
\usepackage{amsthm}

\usepackage{thmtools} 
\usepackage{thm-restate}

\usepackage{makecell}

\usepackage{amssymb,amsmath,balance}

\usepackage{amsfonts}
\usepackage{arydshln}

\usepackage{mathtools} 
\usepackage{nicefrac}

\usepackage{microtype}
\usepackage{graphicx}
\usepackage{subfigure}
\usepackage{booktabs} 
\RequirePackage{mathrsfs}
\RequirePackage{mdframed}

\def\ii{\mathbbm{j}}
\def\hsamp#1#2{\hvarphi^{(#1)}_{#2}}
\def\tsamp#1#2{\tvarphi^{(#1)}_{#2}}
\def\samp#1#2{\varphi^{(#1)}_{#2}}
\def\hftr#1#2{\wh{F}^{\textup{tr}}(\wh{X}^{(#1)}_{#2})}
\def\ftr#1#2{\wh{F}^{\textup{tr}}({X}^{(#1)}_{#2})}

\def\covmatparam{\bSigma^{\natural}}
\def\covmatnonparam{\bSigma^{\natural}}
\def\precmatparam{\bTheta^{\natural}}
\def\precmatnonparam{\bTheta^{\natural}}

\usepackage{hyperref}

\usepackage{enumitem}

\newtheorem{theorem}{Theorem}
\newtheorem{lemma}{Lemma}
\newtheorem{remark}[theorem]{Remark}
\newtheorem{assumption}{Assumption}

\newtheorem{definition}{Definition}

\usepackage{icml2022}

\usepackage{math_notation}

\icmltitlerunning{Submission and Formatting Instructions for ICML 2022}

\begin{document}

\twocolumn[
\icmltitle{Structure Learning in Graphical Models from Indirect Observations}



\icmlsetsymbol{equal}{*}

\begin{icmlauthorlist}
\icmlauthor{Hang Zhang}{ed}
\icmlauthor{Afshin Abdi}{ed}
\icmlauthor{Faramarz Fekri}{ed}
\end{icmlauthorlist}

\icmlaffiliation{ed}{School of Electrical and Computer Engineering, Georgia Institute of Technology, Atlanta, GA, USA.}

 \icmlcorrespondingauthor{}{}
 \icmlcorrespondingauthor{}{}

\icmlkeywords{Machine Learning, ICML}

\vskip 0.3in
]




\iftrue  
\begin{abstract}
This paper considers learning 
of the graphical structure of a $p$-dimensional random vector 
$\bX \in \RR^p$ using both parametric and 
non-parametric methods. 
Unlike the previous works which observe $\bx$ directly,  
we consider the indirect observation scenario  
in which samples $\by$ are
collected via a sensing matrix 
$\bA \in \mathbb{R}^{d\times p}$,  
and corrupted with some additive noise $\ensuremath{\mathbf{w}}$, i.e, 
$\bY = \bA\bX + \ensuremath{\mathbf{W}}$.
For the parametric method, we assume $\bX$ to be Gaussian, 
i.e., $\bx\in \mathbb{R}^p\sim \normdist\left(\ensuremath{\mathbf{\mu}}, \ensuremath{\mathbf{\Sigma}}\right)$,
$\ensuremath{\mathbf{\mu}} \in \mathbb{R}^p$,
and $\ensuremath{\mathbf{\Sigma}} \in \mathbb{R}^{p\times p}$.
For the first time, we show that 
the correct 
graphical structure can be correctly recovered under the indefinite sensing system ($d < p$) using insufficient samples ($n < p$). 
In particular, we show that for the exact recovery, 
we require 
dimension $d = \Omega\bracket{p^{0.8}}$ and 
sample number $n = \Omega\bracket{p^{0.8}\log^3 p}$. 
For the nonparametric method, 
we assume a nonparanormal distribution
for $\bX$ rather than Gaussian.
Under mild conditions, we show that our 
graph-structure estimator
can obtain the correct structure.
We derive the minimum sample 
number $n$ and dimension $d$ as
 $n\gsim (\textup{deg})^4 \log^4 n$ and $d \gsim p + \bracket{\textup{deg}\cdot\log(d-p)}^{\beta/4}$,
 respectively, where $\textup{deg}$ is the maximum Markov blanket in the graphical model and $\beta > 0$ is some fixed positive constant.
Additionally, we obtain a non-asymptotic uniform bound on the 
estimation error of the CDF of $\bX$ from indirect 
observations with  
inexact knowledge of the noise distribution. 
To the best of our knowledge, this bound is derived for the
first time and may serve as an independent interest.
Numerical experiments on both real-world and synthetic data 
are provided confirm the theoretical results.

\end{abstract}
\fi 

\vspace{-2mm}
\section{Introduction}

Graphical models provide a general framework of representing the  
dependency relations among 
random variables. They
have a broad spectrum of applications in biology, 
natural language processing and 
computer vision~\cite{koller2009probabilistic,friedman2004inferring}, etc. 
For an arbitrary 
random vector $\bX\in \RR^p$, we can construct a graphical model $G = (V, E)$ by 
associating each entry $X_i$ with a node $v_i\in V$ 
and adding an edge $e = (v_i, v_j)$ to the edge set $E$
if $X_i$ and $X_j$ are conditionally dependent given other random variables, 
where $X_i$ and $X_j$ denote the $i$th and $j$th entry of $\bX$, respectively. 

\par 
Discovery of the graph structure from a collection of direct 
observations of $\bX$ has been studied in the past 
\cite{ravikumar2011high, cai2011constrained, liu2009nonparanormal,liu2012high, zhao2014positive, xu2016semiparametric, fan2017high}. However, direct observations of the desired 
signal is not always possible.
 Instead, the signal has to be measured
indirectly.
Such a measurement model arises frequently in many practical applications, such as biomedical 
sensing \cite{muller2008sensitivity} where direct measurements
of some desired molecules such as miRNAs are too expensive. 
 Further, one commonly encountered 
problem is that the observations
are contaminated with measurement noise, which leads to inaccurate 
estimation of the graphical structure.  
Inspired by these challenges, 
we consider the graph structure recovery under an 
indirect linear measurement scenario from the desired signal 
as 

{\small \vspace{-8mm}
\begin{align}
\label{eq:prob_sense_relation_def}
\bY^{(s)} = \bA \bX^{(s)} + \bW^{(s)}, ~~1\leq s \leq n, 
\end{align}
\vspace{-8mm}
}

\par \noindent
where $\bY^{(s)}$ denotes the $s$th measurement, $\bA\in \RR^{d\times p}$ denotes
the sensing matrix, and $\bW\in \RR^d$ denotes the sensing noise. 
Our goal is to learn 
the pair-wise independence relation 
(structure) of an undirected 
graphical model, i.e., Markov random fields, from the observations $\{\bY^{(s)}\}_{1\leq s \leq n}$
using both parametric and non-parametric approaches.

In the parametric scenario, 
random vector $\bX \in \RR^p$ satisfies the Gaussian distribution, 
i.e., $\bX \sim \normdist(\bZero, \bSigma)$, and hence 
a \emph{Gaussian Graphical Model} (GGM). 
While for the nonparametric method, we assume random vector $\bX$ follows the nonparanormal distribution, i.e., 
the joint distribution $\bg(\bX)$ exhibits a Gaussian distribution 
$\normdist(\bmu, \bSigma)$, i.e., 
$\bX$ behaves as multivariate Gaussian after some 
transformation $\bg(\bX) = \
\Bracket{g_i(X_i)~\cdots~g_p(X_p)}^{\rmt}$. 
 Let $\bTheta$ be the 
inverse of the covariance matrix $\bSigma$. It is known that
conditional independency 
relation of $\bX$ is completely incorporated into matrix $\bTheta$. 
In other words, we have $\Theta_{i,j} = 0$ iff 
$X_i$ and $X_j$ are independent given the rest of entries $X_{\backslash\{i,j\}}$. 
Hence our goal of learning the graphical structure reduces to 
detecting the support set $S$ of $\bTheta$, where $S \defequal \set{(i,j)~|~\Theta_{i,j}\neq 0,~\forall i\neq j}$. 
This property of precision matrix lays the foundation
for many algorithms that learn the graphical structure.

\subsection{Related Works.}
To the best of our knowledge, our work is the first 
to consider learning 
the undirected graphical model structure 
of a high dimensional signal $\bX$ from 
possibly low dimensional indirect observations of $\bX$.
In the following, we separately review the related work regarding the 
parametric and non-parametric scenarios. 

\noindent \textbf{Parametric Scenario.}
Learning GGM when the data samples are \emph{directly} observed has been extensively studied in the literature,  e.g., 
\cite{ravikumar2011high,cai2011constrained}.
Both works assumed an sparse precision 
matrix $\bTheta$ 
and propose to recover its support set with $\ell_1$-relaxation.  
In \cite{ravikumar2011high}, the authors  
proposed gLasso, which 
can be regarded as an M-estimator \cite{wainwright_2019} (Chap.~$9$). 
The basic assumption is that the ground-truth signal 
lies within a certain low-dimensional space. 
By minimizing the negative log-likelihood
plus some specific regularizers, one 
can reconstruct the signal from insufficient
samples.
Different from \cite{ravikumar2011high}, 
\cite{cai2011constrained} proposed CLIME, a constrained 
convex optimization framework to estimate sparse 
precision matrix $\bTheta$. 
Instead of maximizing the likelihood, the 
authors first obtain an empirical covariance matrix 
$\wh{\bSigma}_n$ and then force $\bTheta$ to 
approximate 
its inverse. 
Other variations of GGM-learning such 
as GGM with hidden variables 
\cite{chandrasekaran2010latent,mazumder2012graphical}, and
GGM with an unknown block structure \cite{marlin2009sparse} have been proposed when some prior information or constraints on the structure exist. For detailed discussions, please refer to \cite{tibshirani2015statistical,wainwright_2019}.

\noindent \textbf{Non-parametric Scenario.}
The most related works on the nonparametric learning of graphical models 
are \cite{liu2009nonparanormal, liu2011forest,liu2012high, zhao2014positive, xu2016semiparametric, fan2017high}. 

The works by \cite{liu2009nonparanormal, liu2012high, xue2012regularized, zhao2014positive} considered a similar problem as ours, that is the discovery of
the graphical structure with nonparanormal distributed random vector $\bX$. 
However, they assumed direct observations of 
$\bX$ and without any measurement noise. 
The differences across the above works lie in the estimation method of the 
covariance matrix. In \cite{liu2009nonparanormal}, the covariance matrix
 is estimated via the 
CDF estimator; while in \cite{xue2012regularized}, it is estimated by the 
Spearman's rho estimator. In an independent work, \cite{liu2012high} 
pointed out the covariance matrix can be estimated by Kendall's tau estimator 
as well. 
Later, \cite{zhao2014positive} proposed a projection based algorithm 
to accelerate the estimator in \cite{liu2012high}. 
Apart from the graphical models following the nonparanormal distribution, 
other types of works include the graph with forest structure \cite{liu2012high}, the graph 
with the elliptical distribution \cite{xu2016semiparametric}, 
and latent Gaussian copula model \cite{fan2017high}, etc. 

Another line of research is the density deconvolution which dates at least 
back to \cite{zhang1990fourier, fan1991optimal, masry1991multivariate}, where 
kernel-based estimators are proposed for the noise with infinite support 
set. To improve the performance for the noise with finite support, 
a ridge-parameter method is proposed in \cite{hall2007ridge}. Similar 
work includes \cite{dattner2011deconvolution,trong2019deconvolution}.  
The work by \cite{youndje2008optimal} generalized \cite{masry1991multivariate} 
and gave a data-driven method to select the optimal bandwidth. 
Apart from the kernel-based methods, \cite{pensky1999adaptive} presented
a projection-based method based on the 
Meyer-type wavelets, which adapts to the super-smooth noise.  
Notice that the above works all assume perfect knowledge of the noise distribution. 
Later, \cite{dattner2013estimation, kappus2014adaptive}
and \cite{phuong2020deconvolution}
studied the unknown noisy case but required repeated measurements
to estimate the distribution of the noise. In contrast, our setting 
assumes inexact knowledge of the noise and requires no extra step for the 
noise estimation. More importantly, our setting is focused on 
indirect measurements. 

\vspace{-2mm}
\subsection{Contributions}
To the best of our knowledge, our work is the first 
on estimating the graphical 
structure under an 
indirect measurement scenario. 
We also provide the theoretical analysis of our estimation. 

\begin{itemize}[leftmargin=*]
\vspace{-4mm}
\item 
For the parametric scenario, we show 
the correct 
graphical structure can be recovered under the indefinite sensing system ($d < p$) using insufficient samples ($n < p$). Specifically,  
we require the sample number $n = \Omega\bracket{p^{0.8}\log^3 p}$ and 
the dimension $d = \Omega\bracket{p^{0.8}}$, which 
suggests with fewer samples than 
signal length $p$, we can still recover the graph structure of 
a high dimensional signal from its low dimensional observations ($d <p$). 

\item 
For the non-parametric scenario, 
we propose an estimator for the graphical structure
together with the sufficient conditions for the correct recovery. 
We show that the sample number $n$ must be 
at least $n\gg (\textup{deg})^4 \log^4 n$, 
where $\textup{deg}$ denotes the maximum Markov blanket in the graph. 
Further, we obtain a lower bound on 
 the dimension as
 $d \gg p + \bracket{\textup{deg}}^{\beta/4}\log^{\beta/4}(d-p)$,
  where $\beta > 0$ is some fixed positive constant to be defined. 

\item 
Additionally, our work is the first to consider the deconvolution 
estimator for the CDF with  
limited knowledge of the noise distribution. 
Compared with the previous work \cite{dattner2013estimation,  kappus2014adaptive, phuong2020deconvolution}, our work 
does not assume perfect knowledge of 
noise distribution, and our estimator requires no additional steps to 
estimate the noise distribution. Moreover, we give 
a non-asymptotic uniform bound on the estimation error, that is, 
$\sup_{x}\abs{\wh{F}(x) - F(x)} \lsim \frac{\log^2(np)}{(d-p)^{\beta/2}}+ \frac{1}{\sqrt{n}}$ (cf. Thm.~\ref{thm:cdf_tail_prob}), where $\wh{F}(\cdot)$ and $F(\cdot)$ denote the estimated 
marginal CDF  and ground-truth marginal CDF, respectively. 
These analytical results might be of interest to some
other areas of statistics.  
\end{itemize}

\section{Problem Formulation}
\label{sec:prob_set}
We start with a formal restatement of the 
sensing relation reading as 

{\small \vspace{-8mm}
\begin{align}
\label{eq:prob_sense_relation_def}
\bY^{(s)} = \bA \bX^{(s)} + \bW^{(s)}, ~~1\leq s \leq n, 
\end{align}\vspace{-8mm}
}

\noindent where $\bY^{(s)} \in \RR^d$ denotes the $s$th reading, 
$\bA \in \RR^{d\times p}$ is the sensing matrix with  
each entry $A_{ij}$ being a standard normal RV, i.e., $A_{ij}\sim \normdist(0, 1)$,  and $\bW^{(s)}$ denotes the measurement noise with 
each entry $\bW_{i}$ being a Gaussian RV with zero mean and variance
$\sigma^2$, i.e., $\bW_{i}\sim \normdist(0, \sigma^2)$. 

We assume that most entries of $\bX$ are pair-wise
 conditionally independent, 
which is widely used in previous works 
such as \cite{ravikumar2011high,cai2011constrained, liu2009nonparanormal}. 
Our goal is to uncover the undirected graphical structure (or pair-wise independence) of 
$\bX$ from the 
samples $\set{\bY^{(s)}}_{1\leq s \leq n}$. 
Before proceeding, we first introduce the notations.

\vspace{-2mm}
\paragraph{Notations.}
We denote $c, c_0, c_i > 0$ as arbitrary fixed positive constants. Notice
that the specific values may not be necessarily identical even  if
they share the same name. For arbitrary real numbers $a$ and $b$, 
we denote $a \lsim b$ if there exists some $c_0 > 0$ such that $a \leq c_0 b$. 
Similarly we define $a\gsim b$. 
We write $a\asymp b$ when $a\lsim b$ and $a\gsim b$ hold simultaneously. 
The maximum of $a$ and $b$ is denoted as $a\vcup b$; while the
minimum is denoted as $a\vcap b$. 
\par 
For an arbitrary matrix $\bM \in \RR^{m_1\times m_2}$, we denote $\bM_i$
as its $i$-th column and $M_{ij}$ as the 
$(i, j)$-th entry. 
Its Frobenious norm 
$\fnorm{\bM}$ is defined as $\sqrt{\sum_{i,j}M_{i,j}^2}$
and the operator norm $\opnorm{\bM}$ 
is defined as $\max_{\norm{\bu}{2} = 1} \norm{\bM\bu}{2}$.  
Furthermore, we define $\Norm{\bM}{a, b}$ as $\max_{\norm{\bz}{a} =1}\norm{\bM \bz}{b}$. 
Two special cases are $\Norm{\bM}{1, 1}$ and $\Norm{\bM}{\infty, \infty}$,
which can be written as 
$\Norm{\bM}{1, 1}\defequal \max_j\sum_{i}\abs{M_{i,j}}$, and $\Norm{\bM}{\infty, \infty} \defequal \max_i\sum_{j}\abs{M_{i,j}}$, respectively. 
Moreover, we define $\offdiagnorm{\bM}\defequal \sum_{i\neq j}\abs{M_{i,j}}$,
$\Norm{\bM}{\textup{off}, \textup{F}}\defequal \sqrt{\sum_{i\neq j}M_{i,j}^2}$, 
and $\infnorm{\bM}\defequal \max_{i, j}|M_{i, j}|$, respectively.

For the covariance matrix $\bSigma$, we denote 
its inverse as $\bTheta$. 
The support set $S$ is defined as $\set{(i,j)~|~\bTheta_{i,j}\neq 0,~\forall~i\neq j}$ and its complement is denoted as $S^c$. 
Moreover, we define the maximum Markov blanket in the graphical model as 
$\textup{deg}$, or equivalently the maximum of the non-zero entries 
among columns of $\bTheta$, i.e., $\textup{deg}\defequal \max_i \|\bTheta_i\|_{0}$, where $\bTheta_i$ denotes the 
$i$th column of $\bTheta$. 
The parameter $\kappa_{\bSigma}$ is defined as $
 \max_i\sum_{j}\abs{\bSigma_{i,j}}$.
Furthermore, we define the Fisher information matrix of $\bSigma$ 
as
 $\bGamma \defequal \bSigma\otimes \bSigma$, 
where $\otimes$ is the kronecker product \cite{golub2012matrix}.
The parameter $\kappa_{\bGamma}$ is defined as 
$\Norm{(\bGamma_{SS})^{-1}}{\infty, \infty}$, 
where $S$ is the support set and $\bGamma_{SS}$ is the 
sub-matrix of $\bGamma$ such that its rows and columns lie within 
the set $S$.

\section{Parametric Method}
For the parametric scenario, we assume 
random vector $\bX\in \RR^p$ follows the Gaussian distribution with 
zero mean and covariance  $\covmatparam$, 
i.e., $\bX\sim \normdist(\bZero, \covmatparam)$.
We focus 
on the scenario where the sensing matrix $\bA\in \RR^{d\times p}$ is 
under-determined, namely, $d < p$. In other words, 
we wish to recover the graphical model of a high dimensional 
$\bX$ from its low dimensional observations $\bY$. 
To avoid measurements being dominated by one or a small group of $X_i$'s, we assume that $\Var(X_i) = 1$, $1\leq i \leq p$. 
Our proposed estimator to reconstruct the graphical structure is 
illustrated in Alg.~\ref{alg:param_glasso}. The analysis of the 
estimator is given in the following.

\subsection{Properties of Graphical Structure Estimator: Parametric Case}
We will show that a correct graphical structure can be obtained 
when the penalty coefficient $\lambda_{\textup{param}}$ in 
\eqref{eq:param_glasso_def} 
is properly chosen. The core of the analysis 
lies in the following lemma. 

\begin{lemma}
The property
$\|\wh{\bSigma}_n^{\textup{param}} - \bSigma^{\natural}\|_{\infty}\leq \tau_{\infty}$ 
holds with probability $1-o(1)$, 
where $\wh{\bSigma}^{\textup{param}}_n$ is defined in ~\eqref{eq:param_cov_mat_estim}, 
and the threshold parameter
$\tau_{\infty}$ is written as 

{\small \vspace{-5mm}
\begin{align}
\label{eq:param_tau_inf_def}
\tau_{\infty} \defequal ~&
\frac{c_0\sqrt{d\log p}}{d+1}
\max_i \norm{\covmatparam_i}{2} + 
\frac{c_1\log p}{d+1}\bracket{1 + \frac{c_2p}{d}}
\Norm{\covmatparam}{\textup{off}, \textup{F}} \notag \\
+~& \
\frac{c_3\log p\sqrt{dp}}{\sqrt{n}(d+1)} 
+ \frac{c_4 p(\log p)^{3/2}}{\sqrt{n}(d+1)}\bracket{1 + \frac{c_2p}{d}}
+ \frac{c_5p\sqrt{\log p}}{\sqrt{d}(d+1)}\notag  \\
+~& \frac{c_6 \sigma^2 \log p}{d}\Bracket{1 + c_{7}\bracket{\sqrt{\frac{d}{n}} \vcup \frac{d}{n}}}.
\end{align}
 \vspace{-4mm}
}
\end{lemma}


\noindent
Having obtained the covariance matrix $\wh{\bSigma}^{\textup{param}}_n$, 
we estimate the graphical structure of $\bX$ by plugging 
$\wh{\bSigma}^{\textup{param}}_n$ into the graphical lasso estimator \cite{tibshirani2015statistical} as in \eqref{eq:param_glasso_def}.
To analyze the properties of the graphical structure estimator, 
we assume the 
\emph{irrepresentable 
condition},
a widely-used assumption in the previous works (cf. \cite{ravikumar2011high, liu2009nonparanormal}), which  
control the influence of
the non-edge terms, i.e., $\bGamma_{S^c S}$, on the edge terms, 
i.e., $\bGamma_{SS}$ (see Section~\ref{sec:prob_set} for definitions of these matrices).


\begin{assumption}[irrepresentable condition]
\label{assump:param_irrepresentable_condition}
Let the support set $S$ and matrix $\bGamma$ be defined in Section~\ref{sec:prob_set}. Then, there exists a positive constant $\theta \in (0, 1]$ such that
$\Norm{\bGamma_{S^c S} (\bGamma_{S S})^{-1}}{1, 1} \leq 1- \theta$.
\end{assumption}
Adopting a similar proof method as in \cite{ravikumar2011high}, 
we can obtain the
conditions (cf. Thm.~\ref{thm:param_glasso}) for the correct recovery of graphical structure 
under Assumption~\ref{assump:param_irrepresentable_condition}.
For the conciseness of presentation, we refer the details to 
\cite{ravikumar2011high} and omit them in this work. 
 
\begin{algorithm}[h]
\caption{Estimation of the Graphical Structure via the Parametric Method.}
\label{alg:param_glasso}
\begin{algorithmic}
\STATE {\bfseries Input:} Samples $\set{\bY^{(i)}}_{i=1}^n \in \RR^p$ 
and sensing matrix $\bA \in \RR^{d\times p}$. 

\STATE {\bfseries Stage I:}
Estimate the covariance matrix $\wh{\bSigma}^{\textup{param}}_n$ as

{\small \vspace{-4mm}
\begin{align}
\label{eq:param_cov_mat_estim}
\wh{\bSigma}^{\textup{param}}_n
	= \bI + \frac{1}{d+1}
\Bracket{\bA^{\rmt}\bracket{\frac{1}{n}\sum_{i=1}^n \bY^{(i)}\bY^{(i)\rmt}}\bA}_{\textup{off}},
\end{align}}

\vspace{-2mm}
\noindent
where $[\cdot]_{\textup{off}}$ denotes the operation of 
picking non-diagonal entries,

\STATE {\bfseries Stage II:}
Obtain $\wh{\bTheta}^{\textup{param}}$ as 
\par

{\small \vspace{-4mm}
\begin{align}
\label{eq:param_glasso_def}
 \hspace{-3.5mm}
\wh{\bTheta}^{\textup{param}} \hspace{-1mm}= \underset{\bTheta \succ \bZero}{\argmin}
\vspace{-1mm}~& -\vspace{-1mm}\logdet(\bTheta) +\trace(\wh{\bSigma}^{\textup{param}}_n \bTheta)\notag \\
+~&\lambda_{\textup{param}} \offdiagnorm{\bTheta}, 
\end{align} \vspace{-7mm}
}

\par 
where $\lambda_{\textup{param}}> 0$ is some 
positive constants.

\STATE {\bfseries Output:} 
Estimated precision matrix $\wh{\bTheta}_{\textup{param}}$. 
\end{algorithmic}
\end{algorithm}

\begin{theorem}
\label{thm:param_glasso}

Let 
$\tau_{\infty}\lsim \frac{\theta}{(\theta + 8)(\textup{deg})\cdot\kappa^3_{\bSigma}\kappa_{\bGamma}}$, 
and set $\lambda_{\textup{param}} = 8\tau_{\infty}/\theta$. 
Then the property $(\wh{\bTheta}^{\textup{param}})_{i, j}= 0$ holds for 
all $(i,j)$ outside the support set $S$, i.e.,
$(i, j)\in S^c$, with probability  $1-o(1)$.
Furthermore, if $\min_{(i,j)\in S}|\bTheta^{\natural}_{i,j}|\geq
2\kappa_{\bGamma}(1+8\theta^{-1})\tau_{\infty}$, 
then $\sign(\wh{\bTheta}^{\textup{param}}) = \sign(\precmatparam)$ with 
probability $1-o(1)$. 
\end{theorem}
This theorem proves that our proposed Alg.~\ref{alg:param_glasso} can detect all conditionally independent pairs (the graph edges) provided that the non-zero elements $\bTheta^{\natural}_{i,j}$, 
$(i,j)\in S$, are not too small, i.e., the absolute value of non-zero elements in $\bTheta^{\natural}$ are above some fixed threshold.

\subsection{Discussions}
\label{subsec:param_discussion}
In the following, we provide more insights regarding:
$(i)$ the minimum sample size $n$ and $(ii)$ the minimum projection dimension $d$. 
An illustration of the infeasible region of $n$ and $d$ is plotted in Fig.~\ref{fig:param_graph_feasible_region}, where the infeasible region is marked with color red. 

\begin{figure}[h]
\centering
\includegraphics[width = 2.8in]{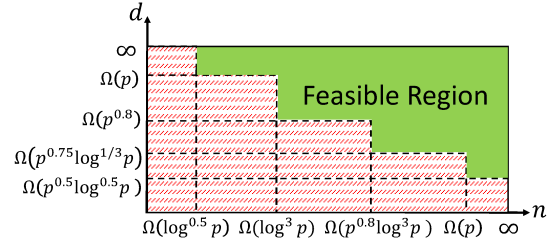}
\caption{Infeasible region w.r.t  number $n$ and dimension $d$.}
\label{fig:param_graph_feasible_region}
\end{figure}

\vspace{-4mm}
\paragraph{Minimum Sample Size $n$.}
In the high-dimensional setting, 
it is desirable to 
reduce
the sample number $n$ to less than the 
dimension of signal $p$, i.e., $n < p$.

Treating parameters $\theta$, $\kappa_{\bSigma}$, and 
$\kappa_{\bGamma}$ as some positive constants, 
Thm.~\ref{thm:param_glasso} requires that
$\textup{deg}\cdot \tau_{\infty}$ to be some 
positive constant, from which we can obtain the 
minimum sample size
$n\gsim \sqrt{\textup{deg}^2 \cdot \log p}$. 
For the case of  indefinite sensing matrix $\bA\in \RR^{d\times p}$ ($d < p$), we need to increase the sample number $n$ to 
the order of $\Omega\bracket{\textup{deg}^2 \cdot (\log p)^3}$,
which is obtained by setting $d = c p$, where $0 < c < 1$ is some
positive constant. 
When $d$ reduces to $\Omega(p^{0.8})$,
we need to further inflate the sample number $n$ to be at least 
$\textup{deg}^2(\log p)^3 p^{0.8}$, which is still less than 
the order $p$.  


\paragraph{Minimum Projection Dimension $d$.}
We investigate  the minimum projection 
dimension $d$ and proves 
the possibility of exact recovery of graph
with an indefinite sensing matrix $\bA$.
Under the high-dimensional setting where we require
$n\leq p$, we need $d$ to be at least 

{\small \vspace{-6mm}
\[
d \gsim~& \sqrt{\textup{deg}\cdot p\log p\cdot \Norm{\covmatparam}{\textup{off}, \textup{F}}} 
\vcup \sigma^2 \textup{deg}\cdot \log p \\
~&\vcup p^{3/4}\bracket{\log p}^{1/3}\sqrt{\textup{deg}} 
\vcup \textup{deg}^2\cdot \log p\cdot\bracket{\max_i \norm{\covmatparam_i}{2}^2}.
\]\vspace{-6mm}
}

Using above, we can confirm that $d$ is less than $p$ given that
$\Norm{\covmatparam}{\textup{off},\textup{F}} = o\bracket{p}$
and $\max_i \norm{\covmatparam_i}{2} = o(\sqrt{p})$. 
If $\Norm{\covmatparam}{\textup{off},\textup{F}}$ is 
some fixed positive constant, we can reduce the dimension 
$d$ to $\sqrt{\textup{deg}}\cdot p^{3/4}(\log p)^{1/3}$, 
approximately of the order $\bracket{\textup{deg}}^{1/2}p^{3/4}$.

\section{Non-parametric Method}
\label{sec:nonparam}
We begin the discussion with some background knowledge on the 
nonparanormal distribution. 

\subsection{Background}

\begin{definition}[Nonparanormal]
\label{def:nonparanormal}
We call a set of random variables $\bX = \Bracket{X_i~\cdots~X_p}^{\rmt}$
follows the \emph{nonparanomral distribution}, namely, 
$\bX \sim \textup{NPN}(\bg, \bmu, \bSigma)$, 
if there exists a set of functions $\set{g_j}_{1\leq j \leq p}$
such that $\bZ = \Bracket{g_1(X_1)~\cdots~g_p(X_p)}^{\rmt}$
has Gaussian distribution, i.e., 
$\bZ \sim \normdist\bracket{\bmu, \bSigma}$. 
\end{definition}

Assume that $\bX$ satisfies the nonparanormal assumption, i.e., 
$\bX \sim \textup{NPN}(\bg, \bmu, \bSigma)$.  
Let $\bTheta$ be the inverse matrix of the covariance matrix 
$\bSigma$. 

\begin{lemma}[Lemma~$3$ in \cite{liu2009nonparanormal}]
\label{lemma:dependence_theta_matrix}
Random variables $X_i$ and $X_j$ are pairwise 
conditionally independent, i.e., $X_i \perp X_j~\big| X_{\setminus{i,j}}$, iff 
$\bTheta_{ij} = 0$. 
\end{lemma}

Thus, the pair-wise independence relation across the entries of $\bX$ is 
fully incorporated into the matrix $\bTheta$.
Denote the marginal distribution function of the $i$th entry $X_i$ as 
$F_i(\cdot)$ and define the function $h_i(x) = \Phi^{-1}\bracket{F_i(x)}$, 
where $\Phi(\cdot)$ is the CDF of the standard normal RV, namely, 
$\Phi(\cdot) = \frac{1}{\sqrt{2\pi}}\int_{-\infty}^{(\cdot)}e^{-t^2/2}dt$. According to \cite{liu2009nonparanormal}, 
we can estimate function $g_i(\cdot)$ by estimating 
function $h_i(\cdot)$, from which
we can conclude the bottleneck in estimating 
the graphical structure lies in the estimation of the CDF functions, i.e., 
$F_i(\cdot)$, $1\leq i \leq p$. 
Different from the previous work \cite{liu2009nonparanormal}, our sensing
relation assumes noisy indirect measurements instead of noiseless direct
measurements, which bring extra difficulties. 

{\small 
\begin{algorithm}
\caption{Estimation of the Graphical Structure via the Non-parametric Method.}
\label{alg:nonparam_graph_estim}
\begin{algorithmic}
\STATE {\bfseries Input:} Samples $\{\bY^{(i)}\}_{i=1}^n \in \RR^p$ 
and sensing matrix $\bA \in \RR^{d\times p}$. 

\STATE {\bfseries Stage I}.
We reconstruct the values $\wh{\bX}^{(s)}$ via the 
\emph{least-square} (LS) estimator as 

{\small \vspace{-4mm}
\begin{align}
\label{eq:ls_def}
\wh{\bX}^{(s)} = \argmin_{\bX} \norm{\bY^{(s)}- \bA \bX}{2}.
\end{align}\vspace{-4mm}
}

\STATE {\bfseries Stage II}.
We estimate the marginal distribution function $\wh{F}_i(\cdot)$ for
the $i$th entry from the samples $\{\wh{X}^{(s)}_i\}_{1\leq s \leq n}$  as 

{\small \vspace{-3mm}
\begin{align}
\label{eq:cdf_estim_def}
\wh{F}_i(x) = 
\frac{1}{2} -~& \frac{1}{n\pi}\sum_{s=1}^n
\int_{0}^{\infty} \frac{ \sin\Bracket{t(\wh{X}^{(s)}_i - x)}}{t} \notag \\
\times~&  \
\frac{\exp\bracket{-\frac{\sigma^2 t^2}{2(d-p)}}}{ \exp\bracket{-\frac{\sigma^2 t^2}{d-p}} \vcup
\gamma t^a}
dt, 
\end{align}
\vspace{-2mm} 
}

\par \noindent
where $\gamma > 0$ and $a > 1$ are some fixed positive constants. Then 
we truncate the estimated CDF function as

{\small \vspace{-3mm}
\begin{align}
\label{eq:cdf_truncate_estim}
\wh{F}^{\textup{tr}}_i(x) = \left\{\
\begin{aligned}
& \delta_{n, d, p},~~&&\wh{F}_i(x) \leq \delta_{n, d, p};\\
&\wh{F}_i(x),~~&& \delta_{n, d, p} \leq \wh{F}_i(x) \leq 1- \delta_{n, d, p}; \\	
&1- \delta_{n, d, p},~~&&\wh{F}_i(x) \geq 1- \delta_{n, d, p}, 
\end{aligned}\right.
\end{align}
\vspace{-4mm}
}

\par \noindent
where $\delta_{n, d, p} > 0$ is some pre-determined parameter.  

\STATE {\bfseries Stage III}. 
First, we estimate the mean $\wh{m}_i$ and the variance $\wh{v}_i$ as 

{\small \vspace{-6mm}
\begin{align*}
\wh{m}_i &= \frac{1}{n}\sum_{s=1}^n \wh{X}^{(s)}_i; \\\wh{v}_i &= \sqrt{\frac{1}{n-1}\sum_{s=1}^n \bracket{ \wh{X}^{(s)}_i - \wh{m}_i}^2 - \frac{n}{n-1}\frac{\sigma^2}{d-p}}. 
\end{align*}\vspace{-4mm}
}

\par \noindent
Then we estimate the covariance matrix $\wh{\bSigma}^{\textup{non-param}}_n$ as  

{\small \vspace{-3mm}
\begin{align}
\label{eq:nonparam_emp_cov_estim}
\wh{\bSigma}^{\textup{non-param}}_n = \frac{1}{n}
\sum_{s=1}^n
(\wh{\bh}(\wh{X}^{(s)}) - \wh{\bmu})
(\wh{\bh}(\wh{X}^{(s)}) - \wh{\bmu})^{\rmt},
\end{align}\vspace{-3mm}
}

\par \noindent
where the $i$th entry of $\wh{\bh}$ is defined as 
$\wh{h}_i(x) = \wh{m}_i + \wh{v}_i \Phi^{-1}(\wh{F}_i^{\textup{tr}}(x))$, 
and $\wh{\bmu}$ is the estimated mean of $\wh{\bh}(\cdot)$,
namely, $n^{-1}{\sum_{s=1}^n \wh{\bh}(\bX^{(s)})}$. 

\STATE {\bfseries Stage IV}.
We reconstruct the matrix 
$\wh{\bTheta}$ as 
\vspace{-2mm}
\begin{align}
\label{eq:nonparam_glasso_def}
\wh{\bTheta}^{\textup{non-param}} = ~&\argmin_{\bTheta}
-\log\det(\bTheta) + \trace\bracket{\bTheta\wh{\bSigma}^{\textup{non-param}}_n} \notag \\
+ & \lambda_{\textup{non-param}} \offdiagnorm{\bTheta}, 
\vspace{-4mm}
\end{align}
where $\lambda_{\textup{non-param}} > 0$ is some positive constant for the 
regularizer $\Norm{\bTheta}{1}$.

\STATE {\bfseries Output:} 
Estimated matrix $\wh{\bTheta}^{\textup{non-param}}$. 
\end{algorithmic}
\end{algorithm}
}

\vspace{-2mm}

\subsection{Estimator Design Intuition}
Compared with the parametric learning method, 
we need to have a reliable estimation of the
transform function $\bg(\cdot)$ in the presence of the
measurement noise $\bW$. 
To handle such an issue, we relax the constraint 
that sensing matrix $\bA$ is under-determined, where  
extra measurements are conducted ($d > p$) to 
suppress the sensing noise.
\begin{remark}
Consider the direct noisy measurement reading as 
$\wt{\bY}^{(s)} = \bX^{(s)} + \bW^{(s)}$.
Previous works \cite{fan1991optimal, zhang1990fourier} suggest
that the error of estimating the CDF $X_i$ from $\{\wt{\bY}^{(s)}\}$
is at least
$\inf_{\wh{F}}\abs{\wh{F}_i(x) - F_i(x)} \geq c_0 (\log n)^{-c_1}$ for a fixed point $x$, which means significant number of samples are required. 

\end{remark}

\begin{remark}
In fact, 
whether a compressive sensing system, i.e., $d < p$, 
can be used to estimate the transform function $\bg(\cdot)$ still remains 
an open-problem.
\end{remark}

The basic idea of our estimator is to first transform the sensing 
relation in \eqref{eq:prob_sense_relation_def} to the additive model such that 
$\wh{\bX}^{(s)} = \bX^{(s)} + \wh{\bW}^{(s)}$, where 
$\wh{\bW}^{(s)}$ is defined as $\bracket{\bA^{\rmt}\bA}^{-1}\bA^{\rmt} \bW^{(s)}$. 
Then our task reduces to estimating the marginal CDF from the 
samples $\wh{\bX}^{(s)}$, which is contaminated by the noise $\wh{\bW}^{(s)}$.
This problem can be broadly categorized as
the \emph{density deconvolution} problem in statistics. 
When comparing with the previous work, 
instead of relying on full knowledge, we only use some
inexact knowledge of the distribution of $\wh{\bW}^{(s)}$. 
To put more specifically, we can compute its mean and the approximated value of its 
variance as
\vspace{-4mm}
\begin{align}
\label{eq:noise_approx_mean_var}
\Expc(\wh{w}_i) = 0,~~\Var(\wh{w}_i) \approx \frac{\sigma^2}{d-p}, 
\vspace{-12mm}
\end{align}
where $\wh{w}_i$ denotes the $i$th entry of the noise $\wh{\bW}$.
Whereas the previous work such as \cite{zhang1990fourier, fan1991optimal, masry1991multivariate, hall2007ridge} assumes perfect 
knowledge of the distributions of $\wh{\bW}$, and the work \cite{dattner2013estimation, kappus2014adaptive, phuong2020deconvolution} 
disregards
the distribution but requires extra steps to estimate 
the variance of noise $\{\wh{\bW}^{(s)}\}_{1\leq s \leq n}$. 
The computation procedure of \eqref{eq:noise_approx_mean_var} 
is deferred to the supplementary material and the
details of the estimator is summarized in Alg.~\ref{alg:nonparam_graph_estim}.

\subsection{Properties of CDF Estimator}
This subsection investigates the properties of the marginal 
CDF estimator in \eqref{eq:cdf_estim_def} where each entry 
$X_i$ is within the region $[0, 1]$. 
We begin the discussion by presenting the assumptions.

\begin{definition}[Density family $\calF_{\alpha, L}$]
The density family $\calF_{\alpha, L}$ is defined as the 
set of all distributions whose density functions $f(\cdot)$ 
and characteristic functions $\phi(\cdot)$ possess the 
following properties:
\vspace{-2mm}
\begin{itemize}[leftmargin=*]
\item 
The  functions $f(\cdot)$ satisfy $\int_{-\infty}^{\infty}x^2 f(x)dx < \infty$; 
\item 
The characteristic functions $\phi(\cdot)$ satisfy 
$\int_{-\infty}^{\infty} \abs{\phi(t)}^2(1+t^2)^{\alpha}dt \leq L$, 
where $\alpha$ is a positive constant controlling the smoothness of the 
corresponding
PDF. 
\end{itemize}
\end{definition}

\vspace{-2mm}
Notice that these assumptions may have been widely used in previous work \cite{trong2019deconvolution, phuong2020deconvolution}. 
Many usual distributions belong to this density family, e.g., 
Cauchy distribution, Gaussian distribution, etc. 
In addition, we need the following two assumptions. 
 
\begin{assumption}
\label{assump:charx}
For an arbitrary entry $X_i$, 
we assume its distribution belongs to the 
density family $\calF_{\alpha, L}$ such that $\alpha > -\frac{1}{2}$. 
\end{assumption}

\begin{assumption}
\label{assump:lipschitz}
For an arbitrary entry $X_i$, we 
assume its density function $f_i(\cdot)$ is bounded by some 
constant $L_f$,
i.e., $|f_i(\cdot)|\leq L_f$. 
\end{assumption}

Setting the parameter $a$ as some fixed positive constant and using 
$\gamma \asymp  \log(np)\bracket{\frac{\sigma^2}{d-p}}^{a/4}$, 
we can prove that the estimation error of the CDF estimator converges to zero with high probability, which is formally stated in the following 
theorem. 

\begin{theorem}
\label{thm:cdf_tail_prob}
Under the Assumptions~\ref{assump:charx} and \ref{assump:lipschitz},  
the estimator error can be bounded as 

{\small \vspace{-6mm}
\begin{align}
\label{eq:cdf_inf_deviation}
\sup_{x\in [0, 1]}\abs{F_i(x) - \wh{F}_i(x)} \leq (\log n)\varepsilon_x+\
c_1\sqrt{\varepsilon_x} + \frac{c_2}{\sqrt{n}}, 
\end{align} \vspace{-6mm}
}

\par \noindent
for $1\leq i \leq p$
with probability exceeding $1-o(1)$ 
when setting 
$\gamma \asymp  \log(np)\bracket{\frac{\sigma^2}{d-p}}^{\frac{a}{4}}$.
The parameter $\varepsilon_x$ is defined as

{\small  \vspace{-6mm} 
\begin{align}
\label{eq:varepsilonx_def}
\varepsilon_x \defequal
\frac{\log^{{2}/{a}}(np) \sigma}{\sqrt{d-p}}+
\frac{\log^2(np)}{(d-p)^{\frac{a}{4}}}+
\bracket{\frac{\sigma^2}{d-p}}^{\frac{2\alpha + 1}{4}} +
\frac{1}{n},
\end{align}
 \vspace{-4mm}
}

\par \noindent
where $a > 1$ is some pre-determined positive constant. 
\end{theorem} 

\vspace{-2mm}
To the best of our knowledge, this is the 
first non-asymptotic uniform bound on the estimation error of a  
CDF deconvolution estimator with inexact knowledge of 
the noise variance.
%


\subsection{Properties of the Graphical Structure Estimator: Non-Parametric Case}
The technical challenges  of this analysis can be divided as two parts:
$(i)$ choosing the appropriate truncation parameter $\delta_{n, d, p}$ 
in \eqref{eq:cdf_truncate_estim}; and $(ii)$ estimating the covariance matrix 
with the noisy samples $\set{\bY^{(s)}}_{1\leq s \leq n}$.

Denote the oracle empirical covariance matrix $\bSigma^{\textup{non-param}}_n$ as 
{\small 
\vspace{-3mm}\begin{align}
\label{eq:nonparam_oracle}
\bSigma^{\textup{non-param}}_n \defequal~& \frac{1}{n}\sum_{s=1}^n 
\bh(\bX^{(s)}) \bh(\bX^{(s)})^{\rmt} \notag \\
-~& 
\big(n^{-1}\sum_{s=1}^n 
\bh(\bX^{(s)})\big) \big(n^{-1}\sum_{s=1}^n 
\bh(\bX^{(s)})\big)^{\rmt}, 
\end{align} \vspace{-4mm}
}

where $h(\cdot)$ denotes the oracle estimator of the 
transform functions in Def.~\ref{def:nonparanormal}.  
The core of the analysis lies on 
bounding the estimation error of the covariance 
matrix in terms of the $\ell_{\infty}$. 
In comparison with the previous work \cite{liu2009nonparanormal}, we cannot 
directly access the samples $\set{\bX^{(s)}}_{1\leq s \leq n}$. Instead, we have to use 
the perturbed samples $\{\wh{\bX}^{(s)}\}_{1\leq s \leq n}$, which will lead to additional errors in the estimation of the covariance matrix.
How to bound these errors constitutes the technical bottleneck.

Define $\beta$ as $\frac{1}{2}\vcap \frac{a}{4} \vcap \frac{2\alpha + 1}{4}$
and set $\delta_{n, d, p}$ in \eqref{eq:cdf_truncate_estim} as 
{\small \vspace{-2mm}
\begin{align}
\label{eq:deltandp_def}
\delta_{n, d, p} = 
\frac{c_0}{(\log n) n^{1/4}} 
+\frac{c_1\log^2(np)}{\sqrt{\log(d-p)} (d-p)^{\beta/4}}. 
\end{align}\vspace{-4mm}
}

\noindent
Then we conclude 
\begin{theorem}
\label{thm:graph_error_infnorm}
Under the Assumptions~\ref{assump:charx} and \ref{assump:lipschitz}, 
the following bound 

{\small \vspace{-8mm}
\[
\infnorm{\bSigma^{\textup{non-param}}_n - \wh{\bSigma}^{\textup{non-param}}_n} \lsim~& 
\sqrt{\log n\vcup \log(d-p)} \\
\cdot ~& \bracket{\frac{\sqrt{\log n}}{n^{1/4}}\vcup 
\frac{\sqrt{\log(d-p)}}{(d-p)^{\beta/4}}}
\]\vspace{-6mm}
}

\noindent
holds with probability exceeding 
$1 - o(1)$, where $\delta_{n, d, p}$ is set as \eqref{eq:deltandp_def},  $\beta$ is defined as 
$\frac{1}{2}\vcap \frac{a}{4} \vcap \frac{2\alpha + 1}{4}$, and 
$\bSigma^{\textup{non-param}}_n$ and $\wh{\bSigma}^{\textup{non-param}}_n$ are defined in 
\eqref{eq:nonparam_oracle} and
\eqref{eq:nonparam_emp_cov_estim}, respectively.
\end{theorem}

Having obtained the covariance matrix $\wh{\bSigma}^{\textup{non-param}}_n$, 
we estimate the graphical structure of $\bX$ by plugging 
$\wh{\bSigma}^{\textup{non-param}}_n$ into the graphical lasso estimator \cite{tibshirani2015statistical}, which 
is put in \eqref{eq:nonparam_glasso_def}.
Adopting the same strategy as in Thm.~\ref{thm:param_glasso}, we can obtain the
conditions for the correct recovery of graphical structure 
under Assumptions~\ref{assump:param_irrepresentable_condition}, \ref{assump:charx}, and~\ref{assump:lipschitz}.
With a slight abuse of notations, we redefine the support set $S$, 
the parameters $\textup{deg}$, 
$\kappa_{\bGamma}$, and $\kappa_{\bSigma}$ w.r.t the covariance matrix 
$\covmatnonparam$ and matrix $\precmatnonparam$
that are correspond to the nonparanormal distribution. 
Then we can show the following results on the graph recovery. 
\begin{theorem}
\label{thm:glasso}
Let $\lambda_{\textup{non-param}} \asymp \theta^{-1}\bracket{\frac{{\log n}}{n^{1/4}}\vcup 
\frac{{\log(d-p)}}{(d-p)^{\beta/4}}}$. Provided that 

{\small \vspace{-8mm}
\begin{align}
\label{eq:glasso_condition}
\sqrt{\log n\vcup \log(d-p)} & \bracket{\frac{\sqrt{\log n}}{n^{1/4}}\vcup 
\frac{\sqrt{\log(d-p)}}{(d-p)^{\beta/4}}}\notag \\
\lsim~& \frac{\theta}{(\theta + 8)(\textup{deg})\cdot\kappa^3_{\bSigma}\kappa_{\bGamma}},
\end{align}\vspace{-5mm}
}

\par \noindent
the property $(\wh{\bTheta}^{\textup{non-param}})_{i, j}= 0$ for 
all $(i,j)$ outside the support set $S$, i.e.,
$(i, j)\in S^c$, holds with probability  $1-o(1)$,
where  $\beta \defequal \frac{1}{2}\vcap \frac{a}{4} \vcap \frac{2\alpha + 1}{4}$.
Furthermore, if 

{\small \vspace{-5mm}
\begin{align*}
\min_{(i,j)\in S}\abs{\precmatnonparam_{i,j}} \geq 2\kappa_{\Gamma}& (1 + 8\theta^{-1})
\sqrt{\log n\vcup \log(d-p)} \\
\times & \bracket{\frac{\sqrt{\log n}}{n^{1/4}}\vcup 
\frac{\sqrt{\log(d-p)}}{(d-p)^{\beta/4}}}, 
\end{align*}
\vspace{-5mm}
}

\par \noindent
then $\sign(\wh{\bTheta}^{\textup{non-param}}) = \sign(\precmatnonparam)$ with 
probability $1-o(1)$. 
\end{theorem}

\subsection{Discussions}
Following the same logic as in the discussion 
of the parametric method, we can 
obtain the minimum sample number 
$n$ in terms of dimension $d$, and length $p$ from \eqref{eq:glasso_condition}. 
Treating parameters $\theta$, $\textup{deg}$, 
$\kappa_{\bSigma}$, and $\kappa_{\bGamma}$ as some 
constants, the condition \eqref{eq:glasso_condition}
requires the left-hand side to be constant. 

\paragraph{Minimum Sample Size $n$.}
We can show that  
\vspace{-2mm}
\[
n \gsim \bracket{\textup{deg}}^4 \log^2 n \bracket{\log^2 n \vcup \log^2(d-p)}.
\] 
In contrast, the previous work \cite{liu2009nonparanormal} only 
requires the sample number $n$ to satisfy 
$n\gsim \bracket{\textup{deg}}^4 \log^4(n)$. 
Hence our result experience a loss of up to $\frac{\log^2(d-p)}{\log^2(n)} \vcup 1$. Since this inflation is closely related to the 
dimension $d$, we conclude the loss is 
due to the indirect measurement scheme. 

\paragraph{Minimum Projection Dimension $d$.}
For the dimension $d$, we require 

{\small\vspace{-7mm}
\[
d \geq p + \bracket{\textup{deg}}^{\beta/4}\log^{\beta/8}(d-p)
\bracket{\log^{\beta/8}n\vcup \log^{\beta/8}(d-p)}, 
\]\vspace{-7mm}
} 

which is a slightly larger than the dimension $p$. 
To the best of our knowledge, this is the first condition involving the dimension $d$ for 
the nonparametric learning of the graphical structure. Whether 
we can use a compressive sensing system, namely, $d < p$, for the 
nonparametric method still remains an open-problem.

Having showed that the correct graphical structure 
can be obtained under mild conditions, next we will present some numerical 
experiments to validate our theoretical analysis.

\section{Simulation Results}
This section presents the numerical results, which 
applies to both the synthetic and the real-world 
data. Due to the space limit, 
we only put a subset of our numerical experiments and 
leave the rest to the supplementary material.

\vspace{-2mm}

\subsection{Synthetic Data}
We adopt the classical setting as \cite{tibshirani2015statistical} ($9.5$, P~$252$), where the ground-truth 
matrix $\bTheta^{\natural}$ is set as 
$(i)$ $\bTheta^{\natural}_{ij} = \rho_1$ if $i = j$; 
$(ii)$ $\bTheta^{\natural}_{ij} = \rho_2$ if $|i-j| = 1$;
and $(iii)$ $\bTheta^{\natural}_{ij} = 0$ otherwise. 
The corresponding edge set is denoted as $E^{\natural}$.
Simulations with other types of graphs are referred 
to the supplementary material. 

\paragraph{Parametric Method.}
We set $\rho_1 = 1$ and $\rho_2 = 0.4$. 
First we create samples $\bX^{(i)}$ and then 
mask it by the sensing relation 
$\bY^{(i)} = \bA \bX^{(i)} + \bW^{(i)}$, where
$A_{ij}\iid \normdist(0, 1)$ and $W^{(i)}\sim \normdist(\bZero, \sigma^2 \bI_{n\times n})$.
\footnote{Notice that 
$\Sigma^{\natural}_{ii} \neq 1$ in this setting. Hence we need to adapt the estimation of the 
diagonal elements  of $\hat{\Sigma}$.
 }
Using Alg.~\ref{alg:param_glasso}, we reconstruct 
$\wh{\bTheta}^{\textup{param}}$ and evaluate it with  
the recall rate and precision rate with the results being put in 
Fig.~\ref{fig:synthetic_band_precision}.

%

\begin{figure}[!h]

\centering
\mbox{
\includegraphics[width = 1.6in]{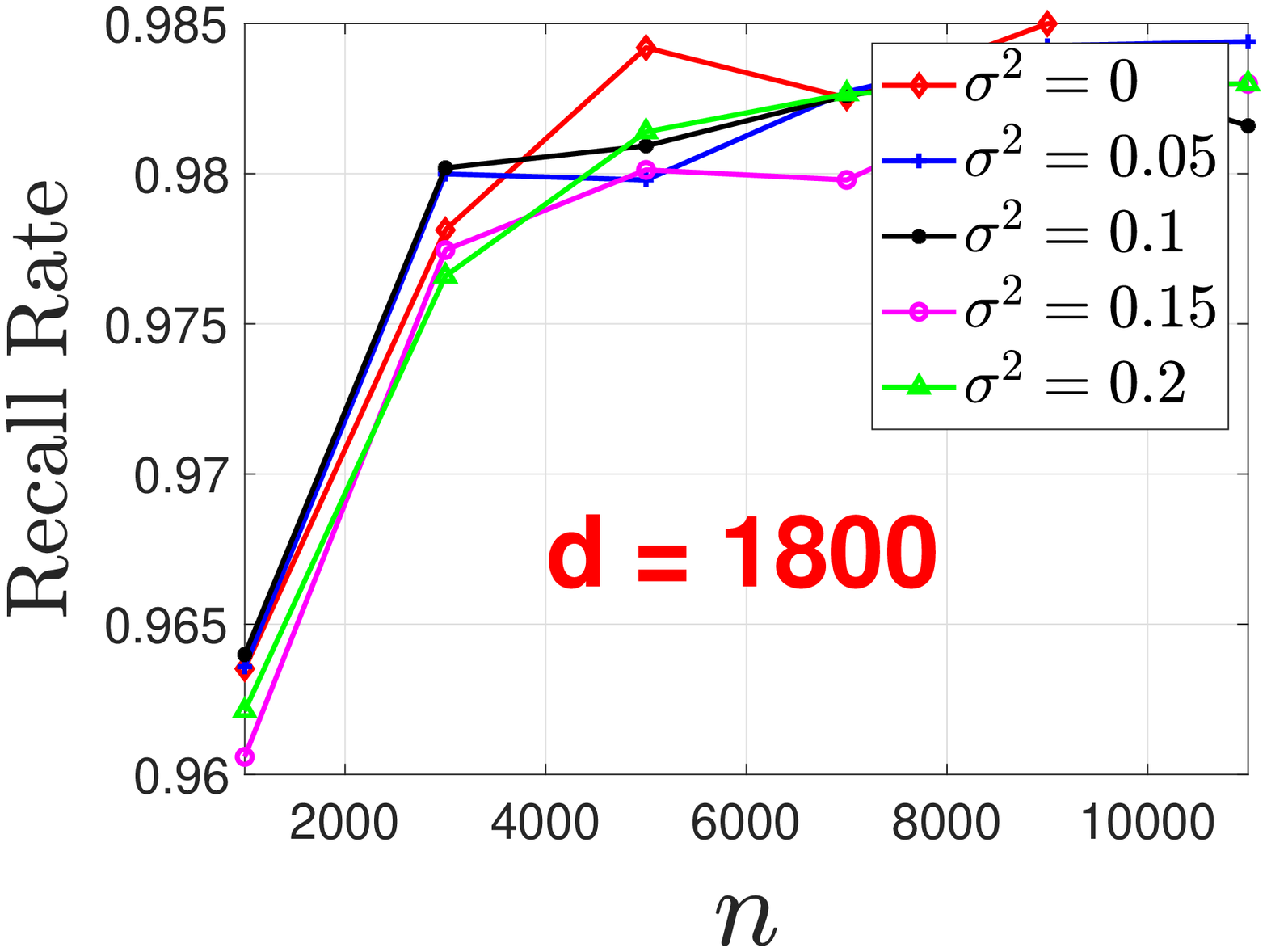}

\includegraphics[width = 1.6in]{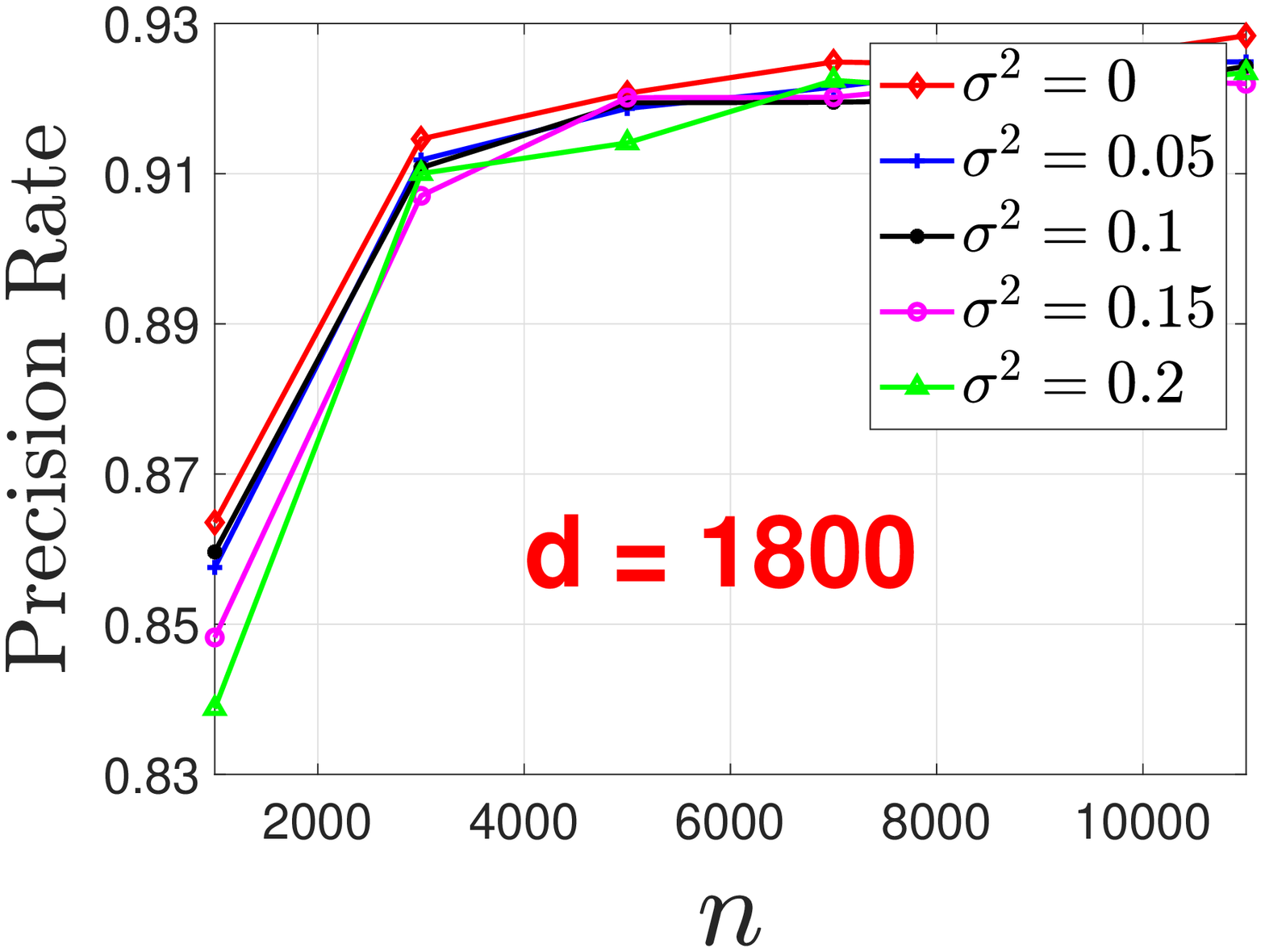}
}
\vspace{-6mm}
\caption{The signal dimension $p$ is fixed as $2000$. We study the  
impact of sample size $n$ 
on the recall rate (\textbf{Left}) on the precision rate (\textbf{Right}).
}
\label{fig:synthetic_band_precision}
\end{figure}

%
%
%
%

We confirm that the edge of graphical model 
can be selected correctly 
with high probability even when the dimension of the 
projection space is much lower than the dimension of the
signal. 
In addition, we notice a threshold effect on 
$n$ when $\sigma^2 = 1$. 
This can be explained by the parameter 
$\tau_{\infty}$ in \eqref{eq:param_tau_inf_def}, 
which contains the term 
$\frac{c_6 \sigma^2\log p}{d} (1 + c_7
(\sqrt{\frac{d}{n}}\vcup \frac{d}{n}))$.
If $d$ is not high enough, this term is still 
lower-bounded by $c_6 \sigma^2 \log p/d$
even as $n\rightarrow \infty$, which 
means a large $\tau_{\infty}$ and further the  
violation of the conditions in Thm.~\ref{thm:param_glasso}. 

From  Fig.~\ref{fig:synthetic_band_precision}, 
we confirm that the edge of graphical model 
can be selected correctly 
with high probability even when the dimension of the 
projection space is much lower than the dimension of the
signal. 
In addition, we notice a threshold effect on 
$n$ when $\sigma^2 = 1$ in \ref{fig:synthetic_band_precision}. 
This can be explained by the parameter 
$\tau_{\infty}$ in \eqref{eq:param_tau_inf_def}, 
which contains the term 
$\frac{c_8 \sigma^2\log p}{d} (1 + c_9
(\sqrt{\frac{d}{n}}\vcup \frac{d}{n}))$.
If $d$ is not high enough, this term is still 
lower-bounded by $\nfrac{c_8 \sigma^2 \log p}{d}$
even as $n\rightarrow \infty$, which 
means a large $\tau_{\infty}$ and further the  
violation of the assumption in Thm.~\ref{thm:param_glasso}. 

Generally speaking, we find larger sample size $n$, higher 
dimension $d$, and lower noise variance $\sigma^2$
contribute to the more accurate edge selection, which 
is consistent with our intuition and verifies Thm.~\ref{thm:param_glasso}.

\vspace{-2mm}
\paragraph{Non-Parametric Method.}
In Fig.~\ref{fig:alg_compare}, we compare the algorithms for the uniform 
distribution where the marginal distribution of 
each entry is uniformly within the region $[0, 1]$. 
We can see that our method has a 
significant improvement when comparing with the method in 
\cite{liu2009nonparanormal} (a widely-used baseline).

\vspace{-2mm}
\begin{figure}[!h]

\centering

\mbox{
\includegraphics[width = 1.6in]{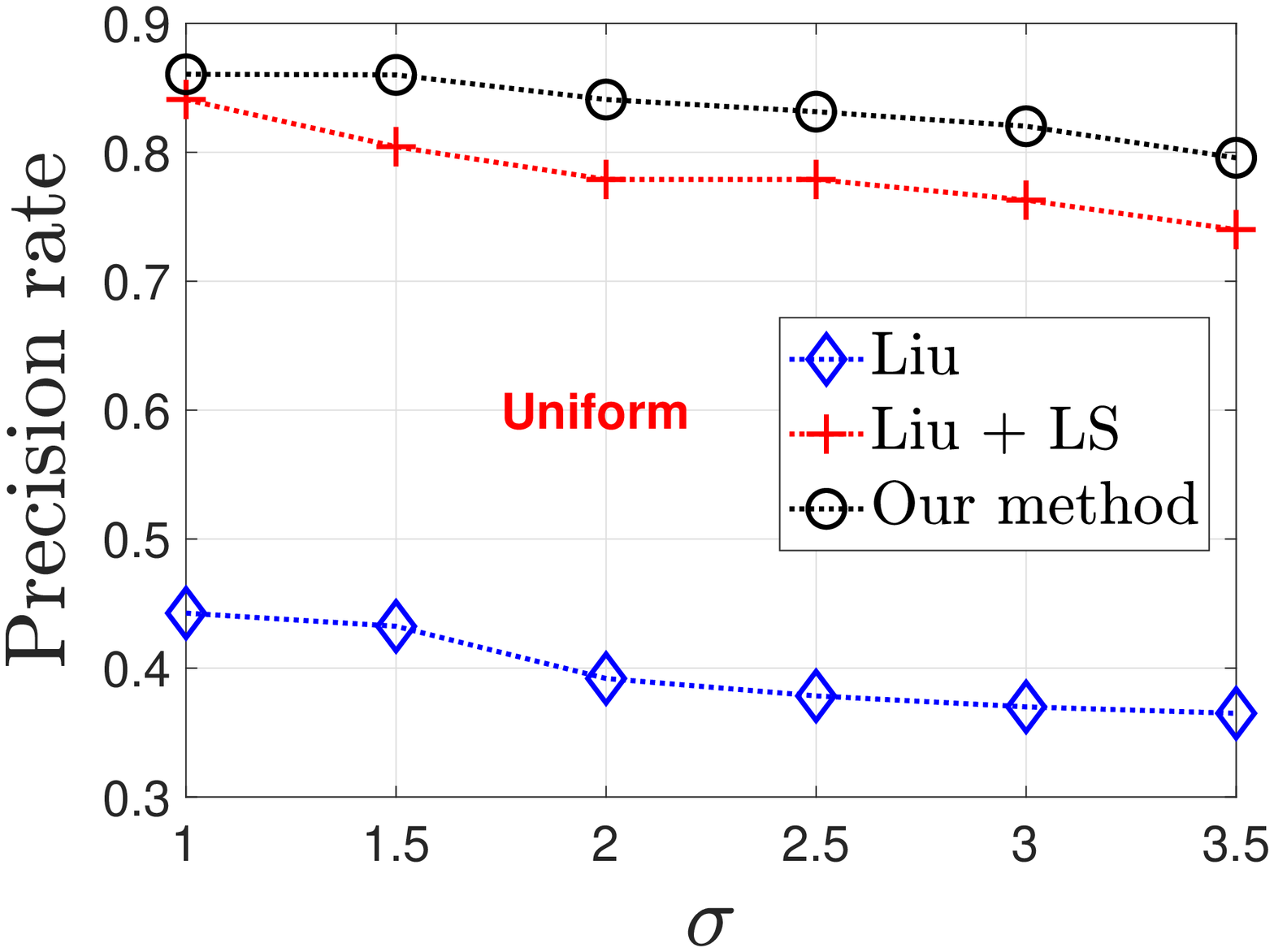}
\includegraphics[width = 1.6in]{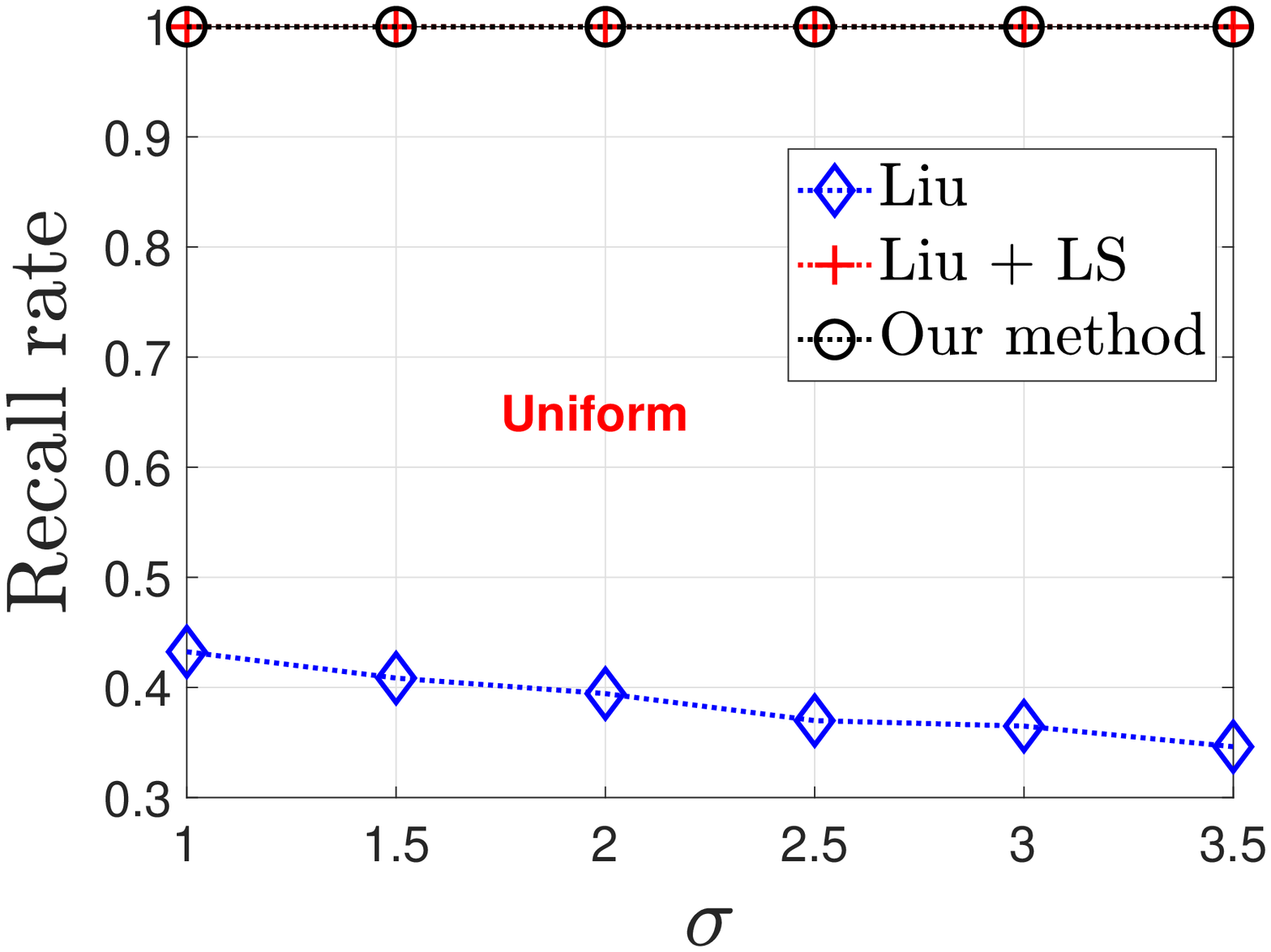}
}

\vspace{-3mm}
\caption{We study the impact of noise with $p = 50$ and $d = 200$. \textbf{Liu} refers to the method in \cite{liu2009nonparanormal}; \textbf{Liu + LS} refers to 
first performing denoising with our method (i.e. least square) followed by \cite{liu2009nonparanormal}, which estimates 
the CDF without density deconvolution. We believe the performance 
gap will increase with a larger noise variance. 
 }
\label{fig:alg_compare}
\end{figure}

In addition, we evaluate the performance with the following three types of 
marginal distribution for the RV $\bX$:  
\vspace{-3mm}
\begin{itemize}[leftmargin=*]
\item 
uniform distribution within the 
region $[0, 1]$;
\vspace{-2mm}
\item 
exponential distribution, i.e., 
$e^{-z}$ for $z\geq 0$;
\vspace{-2mm}
\item 
 Gaussian mixture, 
i.e., $0.25\sum_{i=1}^4 \normdist(\mu_i, 10^{-2})$, where $\mu_i \in \set{\pm 0.25, \pm 0.5}$.
\end{itemize}
\vspace{-3mm}
Due to the spatial limit, we leave the
numerical results to the supplementary material, 
which suggests our algorithm 
has an improvement in both the recall rate and precision rate.


\vspace{-2mm}
\subsection{Real-World Data}
\vspace{-2mm}
We now consider the real-world databases, which 
consists of $5$ databases:
Carolina Breast Cancer (GSE$148426$)  with $2497$ samples
(patients) \cite{Bhattacharya2020.08.14.250902}, 
Lung Cancer (GSE$137140$) with $3924$ samples \cite{asakura2020mirna}, 
Ovarian Cancer (GSE$106817$) with $4046$ samples \cite{yokoi2018integrated},
Colorectal Cancer (GSE$115513$) with $1513$ samples  \cite{slattery2016microrna}, 
and Esophageal Squamous Cell Carcinoma (GSE$122497$)  with $5531$ samples \cite{sudo2019development}. 
Each database is divided into two categories, i.e., \textbf{Healthy group} and \textbf{Patients},
where the measurements 
are given as the concentration of miRNAs. 
The miRNAs are known to have dependency among each other,
i.e., a non-diagonal precision matrix, and hence 
there is an underlying graphical model describing these
dependency structure based on the associated precision 
matrix.

The sensing matrix $\bA \in \RR^{d\times p}$ is assumed to be 
$A_{ij}\iid \normdist(0, 1)$ and the variance of the measurement 
noise is set to one, $W_{i}\iid \normdist(0, 1)$. 
The goal is to reconstruct the underlying dependency graph among miRNAs.
\vspace{-2mm}

\paragraph{Evaluation.}
We adopt the nonparametric method for the evaluation 
due to its wider applications. 
The precision matrix $\bTheta$ learned with noiseless direct 
measurements using the method in \cite{liu2009nonparanormal} is assumed 
to be the ground-truth.  We evaluate the 
performance of our estimator in both the recall rate and precision rate 
of the edge selection. The results are shown in 
Tab.~\ref{tab:real_data_recall}. \footnote{Only a part of the results are put here and
the rest results are left to the supplementary due to the spatial limit.}
These experiments confirm that 
our estimator can obtain the correct dependency relation  
with high-probability.

\vspace{-4mm}
{ 
\begin{table}[!h]
\centering
\caption{
Recall rate and precision rate of edge selection on real-world databases, 
namely, GSE$148426$ \cite{Bhattacharya2020.08.14.250902} and
GSE$137140$ \cite{asakura2020mirna}.
The precision matrix $\bTheta$ 
 learned by direct observations is assumed to be the ground-truth. 
 For more results, please refer to the supplementary material.
}
{\footnotesize
\label{tab:real_data_recall}
\begin{tabular}{@{}ccrlcc@{}}\toprule
 & \multicolumn{2}{c}{Healthy group} & \phantom{abc}& \multicolumn{2}{c}{Unhealthy group} 
\\
\cmidrule{2-3} \cmidrule{5-6} 
\textbf{$d/p$} & \makecell{ GSE \\$148426$} & 
\makecell{GSE\\$137140$} &&\makecell{ GSE \\$148426$} & 
\makecell{GSE\\$137140$} \\ \midrule
\textbf{Recall Rate} \\ 
$2$  & $0.9494$ & $0.8892$ 
&&   $0.9424$ & $0.9692$ \\
$5$ & $1$ & $0.9950$ 
&&  $1$ & $0.9692$
\\
$10$ & $1$ & $1$ 
&&   $1$ & $1$
\\
$12$ & $1$ & $1$
&&   $1$ & $1$ 
\\
$15$ & $1$ & $1$
&&   $1$ & $1$
\\ 
$20$ & $1$ & $1$ 
&&   $1$ & $1$  \\
\hdashline  \vspace{1mm}
\textbf{Precision Rate} \\
$2$  & $1$ & $1$ 
&&   $0.9704$ & $0.9692$ 
\\
$5$ & $0.9080$ & $0.9900$ 
&&  $0.9205$ & $0.9692$ 
\\
$10$ & $0.9080$ & $0.9341$ 
&&   $0.9205$ & $0.9420$ 
\\
$12$ & $0.9080$ & $0.9475$ 
&&   $0.9329$ & $0.9420$ 
\\
$15$ & $0.9080$ & $0.9566$ 
&&   $0.9456$ & $0.9420$ 
\\ 
$20$ & $0.9518$ & $0.9613$ 
&&   $0.9586$ & $0.9420$ \\
\bottomrule
\end{tabular}
}
\end{table}
}


%
\vspace{-6mm}
\section{Conclusions}
\vspace{-2mm}
This is the first work on learning 
of the graphical structure with noisy indirect measurements, where 
both the parametric and non-parametric methods are investigated. 
For the parametric method, we considered the 
Gaussian graphical model and learned the graphical 
structure by using only indirect low-dimensional observations, 
i.e., the observations obtained via compressive 
sensing of the desired signal. 
For the non-parametric method, we relaxed the gaussian distribution to 
the nonparanormal distribution. 
We established a non-asymptotic uniform bound on the 
errors of the CDF estimation. 
To the best of our knowledge, this is the first such results 
on the CDF error bounds when only limited information exists 
regarding the noise distribution.  
For both scenarios, we 
showed our estimator can generate the 
correct graphical structure under mild conditions, 
 from which the relation between the 
sample number $n$, the signal dimension $p$, and the projection 
dimension $d$ are obtained. 
In addition, we provided numerical experiments, using 
both synthetic and real world miRNA data, to corroborate 
the correctness of our theoretical results.
%
%
\bibliography{icml_graph}
\bibliographystyle{icml2022}

\newpage
\onecolumn
%
%
\section{Analysis of Covariance Matrix Estimation via 
Parametric Method}
For the convenience of analysis, we rescale
the sensing matrix $\wt{\bA}$ such that 
$\wt{A}_{ij} \sim \normdist(0, d^{-1})$. 

\begin{lemma}
\label{thm:cov_mat_inf_norm}
Consider the covariance estimator $\wh{\bSigma}^{\textup{param}}_n$ which reads
\begin{align}
\label{eq:cov_mat_estim}
\wh{\bSigma}^{\textup{param}}_n = 
\bI + \frac{d}{d+1}
\Bracket{\wt{\bA}^{\rmt}\bracket{\frac{1}{n}\sum_{i=1}^n \bY^{(i)}\bY^{(i)\rmt}}\wt{\bA}}_{\textup{off}},
\end{align}
where $(\cdot)_{\textup{off}}$ denotes the operation of 
picking non-diagonal entries. We then have 
$\infnorm{\wh{\bSigma}^{\textup{param}}_n - \bSigma^{\natural}}\leq \tau_{\infty}$ 
holding with probability at least 
$1 - c_0p^{-1} - c_1 p^2e^{-c_2 d} - c_3p^2e^{-c_4p}$.
\end{lemma}

We begin the analysis by redefining the following events, 
\begin{align*}
\calE_1 &\defequal \
\set{1-c_0 \sqrt{\frac{\log p}{d}}\leq \norm{\wt{\bA}_i}{2} \leq 1 + c_0 \sqrt{\frac{\log p}{d}},~~\forall~i};  \\
\calE_2 &\defequal \
\set{|\la \wt{\bA}_i, \wt{\bA}_j\ra| \lsim \sqrt{\frac{\log p}{d}} \vcup \frac{\log p}{d},~~\forall~i\neq j}; \\
\calE_3 &\defequal \set{\norm{\wh{\bSigma}^{\textup{param}}_n - \bSigma^{\natural}}{\infty} \lsim \sqrt{\frac{\log p}{n}}}; \\
\calE_4(\bB) &\defequal \
\set{\bigg\|\sum_{\ell\neq 1} B_{\ell}\wt{\bA}_{\ell}\bigg\|_{2} \lsim 
\sqrt{\sum_{\ell\neq 1}B_{\ell}^2}},~~\textup{where } \bB\in \RR^p
\textup{ is a fixed vector}.
\end{align*}
Additionally, we define $\Psi(\calE)$ as $\Expc \Ind(\calE)$. 
Moreover, our analysis focuses on the region when $d \gg \log p$.

\subsection{Main Structure}
Having collected all the lemmas, we turn to the proof of 
Lemma~$1$. 
Notice that by the definition of our estimator, 
$\infnorm{\wh{\bSigma}^{\textup{param}}_n- \bSigma^{\natural}} = \max_{i\neq j} \abs{
\bracket{\wh{\bSigma}^{\textup{param}}_n}_{i,j} - \bSigma^{\natural}_{i, j} }$. Hence, we only consider the off-diagonal entries, which gives
\begin{align*}
& \norm{\wh{\bSigma}^{\textup{param}}_n - \bSigma^{\natural}}{\textup{off}, \infty}
= \norm{\dfrac{d}{d+1}\wt{\bA}^{\rmt}\bracket{\sum_{\ell=1}^n \
\bY^{(\ell)} \bY^{(\ell)\rmt}} \wt{\bA} - \bSigma^{\natural} }{\textup{off}, \infty}  \\
\leq~& 2\norm{\frac{d}{d+1}\wt{\bA}^{\rmt}\wt{\bA} \bSigma^{\textup{param}}_n \wt{\bA}^{\rmt}\wt{\bA} - \bSigma^{\natural}}{\textup{off}, \infty}
+ \dfrac{2d}{n(d+1)}\norm{\wt{\bA}^{\rmt}\bracket{\sum_{\ell=1}^n \bW^{(\ell)} \bW^{(\ell)\rmt}} \wt{\bA}}{\textup{off},\infty} \\
\stackrel{\cirone}{=}~& \dfrac{2d}{d+1}\
\underbrace{\norm{\wt{\bA}^{\rmt}\wt{\bA} \bSigma^{\textup{param}}_n \wt{\bA}^{\rmt}\wt{\bA} 
- \Expc_{\wt{\bA}}\bracket{\wt{\bA}^{\rmt}\wt{\bA} \bSigma^{\textup{param}}_n \wt{\bA}^{\rmt}\wt{\bA}} }{\textup{off}, \infty}}_{\vartheta_1} \\
+~& \frac{2d}{d+1}\underbrace{\norm{ \Expc_{\wt{\bA}}\bracket{\wt{\bA}^{\rmt}\wt{\bA} \bSigma^{\textup{param}}_n \wt{\bA}^{\rmt}\wt{\bA}} - \Expc_{\wt{\bA}, \bX}\bracket{\wt{\bA}^{\rmt}\wt{\bA} \bSigma^{\textup{param}}_n \wt{\bA}^{\rmt}\wt{\bA}}}{\textup{off}, \infty}}_{\vartheta_2} \\
+~& \dfrac{2d}{n(d+1)}\underbrace{\norm{\wt{\bA}^{\rmt}\bracket{\sum_{\ell=1}^n \bW^{(\ell)} \bW^{(\ell)\rmt}} \wt{\bA}}{\textup{off},\infty}}_{\vartheta_3}, 
\end{align*}
where $\bSigma^{\textup{param}}_n$ is defined as $n^{-1}\bracket{\sum_{\ell=1}^n \bX^{(\ell)}\bX^{(\ell)\rmt}}$,
and in $\cirone$ we use 
$d/(d+1)\Expc_{\wt{\bA}, \bX}\bracket{\wt{\bA}^{\rmt}\wt{\bA} \bSigma^{\textup{param}}_n \wt{\bA}^{\rmt}\wt{\bA}} = \bSigma^{\natural}$.
Then we separately upper-bound $\vartheta_1,~\vartheta_2$, 
and $\vartheta_3$
conditional on the event $\bigcap_{i=1}^3 \calE_i$.
For the conciseness in notations, define $\bZ = \wt{\bA}^{\rmt}\wt{\bA} \bSigma^{\textup{param}}_n \wt{\bA}^{\rmt}\wt{\bA}$. 
Therefore $\vartheta_1 = \norm{\bZ - \Expc \bZ}{\textup{off}, \infty}$. 
For an arbitrary entry $Z_{i,j}$, we can expand it as 
\begin{align*}
Z_{i,j} = \sum_{\ell_1,\ell_2} (\bSigma^{\textup{param}}_n)_{\ell_1, \ell_2} 
\la \wt{\bA}_i, \wt{\bA}_{\ell_1}\ra \la \wt{\bA}_j, \wt{\bA}_{\ell_2}\ra,~~\
i\neq j. 
\end{align*}
Compared to the existing work \cite{ravikumar2011high}, our analysis of 
$Z_{ij}$ involves fourth-order Gaussian chaos \cite{talagrand2014upper}, which 
exhibits heavy tails, and constructs the major obstacle. 

\par \noindent 
\textbf{Stage I: Bounding $\vartheta_1$.}
To obtain upper bound on $\vartheta_1$, 
we first adopt the union bound and obtain 
\begin{align}
\label{eq:sn_inf_main_prob}
\Prob\bracket{\vartheta_1 \geq \delta} \
= \Prob\bracket{\max_{i\neq j} \abs{Z_{i,j} - \Expc Z_{i,j}} \geq \delta}
\leq \sum_{i \neq j} \Prob\bracket{\abs{Z_{i,j} - \Expc Z_{i,j}} \geq \delta,~i\neq j}.
\end{align}

Then our focus is to bound the probability 
$\Prob\bracket{\abs{Z_{i,j} - \Expc Z_{i,j}} \geq \delta,~i\neq j}$. 
Without loss of generality, assume that $i = 1$ and $j = 2$ and
expand $Z_{1,2}$ as $\sum_{i= 1}^5 T_i$, which reads
\begin{align*}
T_1 & \defequal  \bracket{\bSigma^{\textup{param}}_n}_{1, 2}\norm{\wt{\bA}_1}{2}^2 \norm{\wt{\bA}_2}{2}^2;  \\
T_2 & \defequal \bracket{\bSigma^{\textup{param}}_n}_{2, 1}\bracket{\la \wt{\bA}_1, \wt{\bA}_2\ra}^2; \\
T_3 & \defequal \sum_{\ell\neq 1}\bracket{\bSigma^{\textup{param}}_n}_{2, \ell} \norm{\wt{\bA}_2}{2}^2 \la \wt{\bA}_1, \wt{\bA}_{\ell}\ra + 
\sum_{\ell \neq 2}\bracket{\bSigma^{\textup{param}}_n}_{\ell, 1}\norm{\wt{\bA}_1}{2}^2 \la \wt{\bA}_2, \wt{\bA}_{\ell}\ra; \\
T_4 & \defequal \sum_{\ell \neq 1, 2}\bracket{\bSigma^{\textup{param}}_n}_{\ell, \ell} \la \wt{\bA}_1, \wt{\bA}_{\ell}\ra
\la \wt{\bA}_2, \wt{\bA}_{\ell}\ra;  \\
T_5 & \defequal \sum_{\substack{\ell_1, \ell_2 \neq 1, 2 \\ \ell_1 \neq \ell_2}}
\bracket{\bSigma^{\textup{param}}_n}_{\ell_1, \ell_2} \la \wt{\bA}_1, \wt{\bA}_{\ell_1}\ra \la \wt{\bA}_2, \wt{\bA}_{\ell_2}\ra.
\end{align*}

Now, we separately bound the deviations 
$\abs{T_i - \Expc T_i}$, $1\leq i \leq 5$. 
First, we have
\begin{align}
\label{eq:sn_inf_norm_t1}
& \abs{T_1 - \Expc T_1} = 
|\bracket{\bSigma^{\textup{param}}_n}_{1, 2}|\times 
\abs{\norm{\wt{\bA}_1}{2}^2 \norm{\wt{\bA}_2}{2}^2 - 1} \notag \\
=~& |\bracket{\bSigma^{\textup{param}}_n}_{1, 2}|\times 
\abs{\norm{\wt{\bA}_1}{2}\norm{\wt{\bA}_2}{2} - 1}\times \
\bracket{\norm{\wt{\bA}_1}{2} \norm{\wt{\bA}_2}{2} + 1} \notag \\
\stackrel{\cirtwo}{\leq}~&
|\bracket{\bSigma^{\textup{param}}_n}_{1, 2}|\times 
\bracket{\abs{\norm{\wt{\bA}_1}{2}-1}\norm{\wt{\bA}_2}{2} + 
\abs{\norm{\wt{\bA}_2}{2} - 1}
}\times
\bracket{\bracket{1 + c_0\sqrt{{\log p}/{d}}}^2 + 1} \notag \\
\stackrel{\cirthree}{\leq}~& |\bracket{\bSigma^{\textup{param}}_n}_{1, 2}|\times
c_0\sqrt{\log p/d}(2 + c_0\sqrt{\log p/d}) 
\times \bracket{\bracket{1 + c_0\sqrt{{\log p}/{d}}}^2 + 1} \notag \\
\stackrel{\cirfour}{\lsim}~& |\bracket{\bSigma^{\textup{param}}_n}_{1, 2}|\sqrt{{\log p}/{d}}
\stackrel{\cirfive}{\lsim} \sqrt{\frac{\log p}{d}}\bracket{\abs{\Sigma^{\natural}_{1, 2}} + \sqrt{\frac{\log p}{n}}},
\end{align}
where in $\cirtwo$ and $\cirthree$ we condition on 
event $\calE_1$, $\cirfour$ is due to $d \gg \log p$, 
and $\cirfive$ is by the definition of event $\calE_3$.
\par  
For $\abs{T_2 -\Expc T_2}$, by invoking Lemma~\ref{lemma:two_two_same},  
\begin{align}
\label{eq:sn_inf_norm_t2}
& \abs{T_2 - \Expc T_2} \lsim 
|\bracket{\bSigma^{\textup{param}}_n}_{2, 1}|\times \
\bracket{\frac{\log p}{d}+ \frac{\sqrt{\log p}}{d^{3/2}}}
\stackrel{\cirsix}{\lsim} \abs{\bracket{\bSigma^{\textup{param}}_n}_{2, 1}}\frac{\log p}{d} \notag \\
\stackrel{\cirseven}{\lsim}~& \
\frac{\log p}{d}\bracket{\abs{\Sigma^{\natural}_{2, 1}} + \sqrt{\frac{\log p}{n}}}, 
\end{align}
holds with probability exceeding $1-2p^{-3}$, 
where $\cirsix$ is because $d\gg \log p$, and $\cirseven$ is due to event 
$\calE_3$. 
\par
We continue to bound $\abs{T_3-\Expc T_3}$ by 
\begin{align}
\label{eq:sn_inf_norm_t3}
& \abs{T_3 - \Expc T_3} \leq \
\norm{\wt{\bA}_1}{2}^2 \times \bigg|\sum_{\ell\neq 2}\bracket{\bSigma^{\textup{param}}_n}_{1, \ell} \la \wt{\bA}_2, \wt{\bA}_{\ell}\ra\bigg| + \
\norm{\wt{\bA}_2}{2}^2\times \
\bigg|\sum_{\ell \neq 1}\bracket{\bSigma^{\textup{param}}_n}_{\ell, 2} \la \wt{\bA}_1, \wt{\bA}_{\ell}\ra
\bigg| \notag \\
\leq ~& \bracket{1 + c_0\sqrt{\frac{\log p}{d}}}^2 
 \bigg|\sum_{\ell\neq 2}\bracket{\bSigma^{\textup{param}}_n}_{1, \ell} \la \wt{\bA}_2, \wt{\bA}_{\ell}\ra\bigg| + 
\bracket{1 + c_0\sqrt{\frac{\log p}{d}}}^2 \
\bigg|\sum_{\ell \neq 1}\bracket{\bSigma^{\textup{param}}_n}_{\ell, 2} \la \wt{\bA}_1, \wt{\bA}_{\ell}\ra
\bigg| \notag \\
\stackrel{\cireight}{\lsim}~&
\bracket{1 + c_0\sqrt{\frac{\log p}{d}}}^2 \times 
\sqrt{\frac{\log p}{d}}
\bracket{\sqrt{\sum_{\ell \neq 2}\bracket{\bSigma^{\textup{param}}_n}^2_{1, \ell}} + \
\sqrt{\sum_{\ell \neq 1}\bracket{\bSigma^{\textup{param}}_n}^2_{\ell, 2}} } \notag \\
\stackrel{\cirnine}{\lsim}~& \sqrt{\frac{\log p}{d}}
\bracket{\sqrt{\sum_{\ell \neq 2}\bracket{\bSigma^{\textup{param}}_n}^2_{1, \ell}} + \
\sqrt{\sum_{\ell \neq 1}\bracket{\bSigma^{\textup{param}}_n}^2_{\ell, 2}} } 
\stackrel{\cira}{\lsim} \sqrt{\frac{\log p}{d}}
\bracket{ \norm{\bSigma^{\natural}_1}{2} + \norm{\bSigma^{\natural}_2}{2} +
\sqrt{\frac{p\log p}{n}}   }
\end{align}
which holds for probability exceeding $1- c_0 p^{-3} - c_1e^{-c_2 d}$, 
where in $\cireight$ we invoke Lemma~\ref{lemma:two_diff}, 
in $\cirnine$ we use that $\bSigma^{\textup{param}}_n$ is symmetric 
and $d \gg \log p$, and  $\cira$ is due to  
event $\calE_3$. 
\par 
To bound $\abs{T_4 - \Expc T_4}$, we invoke Lemma~
\ref{lemma:two_one_one}, 
\begin{align}
\label{eq:sn_inf_norm_t4}
&\abs{T_4 - \Expc T_4} \lsim 
\frac{\abs{\sum_{\ell > 2 }\bracket{\bSigma^{\textup{param}}_n}_{\ell, \ell} }}{d}\sqrt{\frac{\log p}{d}}
+ \frac{\sqrt{\log p}}{d} 
\bigg[\sqrt{\sum_{\ell > 2}\bracket{\bSigma^{\textup{param}}_n}_{\ell, \ell}^2}\vcup \
\bracket{\sqrt{\log p}(\max_{\ell}|\bracket{\bSigma^{\textup{param}}_n}_{\ell, \ell}|)}\bigg]\notag \\
\stackrel{\cirb}{\lsim}~&
\frac{p\sqrt{\log p}}{d^{3/2}}\bracket{1 + c_0\sqrt{\frac{\log p}{n}}}
+ \frac{\sqrt{p\log p}}{d}
\underbrace{\Bracket{\bracket{1 + c_0\sqrt{\frac{\log p}{n}}} \vcup \sqrt{\frac{\log p}{p}}\bracket{1 + c_0\sqrt{\frac{\log p}{n}}}}}_{1 + c_0\sqrt{\frac{\log p}{n}}},
\end{align}
holds with probability exceeding $1-2p^{-3}$, where $\cirb$ is by definition of 
event $\calE_3$. 
\par 
The deviation $\abs{T_5 - \Expc T_5}$ is bounded 
via Lemma~\ref{lemma:four_diff}, which gives
\begin{align}
\label{eq:sn_inf_norm_t5}
\abs{T_5 - \Expc T_5} 
\lsim \frac{\log p}{d}\bracket{1 + c_0\sqrt{\frac{p}{d}}}^2 
\Norm{\bSigma^{\natural}}{\textup{off}, \textup{F}}
\stackrel{\circc}{\lsim}
\frac{\log p}{d}\bracket{1 + \frac{c_1 p}{d}} \bracket{\Norm{\bSigma^{\natural}}{\textup{off}, \textup{F}}+  p\sqrt{\frac{\log p}{n}}},
\end{align}
with probability exceeding $1-4p^{-3} - e^{-c_0 p}$,
where in $\circc$ we condition on event $\calE_3$. 
\par 
Combining (\ref{eq:sn_inf_norm_t1}),
(\ref{eq:sn_inf_norm_t2}), 
(\ref{eq:sn_inf_norm_t3}),
(\ref{eq:sn_inf_norm_t4}),
and (\ref{eq:sn_inf_norm_t5}), we conclude that 
\begin{align*}
\abs{Z_{i,j} - \Expc Z_{i, j}} 
\lsim~& 
\sqrt{\frac{\log p}{d}}\bracket{
\norm{\bSigma^{\natural}_i}{2} + 
\norm{\bSigma^{\natural}_j}{2}} + 
\frac{\log p}{d}\bracket{1 + \frac{c_1p}{d}}
\Norm{\bSigma^{\natural}}{\textup{off}, \textup{F}} \\
+~& \frac{\log p\sqrt{p}}{\sqrt{nd}} + \
\frac{c_0 p(\log p)^{3/2}}{n^{1/2}d}\bracket{1 + \frac{c_1p}{d}}
+ \frac{\sqrt{p\log p}}{d} +\
\frac{p\sqrt{\log p}}{d^{3/2}},
\end{align*}
with probability $1-c_0p^{-3} - c_1e^{-c_2d}-c_3e^{-c_4p}$. 
We conclude the proof by plugging into (\ref{eq:sn_inf_main_prob})
and $\Norm{\bSigma^{\textup{param}}_n - \bSigma^{\natural}}{\textup{off}, \infty} \leq 
\infnorm{\bSigma^{\textup{param}}_n - \bSigma^{\natural}} \lsim \sqrt{\log p/n}$
according to event $\calE_3$.

\par \noindent
\textbf{Stage II: Bounding $\vartheta_2$.}
We rewrite $\vartheta_2$ as $\norm{\bSigma^{\textup{param}}_n - \bSigma^{\natural}}{\textup{off}, \infty}$ and 
invoke event $\calE_3$.

\par \noindent
\textbf{Stage III: Bounding $\vartheta_3$.}
This stage is completed by invoking Lemma~\ref{lemma:noise_term}. 
Combing the above three stages will then yield the proof.

\subsection{Supporting Lemmas}
\par \vsp 
We first compute the values of $\Psi(\calE_i)$, $1\leq i \leq 3$. 
\begin{lemma}\label{lemma:event1}
$\Psi(\calE_1) \geq 1-2p^{-1}$. 	
\end{lemma}

\begin{proof}
We conclude that 
\begin{align*}
& \Prob\bracket{\abs{\norm{\wt{\bA}_i}{2} - 1}\geq c_0 
\sqrt{\frac{\log p}{d}},~\exists 1\leq i \leq p} \leq 
p\Prob\bracket{\abs{\norm{\wt{\bA}_i}{2} - 1}\geq c_0 
\sqrt{\frac{\log p}{d}}} 
\stackrel{\cirone}{\leq}4p\exp\bracket{-2d\times \frac{\log p}{d}}
= 2p^{-1},
\end{align*}
where $\cirone$ is due to the properties of $\chi^2$ distribution.  
	
\end{proof}

\par 

\begin{lemma}
\label{lemma:event2}
Conditional on $\calE_1$, 
we have $\Psi(\calE_2) \geq 1- 2p^{-1}$. 	
\end{lemma}

\begin{proof}
Due to the independence between 
$\wt{\bA}_i$ and $\wt{\bA}_j$, where $i\neq j$. 
We can condition on $\wt{\bA}_i$ and 
view $\la \wt{\bA}_i, \wt{\bA}_j \ra$
as a Gaussian RV with mean zero 
and variance $d^{-1}\norm{\wt{\bA}_i}{2}^2$, namely, 
$\normdist\bracket{0, d^{-1}\norm{\wt{\bA}_i}{2}^2}$.
Then we conclude that 
\begin{align*}
& \Psi\bracket{\br{\calE}_2~|~\calE_1} \
\stackrel{\cirone}{\leq} 
p^2 \Psi\bracket{|\la \wt{\bA}_i, \wt{\bA}_j\ra| \geq \delta~|~ \calE_1}
\stackrel{\cirtwo}{\leq} 2p^2\bracket{\Phi\bracket{-\frac{\sqrt{d}\delta}{\norm{\wt{\bA}_i}{2}}}\times \Psi(\calE_1)} \\
\stackrel{\cirthree}{\leq} ~& 2p^2\bracket{\exp\bracket{-\frac{d\delta^2}{2\norm{\wt{\bA}_i}{2}^2}}\times
\Psi(\calE_1)} \stackrel{\cirfour}{\leq} 
2p^2\exp\bracket{-\frac{d\delta^2}{2\bracket{1+ c_0\sqrt{\log p/d}}^2}},
\end{align*}
where $\cirone$ is due to the union bound, in $\cirtwo$ we
denote $\Phi(t) = 1/\sqrt{2\pi}\int_{-\infty}^t e^{-x^2/2}dx$, the 
CDF of the normal distribution, 
$\cirthree$ is because $\Phi(x) \leq e^{-x^2/2}$, 
and $\cirfour$ is according to event $\calE_1$. 
\par 
In the end, we complete the proof by
setting $\delta$ as 
$c_1\bracket{\sqrt{\frac{\log p}{d}} \vcup \frac{\log p}{d}}$, 
which yields
$\Psi(\br{\calE}_2~|~ \calE_1) \leq 2p^{-1}$.
\end{proof}
\begin{lemma}
\label{lemma:event3}
$\Psi(\calE_3) \geq 1 - 4p^{-1}$. 	
\end{lemma}

\begin{proof}
The proof can be found in the proof of Thm.~$1$ and Thm.~$4$
in \cite{cai2011constrained}. 
\end{proof}

\begin{lemma}
\label{lemma:event4}
For an arbitrary fixed vector $\bB\in \RR^p$, we have
$\Psi(\calE_4(\bB)) \geq 1 - e^{-0.8d}$.  
\end{lemma}

\begin{proof}
Notice that 
$\sum_{\ell \neq 1} B_{\ell}\wt{\bA}_{\ell}$ 
is a vector satisfying 
$\normdist\bracket{\bZero, d^{-1}\bracket{\sum_{\ell\neq 1}B_{\ell}^2 }\bI}$. Hence, 
$\frac{d}{\sum_{\ell\neq 1}B_{\ell}^2}\|\sum_{\ell\neq 1} B_{\ell}\wt{\bA}_{\ell}\|_2^2$ is a $\chi^2$ RV with freedom 
$d$, which suggests
\begin{align*}
\Prob\bracket{\bigg\|\sum_{\ell\neq 1} B_{\ell}\wt{\bA}_{\ell}\bigg\|_{2}^2 \geq 4 \bigg(\sum_{\ell\neq 1}B_{\ell}^2 \bigg)} = 
\Prob\bracket{\frac{d}{\sum_{\ell\neq 1}B_{\ell}^2}\bigg\|\sum_{\ell\neq 1} B_{\ell}\wt{\bA}_{\ell}\bigg\|_{2}^2 \geq 4d} 
\stackrel{\cirone}{\leq} \exp\bracket{\frac{d}{2}\bracket{\log 4 - 3}} \leq  
e^{-0.8d},
\end{align*}
where $\cirone$ is due to the properties of $\chi^2$ distribution. 
\end{proof}


\begin{lemma}
\label{lemma:two_two_same}
Conditional on event $\calE_1$, 
\begin{align*}
\abs{\sum_{\ell\neq 1}B_{\ell}\la \wt{\bA}_1, \wt{\bA}_{\ell}\ra^2
- d^{-1}\bigg(\sum_{\ell\neq 1}B_{ \ell}\bigg) }
\lsim \dfrac{1}{d}
\bracket{
\sqrt{\log p}\sqrt{\sum_{\ell \neq 1} B_{\ell}^2}
\vcup (\log p)\bracket{\max_{\ell\neq 1}\abs{B_{\ell}}}
} + \
\frac{\sum_{\ell \neq 1} B_{\ell}}{d}\sqrt{\frac{\log p}{d}}
\end{align*}	
holds with probability exceeding $1-2p^{-3}$, 
for an arbitrary fixed vector $\bB\in \RR^p$.
\end{lemma}

\begin{proof}
First, we decompose the above term as 
\begin{align*}
&\abs{\sum_{\ell\neq 1}B_{\ell}\la \wt{\bA}_1, \wt{\bA}_{\ell}\ra^2 
- d^{-1}\bigg(\sum_{\ell\neq 1}B_{\ell}\bigg) } \\
\leq~& \
\frac{\norm{\wt{\bA}_1}{2}^2}{d}
\abs{
\sum_{\ell\neq 1}B_{\ell}\bracket{\
\frac{\sqrt{d}\la \wt{\bA}_1, \wt{\bA}_{\ell}\ra}{\norm{\wt{\bA}_1}{2} }}^2 - \
\sum_{\ell\neq 1}B_{\ell}}
+ 
d^{-1}\bigg|\sum_{\ell\neq 1}B_{\ell}\bigg|\
\abs{
\norm{\wt{\bA}_1}{2}^2 - 1} \\
\stackrel{\cirone}{\leq} ~& \
\dfrac{1}{d}\bracket{1 + c_0\sqrt{\frac{\log p}{d}}}
\underbrace{\abs{
\sum_{\ell\neq 1}B_{\ell}\bracket{\
\frac{\sqrt{d}\la \wt{\bA}_1, \wt{\bA}_{\ell}\ra}{\norm{\wt{\bA}_1}{2} }}^2 - \
\sum_{\ell\neq 1}B_{\ell}}}_{T} + \
\frac{\abs{\sum_{\ell \neq 1} B_{\ell}}}{d}\sqrt{\frac{\log p}{d}} \\
\stackrel{\cirtwo}{\lsim} ~& \frac{T}{d} + \frac{\abs{\sum_{\ell \neq 1} B_{\ell}}}{d}\sqrt{\frac{\log p}{d}}
\end{align*}	
where $\cirone$ is due to the definition of event $\calE_1$, 
and $\cirtwo$ is because $d \gg \log p$. 
Our following analysis focuses on upper-bounding 
$T$. 
First we  define $\bLambda$ as 
$\bLambda = \Bracket{\Lambda_2~\cdots~\Lambda_d}^{\rmt}$, 
where $\Lambda_i = \sqrt{d}\la \wt{\bA}_1, \wt{\bA}_{\ell}\ra/\norm{\wt{\bA}_1}{2}$.
Then we can rewrite $T$ as 
\begin{align*}
T = 
\sum_{\ell\neq 1}B_{\ell} \Lambda_i^2 = 
\bLambda^{\rmt}\diag(B_{\ell})_{\ell\neq 2}\bLambda.
\end{align*}

Due to the independence between $\wt{\bA}_1$ and $\wt{\bA}_{\ell}$, $\ell \neq 1$, 
we can condition $\wt{\bA}_1$ and can view 
$\Lambda_i$ as a Gaussian RV
satisfying $\normdist(0, 1)$.   
Invoking the Hanson-Wright inequality 
(Theorem 6.2.1 in \cite{vershynin2016high}), we conclude that 
\begin{align*}
\Prob\bracket{T \geq \delta} \leq \
2\exp\bracket{-c\bracket{\
\frac{\delta^2}{\Fnorm{\diag(B_{\ell})_{\ell\neq 1}}^2} \
\vcap \frac{\delta}{\Opnorm{\diag(B_{\ell})_{\ell\neq 1}}} 
}}.
\end{align*}
Setting $\delta$ as  
$
\bracket{
c_0\sqrt{\log p}\sqrt{\sum_{\ell \neq 1} B_{\ell}^2}}
\vcup c_1\bracket{\log p\times \max_{\ell\neq 1}\abs{B_{\ell}}}, 
$
we conclude that 
$\Prob(T \geq \delta) \leq 2p^{-3}$ and complete the proof. 	
\end{proof}

\begin{lemma}
\label{lemma:two_diff}
Given a fixed vector $\bB \in \RR^{p}$, 
\begin{align*}
\bigg|\sum_{\ell\neq 1}B_{\ell} \la \wt{\bA}_1, \wt{\bA}_{\ell}\ra\bigg|
\lsim \sqrt{\frac{\log p}{d}}\sqrt{\sum_{\ell\neq 1}B_{\ell}^2}, 
\end{align*}	
holds with probability at least $1 - p^{-3} - e^{-0.8d}$. 
\end{lemma}

\begin{proof}
Due to the independence between $\wt{\bA}_1$ and 
$\wt{\bA}_{\ell}$, where $\ell \neq 1$. We condition on 
$\wt{\bA}_{\ell}$ and view 
$\sum_{\ell\neq 1}B_{\ell} \la \wt{\bA}_1, \wt{\bA}_{\ell}\ra$
as a Gaussian distributed RV
$\normdist\bracket{0, d^{-1}\|\sum_{\ell\neq 1}B_{\ell} \wt{\bA}_{\ell}\|^2_2}$.
Then we obtain 
\begin{align*}
& \Prob\bracket{
\bigg|\sum_{\ell\neq 1}B_{\ell} \la \wt{\bA}_1, \wt{\bA}_{\ell}\ra\bigg|
\gsim \sqrt{\frac{\log p}{d}}\sqrt{\sum_{\ell\neq 1}B_{\ell}^2}} \\
= ~&
\underbrace{\Prob\bracket{
\bigg|\sum_{\ell\neq 1}B_{\ell} \la \wt{\bA}_1, \wt{\bA}_{\ell}\ra\bigg|
\gsim \sqrt{\frac{\log p}{d}}\sqrt{\sum_{\ell\neq 1}B_{\ell}^2}}
\times \Psi(\calE_4(\bB))}_{T_1} \\
+~& \
\underbrace{\Prob\bracket{
\bigg|\sum_{\ell\neq 1}B_{\ell} \la \wt{\bA}_1, \wt{\bA}_{\ell}\ra\bigg|
\gsim \sqrt{\frac{\log p}{d}}\sqrt{\sum_{\ell\neq 1}B_{\ell}^2}}
\times \Psi(\br{\calE}_3(\bB))}_{T_2}.
\end{align*}
The proof is then completed by separately bounding 
$T_1$ and $T_2$. 
For term $T_1$, we have
\begin{align*}
T_1 \stackrel{\cirone}{\leq} \Expc_{\wt{\bA}_{\ell}} \exp\bracket{-\frac{c\log p (\sum_{\ell\neq 1}B_{\ell}^2)}{\|\sum_{\ell\neq 1}B_{\ell} \wt{\bA}_1\|^2_2}}\times 
\Psi(\calE_4(\bB)) \stackrel{\cirtwo}{\leq} 
\exp\bracket{-\frac{3\log p (\sum_{\ell\neq 1}B_{\ell}^2)}{\sum_{\ell\neq 1}B_{\ell}^2}}, 
\end{align*}
where $\cirone$ is due to the 
tail bound for the Gaussian RV $\sum_{\ell\neq 1}B_{\ell} \la \wt{\bA}_1, \wt{\bA}_{\ell}\ra$ conditional on $\wt{\bA}_{\ell}$, 
and $\cirtwo$ is according to the definition of $\calE_4(\bB)$. 
For term $T_2$, we have 
$T_2 \leq \Psi(\br{\calE}_3(\bB)) \leq e^{-0.8d}$. Summaring the 
above analysis finishes the proof.   

\end{proof}

\begin{lemma}
\label{lemma:two_one_one}
Conditional on events $\calE_1, \calE_2$, we have
\begin{align*}
 \abs{\sum_{\ell \neq 1, 2}B_{\ell} \la \wt{\bA}_1, \wt{\bA}_{\ell}\ra
\la \wt{\bA}_2, \wt{\bA}_{\ell}\ra}
\lsim \frac{\abs{\sum_{\ell > 2 }B_{\ell}}}{d}\sqrt{\frac{\log p}{d}}
+ \frac{\sqrt{\log p}}{d} 
\bigg[\sqrt{\sum_{\ell > 2}B_{\ell}^2}\vcup \
\bracket{\sqrt{\log p}(\max_{\ell}|B_{\ell}|)}\bigg],
\end{align*}
hold with probability at least $1-2p^{-3}$ 
for an arbitrary fixed vector $\bB\in \RR^p$. 	
\end{lemma}

\begin{proof}
First we rewrite the term 
$\sum_{\ell \neq 1, 2}B_{\ell} \la \wt{\bA}_1, \wt{\bA}_{\ell}\ra
\la \wt{\bA}_2, \wt{\bA}_{\ell}\ra$ as 
\begin{align*}
\sum_{\ell \neq 1, 2}B_{\ell} \la \wt{\bA}_1, \wt{\bA}_{\ell}\ra
\la \wt{\bA}_2, \wt{\bA}_{\ell}\ra = 
\sum_{\ell\neq 1, 2}\wt{\bA}_{\ell}^{\rmt} \
\bracket{B_{\ell} \wt{\bA}_1 \wt{\bA}_2^{\rmt}}\wt{\bA}_{\ell}.
\end{align*}
Then we concatenate the vectors $\wt{\bA}_{\ell}$, $\ell \neq 1, 2$
to a vector of length $d(p-2)$ and denote it as $\vect{\wt{\bA}}$.
Hence the summarization can be rewritten as 
\begin{align*}
\sum_{\ell\neq 1, 2}\wt{\bA}_{\ell}^{\rmt} \
\bracket{B_{\ell} \wt{\bA}_1 \wt{\bA}_2^{\rmt}}\wt{\bA}_{\ell} = 
\vect{\wt{\bA}}^{\rmt}\bLambda\vect{\wt{\bA}}, 
\end{align*}
where $\bLambda$ is a block-diagonal matrix whose 
$i$th block is $B_{i+2}\wt{\bA}_1\wt{\bA}_2^{\rmt}$, $1\leq i \leq p- 2$. 
Due to the independence between $\wt{\bA}_{\ell}$, $1\leq \ell \leq p$, 
we first condition on $\wt{\bA}_1, \wt{\bA}_2$ and perform the following 
decomposition
\begin{align*}
\abs{\sum_{\ell \neq 1, 2}B_{ \ell} \la \wt{\bA}_1, \wt{\bA}_{\ell}\ra
\la \wt{\bA}_2, \wt{\bA}_{\ell}\ra}  
\leq d^{-1}
\bigg(\underbrace{\abs{\wt{\bA}_1^{\rmt} \wt{\bA}_2} \
\bigg|\sum_{\ell\neq 1, 2}B_{\ell}\bigg|}_{\vartheta_1} + 
\underbrace{\abs{d\vect{\wt{\bA}}^{\rmt}\bLambda\vect{\wt{\bA}} - \sum_{\ell\neq 1, 2}B_{\ell}\wt{\bA}_1^{\rmt} \wt{\bA}_2}}_{\vartheta_2}\bigg). 
\end{align*}
The upper-bound for $\vartheta_1$ is relatively easy, which 
reads as 
\begin{align*}
\vartheta_1 \lsim \abs{\sum_{\ell > 2} B_{\ell}}\times \
\sqrt{\frac{\log p}{d}}.
\end{align*}
The following analysis focus on bound $\vartheta_2$.
Since $\vartheta_2$ also 
reads as 
\begin{align*}
\vartheta_2 = 
\abs{d\vect{\wt{\bA}}^{\rmt}\bLambda\vect{\wt{\bA}} - \Expc_{\wt{\bA}_{\ell}, \ell > 2}
\bracket{d\vect{\wt{\bA}}^{\rmt}\bLambda\vect{\wt{\bA}}} },
\end{align*} 
we can invoke the Hanson-Wright inequality 
(Theorem 6.2.1 in \cite{vershynin2016high}) and obtain 
\begin{align*}
\Prob\bracket{\vartheta_2 \geq \delta} 
\leq 2\exp\bracket{-c\bracket{\frac{\delta^2}{\Fnorm{\bLambda}^2} \vcap \frac{\delta}{\Opnorm{\bLambda}}}}, 
\end{align*}
where $\wt{\bA}_1, \wt{\bA}_2$ are viewed as constants. 
 We complete the proof by setting $\delta$ as 
\begin{align*}
\delta \asymp \
\bracket{\sqrt{\log p}\Fnorm{\bLambda}} \vcup \
\bracket{\log p\Opnorm{\bLambda}}, 
\end{align*}
which yields
$\Prob(\vartheta_2 \geq \delta) \leq 2p^{-3}$. The 
specific values of $\Fnorm{\bLambda}$ and $\Opnorm{\bLambda}$ are 
computed as 
\begin{align*}
\Fnorm{\bLambda}^2 &= \bracket{\sum_{\ell >2 }B_{\ell}^2} 
\Fnorm{\wt{\bA}_1\wt{\bA}_2^{\rmt}}^2 \
\stackrel{\cirone}{=} \
\bracket{\sum_{\ell > 2 }B_{\ell}^2}
\trace\bracket{ \wt{\bA}_2 \wt{\bA}_1^{\rmt}\wt{\bA}_1\wt{\bA}_2^{\rmt}} \\
&~\stackrel{\cirtwo}{=} \norm{\wt{\bA}_1}{2}^2\norm{\wt{\bA}_2}{2}^2
\bracket{\sum_{\ell > 2}B_{\ell}^2}
\stackrel{\cirthree}{\lsim} \sum_{\ell > 2}B_{\ell}^2 \\
\Opnorm{\bLambda} &= \max_{\ell}\abs{B_{\ell}}\Opnorm{\wt{\bA}_1 \wt{\bA}_2^{\rmt}} \
\stackrel{\cirfour}{=} \max_{\ell}\abs{B_{\ell}}\times \norm{\wt{\bA}_1}{2}\norm{\wt{\bA}_2}{2}
\lsim \max_{\ell}\abs{B_{\ell}},
\end{align*}
where in $\cirone$ we use $\Fnorm{\bM}^2 = \trace(\bM^{\rmt}\bM)$
for arbitrary matrix $\bM$, in 
$\cirtwo$ we use $\trace(\bM_1 \bM_2) = \trace(\bM_2 \bM_1)$, 
in $\cirthree$ we condition on event $\calE_1$, 
and in $\cirfour$ we use $\Opnorm{\bu \bv^{\rmt}} = \norm{\bu}{2}\norm{\bv}{2}$
for arbitrary vectors $\bu, \bv$.

\end{proof}


\begin{lemma}
\label{lemma:four_diff}
We have 
\begin{align*}
\bigg|\sum_{\substack{\ell_1, \ell_2 \neq 1, 2 \\ \ell_1 \neq \ell_2}}
B_{\ell_1, \ell_2} \la \wt{\bA}_1, \wt{\bA}_{\ell_1}\ra \la \wt{\bA}_2, \wt{\bA}_{\ell_2}\ra \bigg|
\lsim \frac{\log p}{d}\bracket{1 + c_0\sqrt{\frac{p}{d}}}^2 
\sqrt{\sum_{\substack{i\neq 1,2 \\ j\neq 1, 2, i}}B_{ij}^2}, 
\end{align*}
holds with probability exceeding 
$1 - 4p^{-3} - e^{-c_0 p}$ for a fixed matrix 
$\bB$. 
\end{lemma}

\begin{proof}
We begin the proof by first rewriting 
$\sum_{\substack{\ell_1, \ell_2 \neq 1, 2 \\ \ell_1 \neq \ell_2}}
B_{\ell_1, \ell_2} \la \wt{\bA}_1, \wt{\bA}_{\ell_1}\ra \la \wt{\bA}_2, \wt{\bA}_{\ell_2}\ra$
as $\wt{\bA}_1^{\rmt}\bLambda \wt{\bA}_2$, where 
$\bLambda$ is defined as 
\begin{align*}
\bLambda \defequal \sum_{ \ell_1 \neq 1, 2}
\sum_{\ell_2 \neq 1, 2, \ell_1}
B_{\ell_1, \ell_2}  \wt{\bA}_{\ell_1}  \wt{\bA}^{\rmt}_{\ell_2} . 
\end{align*}
The whole proof procedure can be divided into the 
following stages. 
\par \noindent
\textbf{Stage I.}
Due to the independence across $\wt{\bA}_{\ell}$, we have 
\begin{align}
\label{eq:four_diff_stageI}
& \Prob\bracket{\abs{\wt{\bA}_{1}^{\rmt} \bLambda \wt{\bA}_{2}} \geq \delta}
\leq \Prob\bracket{\abs{\wt{\bA}_{1}^{\rmt} \bLambda \wt{\bA}_{2}} \geq \delta,~
\norm{\bLambda \wt{\bA}_2}{2} \leq \delta_1} + 
\Prob\bracket{\norm{\bLambda \wt{\bA}_2}{2} \geq \delta_1} \notag \\
\stackrel{\cirone}{\leq} ~& 2\exp\bracket{-\frac{d\delta^2}{2\delta_1^2}}
+ \Prob\bracket{\norm{\bLambda \wt{\bA}_2}{2} \geq \delta_1}, 
\end{align}
where in $\cirone$ we first condition on $\wt{\bA}_{\ell}$, $(\ell > 2)$ and 
view $\wt{\bA}_{1}^{\rmt} \bLambda \wt{\bA}_{2}$ as a Gaussian RV
satisfying $\normdist\bracket{0, d^{-1}\|\bLambda \wt{\bA}_2\|^2_2}$.
\par \noindent
\textbf{Stage II.}
To bound the probability 
$\Prob\bracket{\norm{\bLambda \wt{\bA}_2}{2} \geq \delta_1}$, 
we need to upper bound the Frobenius norm $\Fnorm{\bLambda}$. 
First we define two matrices, namely, $\wt{\wt{\bA}}$ and $\wt{\bB}$, 
for the conciseness of notation, which reads 
\begin{align*}
\wt{\wt{\bA}} \defequal \Bracket{\wt{\bA}_3~\cdots~\wt{\bA}_p},~~\
(\wt{\bB})_{i, j} \defequal  \left\{
\begin{aligned}
&B_{i+2, j+2},~~&&\textup{if}~~i \neq j; \\
&0,~~&&\textup{otherwise}.	
\end{aligned}\right. 
\end{align*}
Easily we can verify $\bLambda$ is equivalent to
$\bLambda$ as $\wt{\wt{\bA}} \wt{\bB} \wt{\wt{\bA}}^{\rmt}$, which
gives
\begin{align*}
\Fnorm{\bLambda} = \Fnorm{\wt{\wt{\bA}} \wt{\bB} \wt{\bA}^{\rmt}}
\stackrel{\cirtwo}{\leq} \Opnorm{\wt{\wt{\bA}}}^2 \Fnorm{\wt{\bB}} 
\stackrel{\cirthree}{\leq} \Opnorm{\wt{\bA}}^2  \Fnorm{\wt{\bB}},
\end{align*}
where in $\cirtwo$ we adopt the relation 
$\Fnorm{\bM_1\bM_2} \leq \Opnorm{\bM_1}\Fnorm{\bM_2}$ 
such that $\bM_1, \bM_2$ are arbitrary matrix, 
and in $\cirthree$ we use the relation 
$\Opnorm{\wt{\wt{\bA}}} \leq \Opnorm{\wt{\bA}}$ in Corol.~$2.4.2$ \cite{golub2012matrix} 
since $\wt{\wt{\bA}}$ can be viewed as a sub-matrix of $\wt{\bA}$.
\par 
Hence we conclude that 
\begin{align*}
& \Prob\bracket{\Fnorm{\bLambda} \geq \bracket{1 + c_0\sqrt{\frac{p}{d}}}^2\Fnorm{\wt{\bB}}} \\
=~& \Prob\bracket{\Fnorm{\bLambda} \gsim \frac{p}{d}\Fnorm{\wt{\bB}},~
\Opnorm{\wt{\bA}} \leq \sqrt{\frac{c_0p}{d}}} + 
 \Prob\bracket{\Fnorm{\bLambda} \gsim \frac{p}{d}\Fnorm{\wt{\bB}},~
\Opnorm{\wt{\bA}} > \sqrt{\frac{c_0p}{d}}} \\
\leq~& \underbrace{\Prob\bracket{\Opnorm{\wt{\bA}}^2  \gsim \frac{p}{d}, \Opnorm{\wt{\bA}} \leq \sqrt{\frac{c_0p}{d}}}}_{0} + 
\Prob\bracket{\Opnorm{\wt{\bA}} \geq \sqrt{\frac{c_0p}{d}}} \
\stackrel{\cirfour}{\leq} e^{-c_0p}, 
\end{align*}
where $\cirfour$ is due to Thm.~$6.1$ in \cite{wainwright_2019}. 

\par \noindent
\textbf{Stage III.}
We bound $\Prob\bracket{\norm{\bLambda \wt{\bA}_2}{2} \geq \delta_1}$
by splitting it as 
\begin{align}
\label{eq:four_diff_stageIII_main}
\Prob\bracket{\norm{\bLambda \wt{\bA}_2}{2} \geq \delta_1} = ~&
\underbrace{\Prob \bracket{\norm{\bLambda \wt{\bA}_2}{2} \geq \delta_1,~\Fnorm{\bLambda}\
\geq  \bracket{1 + c_0\sqrt{\frac{p}{d}}}^2\Fnorm{\wt{\bB}}}}_{\leq 
~\Prob\bracket{\Fnorm{\bLambda}\
\geq  \bracket{1 + c_0\sqrt{\frac{p}{d}}}^2\Fnorm{\wt{\bB}}}
} \notag \\
+~&\Prob \bracket{\norm{\bLambda \wt{\bA}_2}{2} \geq \delta_1,~\Fnorm{\bLambda}\
< \bracket{1 + c_0\sqrt{\frac{p}{d}}}^2\Fnorm{\wt{\bB}}}.
\end{align}
Notice that the first term is bounded in Stage II, 
we focus on bounding the second term in this stage, which 
proceeds as  
\begin{align}
\label{eq:four_diff_stageIII_split}
&\Prob\bracket{\norm{\bLambda \wt{\bA}_2}{2} \geq \delta_1,~\Fnorm{\bLambda}\
< \bracket{1 + c_0\sqrt{\frac{p}{d}}}^2\Fnorm{\wt{\bB}} } \notag \\
\stackrel{\cirfive}{\leq}~&
\Expc\Ind\bracket{\abs{d\norm{\bLambda \wt{\bA}_2}{2}^2 - \Fnorm{\bLambda}^2}
\geq \delta_2,~\Fnorm{\bLambda}\
< \bracket{1 + c_0\sqrt{\frac{p}{d}}}^2\Fnorm{\wt{\bB}}} \notag \\
\stackrel{\cirsix}{\leq}~&
2\exp\bracket{-c_0 \bracket{\frac{\delta_2}{\Opnorm{\bLambda^{\rmt}\bLambda}} \vcap \frac{\delta_2^2}{\Fnorm{\bLambda^{\rmt}\bLambda}^2}}}
\times \Ind\bracket{\Fnorm{\bLambda}\
< \bracket{1 + c_0\sqrt{\frac{p}{d}}}^2\Fnorm{\wt{\bB}}} \notag \\
\leq~& 
2\exp\bracket{-c_0 \bracket{\frac{\delta_2}{\Fnorm{\bLambda^{\rmt}\bLambda}} \vcap \frac{\delta_2^2}{\Fnorm{\bLambda^{\rmt}\bLambda}^2}}}
\times \Ind\bracket{\Fnorm{\bLambda}\
< \bracket{1 + c_0\sqrt{\frac{p}{d}}}^2\Fnorm{\wt{\bB}}} \notag \\
\stackrel{\cirseven}{\leq}~& 2p^{-3}, 
\end{align}
where in $\cirfive$ we require $d\delta_1^2 \geq \Fnorm{\bLambda}^2 + \delta_2$, 
$\cirsix$ is due to the Hanson-Wright inequality 
(cf. Thm.~ 6.2.1 in \cite{vershynin2016high}), 
and in $\cirseven$ we set 
$\delta_2$ as 
$(1 + c_0\sqrt{p/d})^4 \log p  \Fnorm{\wt{\bB}}^2$.
\par 
Combining \eqref{eq:four_diff_stageI}, 
\eqref{eq:four_diff_stageIII_main}, and 
\eqref{eq:four_diff_stageIII_split}, we 
set $\delta_1, \delta$ as 
$c\sqrt{\log p/d}(1 + c_0\sqrt{p/d})^2 \Fnorm{\wt{\bB}}$
and $c_3\delta \asymp \sqrt{\log p/d}\delta_1$, respectively,
which yields
\begin{align*}
\Prob\bracket{\abs{\wt{\bA}_{1}^{\rmt} \bLambda \wt{\bA}_{2}} \gsim 
\frac{\log p}{d}\bracket{1 + c_0\sqrt{\frac{p}{d}}}^2 \Fnorm{\wt{\bB}} }
\leq 4p^{-3} + e^{-c_0p}, 
\end{align*}
and completes the proof. 
\end{proof}

\begin{lemma}
\label{lemma:noise_term}
Conditional on the event $\calE_1$, we have 
\begin{align*}
& \Prob\Bracket{n^{-1}\abs{\wt{\bA}_i^{\rmt}\bracket{\sum_{\ell= 1}^n \bW^{(i)}\bW^{(i)\rmt}}\wt{\bA}_j}  \geq 
c_0 \sigma^2 \bracket{1 + c_1\sqrt{\frac{\log p}{d}}}\sqrt{\frac{\log p}{d}}
\bracket{1 + c_2\bracket{\sqrt{\frac{d}{n}} \vcup \frac{d}{n}}}
 } 
 \leq p^{-4} + c_3 e^{-c_4 d}
\end{align*}
with fixed $i, j$.

\end{lemma} 
\begin{proof}
First we define the matrix
$\wh{\bXi}_n$ as $\frac{1}{n}\sum_{\ell= 1}^n \bW^{(i)}\bW^{(i)\rmt}$ for the conciseness of notation. Then we upper-bound
$\wt{\bA}_i \wh{\bXi}_n \wt{\bA}_j$, $i\neq j$ as 
\begin{align*}
& \Prob\bracket{\abs{\wt{\bA}_i^{\rmt} \wh{\bXi}_n \wt{\bA}_j} \geq \delta} 
\stackrel{\cirone}{\leq} \Prob\bracket{\|\wh{\bXi}_n \wt{\bA}_j\|_{2} \geq 
\delta_1 }  + \exp\bracket{-\frac{d\delta^2}{2\delta_1^2}}  \\
\stackrel{\cirtwo}{\leq}~& \Prob\bracket{\Opnorm{\wh{\bXi}_n}\norm{\wt{\bA}_j}{2}\geq \delta_1} + 
p^{-4}
\stackrel{\cirthree}{\leq}
\Prob\bracket{\Opnorm{\wh{\bXi}_n} \geq \delta_2} + p^{-4} \\
\stackrel{\cirfour}{\leq}~& 
\Prob\bracket{\Opnorm{\wh{\bXi}_n - \sigma^2 \bI } \geq c_2\sigma^2\bracket{\frac{d}{n} \vcup 
\sqrt{\frac{d}{n}}} }
+ p^{-4} \stackrel{\cirfive}{\leq} 
p^{-4} + c_0 e^{-c_1 d}, 
\end{align*}
where  in $\cirone$ we 
exploit the independence between $\wt{\bA}_i$ and $\wt{\bA}_j$ when
$i \neq j$ and treat 
$\wt{\bA}_i^{\rmt} \wh{\bXi}_n \wt{\bA}_j$ as a Gaussian 
RV with mean zero and variance $d^{-1}\|\wh{\bXi}_n \wt{\bA}_j\|_{2}^2$, in $\cirtwo$ we use the fact 
$\delta = 2\delta_1\sqrt{2\log p/d}$, and 
in $\cirthree$ we condition on event $\calE_1$ and set
$\delta_1 \geq \delta_2 \bracket{1 + c_0\sqrt{\log p/d}}$, 
$\cirfour$ is because 
$\delta_2 = \sigma^2\bracket{1 +c_2\bracket{d/n \vcup \sqrt{d/n}}}$, 
and in $\cirfive$ we use Thm.~$6.5$ in \cite{wainwright_2019}.
The proof is then completed by setting 
$\delta$ as 
\begin{align*}
\delta = 
c_0 \sigma^2 \bracket{1 + c_1\sqrt{\frac{\log p}{d}}}\sqrt{\frac{\log p}{d}}
\bracket{1 + c_2\bracket{\sqrt{\frac{d}{n}} \vcup \frac{d}{n}}}.
\end{align*} 
\end{proof}

\subsection{Insight Behind the Design of the Covariance Matrix Estimator}
\label{sec:cov_mat_estim}
We now explain the 
rational behind the covariance matrix estimators
of $\bX$ that we use in the graphical structure estimation via the parametric method. 
For this purpose, we exploit the statistical properties 
of $\wt{\bA}$.

We approximate data samples, $\wh{\bX}^{(i)}$'s, as
$\wh{\bX}^{(i)} = \wt{\bA}^{\rmt}\wt{\bA} \bX^{(i)} = \wt{\bA}^{\rmt}\bY^{(i)}$, 
$1\leq i \leq n$.
Due to the assumption on sensing matrix $\wt{\bA}$,  
we have 
$\Expc_{\wt{\bA}}\wh{\bX}^{(i)} = \bX^{(i)}$. Hence, we 
can view the samples $\{\wh{\bX}^{(i)}\}_{i=1}^n$
as a ``perturbed'' version of the true data points 
$\set{\bX^{(i)}}_{i=1}^n$. Hence, we propose to estimate the covariance matrix of $\bX$ 
from $\{\wh{\bX}^{(i)}\}_{i=1}^n$. 
Following the above approach, a naive covariance estimator is given as
\vspace{-2mm}
\begin{align*}
\wh{\bSigma}_{n, 1} = 
\frac{1}{n}\sum_{i=1}^n \wh{\bX}^{(i)}\wh{\bX}^{(i)\rmt} = \
\wt{\bA}^{\rmt}\bracket{\frac{1}{n}\sum_{i=1}^n \bY^{(i)}\bY^{(i)\rmt}} \wt{\bA}. 
\vspace{-2mm}
\end{align*}
However, through numerical experiments, we observed that this estimator performs poorly. 
To improve the performance of $\wh{\bSigma}_{n, 1}$, we 
first analyze its properties thoroughly and 
then refine the estimator. 

\subsubsection{Theoretical Properties}
First, we evaluate the mean and variance of the naive covariance estimator $\wh{\bSigma}_{n, 1}$.
\begin{lemma}
\label{thm:naive_covmat_estim}
The mean of the naive covariance estimator is given by
\begin{align*}
\Expc_{\wt{\bA}, \bX}[\wh{\bSigma}_{n, 1}] = 
\frac{d+1}{d}\bSigma^{\natural} + \frac{p}{d}\bI 
+ \sigma^2 \bI.
\end{align*}
\end{lemma}

\begin{proof}
Due to the independence between $\bX$ and $\wt{\bA}$, we first condition 
on $\wt{\bA}$ and take expectation w.r.t $\bX, \bW$, which gives
\[
\Expc_{\bX, \bW} \wh{\bSigma}_{n, 1} = \
\wt{\bA}^{\rmt}\wt{\bA}\bSigma^{\natural} \wt{\bA}^{\rmt}\wt{\bA} 
+ \sigma^2 \wt{\bA}^{\rmt}\wt{\bA}. 
\]
Then we conclude that  
\begin{align*}
& \Expc_{\wt{\bA}}\bracket{\wt{\bA}^{\rmt}\wt{\bA}\bSigma^{\natural} \wt{\bA}^{\rmt}\wt{\bA}}_{ij} = 
\Expc_{\wt{\bA}}\sum_{\ell_1,\ell_2}\bSigma^{\natural}_{\ell_1,\ell_2}\
 \la \wt{\bA}_{\ell_1}, \wt{\bA}_i\ra \la \wt{\bA}_{\ell_2}, \wt{\bA}_j\ra + \sigma^2 \Ind(i = j)\\
=~& \Expc_{\wt{\bA}}\sum_{\ell_1, \ell_2}\bSigma^{\natural}_{\ell_1,\ell_2} \
\bracket{\sum_{\ell_3} \wt{\bA}_{\ell_3, \ell_1} \wt{\bA}_{\ell_3, i}}\cdot \
\bracket{\sum_{\ell_4} \wt{\bA}_{\ell_4, \ell_2} \wt{\bA}_{\ell_4, j}} + \sigma^2 \Ind(i = j)\\
=~& \sum_{\ell_1, \ell_2, \ell_3, \ell_4} \bSigma^{\natural}_{\ell_1,\ell_2} \
\Expc_{\wt{\bA}}\bracket{\wt{\bA}_{\ell_3, \ell_1} \wt{\bA}_{\ell_3, i} \wt{\bA}_{\ell_4, \ell_2} \wt{\bA}_{\ell_4, j}}
+ \sigma^2 \bI \\
\stackrel{\cirone}{=}~& 
d^{-2}\bigg[\sum_{\ell_1, \ell_2, \ell_3, \ell_4} \bSigma^{\natural}_{\ell_1,\ell_2}\
\bigg(
\Ind\bracket{\ell_1 = i} \Ind(\ell_2 = j) + \
\Ind\bracket{\ell_3 = \ell_4}\Ind(\ell_1 = \ell_2)\Ind(i = j) \\
~&\quad \quad \quad \quad \quad \quad \quad
+ \Ind(\ell_3 = \ell_4)\Ind(\ell_1 = j)\
\Ind(\ell_2 = i)\bigg)\bigg] + \sigma^2 \Ind(i = j)\\
=~& d^{-2}\Bracket{\sum_{\ell_3, \ell_4} \bSigma^{\natural}_{i, j} + 
\Ind(i = j)\bracket{\sum_{\ell_1} \sum_{\ell_2}\bSigma^{\natural}_{\ell_1, \ell_1}} + \
\bSigma^{\natural}_{j, i}\sum_{\ell_3, \ell_4}\Ind(\ell_3 = \ell_4)} + \sigma^2 \Ind(i = j)\\
=~& d^{-2}\Bracket{d^2 \bSigma^{\natural}_{i, j} + \Ind(i=j)d \sum_{\ell}\bSigma^{\natural}_{\ell, \ell}
+ d\bSigma^{\natural}_{i, j}} + \sigma^2 \Ind(i = j)\\
=~& 
\bracket{1 + d^{-1}}\bSigma^{\natural}_{i, j} + d^{-1}\bracket{\sum_{\ell}\bSigma^{\natural}_{\ell, \ell}}\Ind(i=j) + \sigma^2 \Ind(i = j), 
\end{align*}
which completes the proof with the fact
$\bSigma^{\natural}_{\ell, \ell} = 1$. 
In $\cirone$, we use the Wick's theorem, which 
is listed as Thm.~\ref{thm:wick} for the sake of self-containing.
\end{proof}

Then we study the variance 
$\Var_{\wt{\bA}} \wh{\bSigma}_{n, 1}$, which is listed as the following.
Due to the complex formula of the variance, 
we only consider the noiseless case, namely, $\sigma^2 = 0$. 

\begin{lemma}
\label{lemma:naive_covmat_var}
Consider the noiseless case where $\sigma^2 = 0$, we have 
$\Var_{\wt{\bA}}(\wh{\bSigma}_{n, 1})_{i, j} = 
\Omega(d^{-1}) + \Omega(d^{-1})\Ind(i = j) + \Omega(d^{-2}).$
\end{lemma}
The detailed proof is given in the appendix. 
We use results of Lemma~\ref{thm:naive_covmat_estim} to improve our naive covariance estimator in two perspectives: 
\textbf{bias correction} and \textbf{variance reduction}. 

\begin{proof}
With the relation $\Var_{\wt{\bA}} \wh{\bSigma}_{n, 1} = \Expc_{\wt{\bA}}(\wh{\bSigma}_{n, 1})_{i,j}^2 - \bracket{\Expc \wh{\bSigma}_{n,1}}_{
i, j}^2$, we complete the proof by invoking 
Lemma~\ref{thm:naive_covmat_estim} and 
Lemma~\ref{lemma:cov_mat_var}. 
The following context focuses on  
proving Lemma~\ref{lemma:cov_mat_var}. 
\end{proof}

\begin{lemma}
\label{lemma:cov_mat_var}
We have 
\begin{align*}
&\Expc_{\wt{\bA}}(\wh{\bSigma}_{n, 1})_{i, j}^2
= (\bSigma_n)_{i,j}^2 + 
d^{-1} \Bracket{\norm{(\bSigma_n)_i}{2}^2 + 
\norm{(\bSigma_n)_j}{2}^2+ 4(\bSigma_n)^2_{i,j}  +
2(\bSigma_n)_{i,i}(\bSigma_n)_{j,j}} \\
+~&
d^{-2}\Bracket{2\bracket{(\bSigma_n)_{i,i} + (\bSigma_n)_{j,j} }\trace(\bSigma_n)
+ 4(\bSigma_n)_{i,i}(\bSigma_n)_{j,j} + 
2 \norm{(\bSigma_n)_j}{2}^2+ 
2\norm{(\bSigma_n)_i}{2}^2 + 
\fnorm{\bSigma_n}^2 + 7 (\bSigma_n)_{i,j}^2}  \\
+~& d^{-3}\Bracket{3\norm{(\bSigma_n)_i}{2}^2 + 
3\norm{(\bSigma_n)_j}{2}^2+ 
2\bracket{(\bSigma_n)_{i,i} + (\bSigma_n)_{j,j}}\trace(\bSigma_n) +  
\trace^2(\bSigma_n) + 2(\bSigma_n)_{i,i}(\bSigma_n)_{j,j}
+ 4(\bSigma_n)_{i,j}^2} \\
+~& d^{-1}\Ind(i= j)\Bracket{\la (\bSigma_n)_i, (\bSigma_n)_j\ra 
+ 2(\bSigma_n)_{i,j} \trace(\bSigma_n)} \\
+~& d^{-2}\Ind(i=j)\Bracket{11\la (\bSigma_n)_i, (\bSigma_n)_j\ra 
+  \trace^2(\bSigma_n) + 6 (\bSigma_n)_{i,j}\trace(\bSigma_n)} \\
+~&d^{-3}\Ind(i = j)\Bracket{12\la (\bSigma_n)_i, (\bSigma_n)_j\ra  + 
8(\bSigma_n)_{i,j}\trace(\bSigma_n) + 
\trace^2(\bSigma_n) + 
2\fnorm{\bSigma_n}^2}.
\end{align*}
\end{lemma}

\begin{proof}
The proof procedure is fundamentally the same as 
the steps in computing $\Expc_{\bX} \wh{\bSigma}^{\textup{param}}_n$ but 
is more involved (requires 
computing over $100$ terms).

We begin the proof with the following expansion 
\begin{align*}
&\Expc \Bracket{\bracket{\wt{\bA}^{\rmt}\wt{\bA} (\bSigma_n) \wt{\bA}^{\rmt}\wt{\bA}}_{ij}}^2 \\
=~& 
\sum_{\ell_1,\ell_2, \ell_3, \ell_4, \ell_5,\ell_6,  \ell_7, \ell_8}\
(\bSigma_n)_{\ell_1,\ell_2}(\bSigma_n)_{\ell_5,\ell_6}\
\Expc\bracket{
\wt{\bA}_{\ell_3, \ell_1} \wt{\bA}_{\ell_3, i} \
\wt{\bA}_{\ell_4, \ell_2} \wt{\bA}_{\ell_4, j} \
\wt{\bA}_{\ell_7, \ell_5} \wt{\bA}_{\ell_7, i} \
\wt{\bA}_{\ell_8, \ell_6} \wt{\bA}_{\ell_8, j}}.
\end{align*}
 
Then we expand the term 
$\Expc\bracket{
\wt{\bA}_{\ell_3, \ell_1} \wt{\bA}_{\ell_3, i} \
\wt{\bA}_{\ell_4, \ell_2} \wt{\bA}_{\ell_4, j} \
\wt{\bA}_{\ell_7, \ell_5} \wt{\bA}_{\ell_7, i} \
\wt{\bA}_{\ell_8, \ell_6} \wt{\bA}_{\ell_8, j}}$ via the 
Wick's Theorem, which is also listed as Thm.~\ref{thm:wick}
for the sake of self-containing. 
Additionally, we need to 
divide $d^4$ for the following result.

\end{proof}

\subsubsection{Estimator Refinement}

\paragraph{Bias Correction.}
First we note that the naive estimator is biased $\Expc_{\wt{\bA}, \bX}\wh{\bSigma}_{n, 1} \neq \bSigma^{\natural}$. 
Adopting ideas similar to the moment estimator, 
we correct the bias of the estimator via
\begin{align*}
\wh{\bSigma}_{n, 2} \defequal \
\dfrac{d}{d+1}\wt{\bA}^{\rmt}\bracket{\frac{1}{n}\sum_{i=1}^n \bY^{(i)} \bY^{(i)\rmt}}\wt{\bA}
-\dfrac{p + d\sigma^2}{d+1}\bI.
\end{align*}
It can be easily verified that 
$\Expc_{\wt{\bA}, \bX}[\wh{\bSigma}_{n, 2}] = \bSigma^{\natural}$.

\par 

\paragraph{Variance Reduction.}
To further improve the performance of the 
covariance estimator, we perform variance reduction. 
From Lemma~\ref{lemma:naive_covmat_var}, we 
observe that the diagonal entries $(\wh{\bSigma}_{n, 2})_{i, i}$ 
have a higher variance than non-diagonal entries. 
On the other hand, there is no need for estimating 
$\diag(\bSigma^{\natural})$ since by the assumptions, we know that the diagonal elements of the covariance matrix are $1$.
Therefore, we suggest refining the estimator $\wh{\bSigma}_{n, 2}$ by 
fixing its diagonal entries to $1$, i.e., the resulting 
refined covariance estimator is given as
\begin{align}
\label{eq:cov_mat_estim}
\wh{\bSigma}^{\textup{param}}_n = 
\bI + \frac{d}{d+1}
\Bracket{\wt{\bA}^{\rmt}\bracket{\frac{1}{n}\sum_{i=1}^n \bY^{(i)}\bY^{(i)\rmt}}\wt{\bA}}_{\textup{off}},
\end{align}
where $(\cdot)_{\textup{off}}$ denotes the operation of 
picking non-diagonal entries. It can be easily verified that the estimator $\wh{\bSigma}^{\textup{param}}_n$ in (\ref{eq:cov_mat_estim})
is unbiased.



%
\section{Proof of Thm.~$1$}
The analysis is based on \textbf{primal-dual method}, which 
is adapted from \cite{ravikumar2011high}. 
First we write the optimality condition for 
\eqref{eq:param_glasso_def} as 
\begin{align}
\label{eq:glasso_grad_optim}
\wh{\bSigma}^{\textup{param}}_n &- \wh{\bTheta}_{\textup{param}}^{-1} + \lambda_{\textup{param}} \bG = \bZero, 
\end{align}
where $\bG$ is the sub-gradient \cite{rockafellar1970convex} of 
$\offdiagnorm{\wh{\bTheta}_{\textup{param}}}$ and is defined as 
\vspace{-2mm}
\begin{align*}
G_{ij} \defequal \left\{ \
\begin{aligned}
&\sgn(\wh{\bTheta}_{\textup{param}})_{i,j},~~&&\textup{if}~~(\wh{\bTheta}_{\textup{param}})_{i,j}\neq 0; \\
&\in [-1, 1],~~&&\textup{otherwise}.	
\end{aligned}\right.
\vspace{-2mm}
\end{align*} 
for $i\neq j$.
However, because of
the complexity of (\ref{eq:glasso_grad_optim}),
directly bounding the deviation $\|\wh{\bTheta}_{\textup{param}} - \bTheta^{\natural}\|_{\infty}$
can be difficult. 
Instead, we construct a pair 
$\bracket{\wt{\bTheta}_{\textup{param}}, \wt{\bG}}$ which 
satisfies: $(i)$ $\wt{\bG}$ is the 
sub-differential of $\offdiagnorm{\wt{\bTheta}_{\textup{param}}}$; and
$(ii)$ the pair $(\wt{\bTheta}_{\textup{param}}, \wt{\bG})$
satisfies the condition in \eqref{eq:glasso_grad_optim}.
Then we show it coincides with the solution of ($5$).
The basic rational is as follows.
First we verify that $\wh{\bTheta}_{\textup{param}}$ is the unique solution of 
(\ref{eq:glasso_grad_optim}). 
Since our constructed pairs $(\wt{\bTheta}_{\textup{param}}, \wt{\bG})$
satisfies the condition in \eqref{eq:glasso_grad_optim}, which 
corresponds to the solution in ($5$) exclusively, 
we can hence show the constructed pair 
$(\wt{\bTheta}_{\textup{param}}, \wt{\bG})$
is the identical solution of ($5$), namely,
$(\wh{\bTheta}_{\textup{param}}, \bG)$.
Afterwards we can 
upper bound 
$\|\wh{\bTheta}_{\textup{param}} - \bTheta^{\natural}\|_{\infty}$
by investigating 
$\|\wt{\bTheta}_{\textup{param}} - \bTheta^{\natural}\|_{\infty}$, 
which is more amenable  
for the analysis as following. 

\noindent 
\textbf{Stage I: construct $\wt{\bTheta}_{\textup{param}}$.} 
We construct the primal-dual witness solution 
$\bracket{\wt{\bTheta}_{\textup{param}}, \wt{\bG}}$ 
assuming the support set 
$S$ is given as a prior. This step 
can further be divided into three stages, as illustrated 
in the main context.  
First we construct the 
matrix $\wt{\bTheta}_{\textup{param}}$. For the 
entries $(i, j)$ restricted to the 
set $S$, we set them as 
\begin{align*}
(\wt{\bTheta}_{\textup{param}})_S = &\argmin_{\
\substack{\bTheta \succ \bZero \\ \bTheta =\bTheta^{\rmt}, \bTheta_{S^c} = \bZero}} -\logdet \bTheta +\la \wh{\bSigma}^{\textup{param}}_n, \bTheta\ra + \lambda_{\textup{param}}\offdiagnorm{\bTheta}.
\end{align*}
For the rest of entries $(\wt{\bTheta}_{\textup{param}})_{S^c}$, 
we set them to be zero values. Then we construct the 
sub-differential $\wt{\bG}$ corresponding to the 
matrix $\wt{\bTheta}_{\textup{param}}$. For the entry $(i, j)\in S$, we set 
$\wt{G}_{i,j} = \sign(\wt{\bTheta}_{\textup{param}})_{i, j}$. 
For the entry $(i, j)$ that is outside of the support 
set $S$, 
we set 
$\wt{G}_{i, j}$ as 
\begin{align*}
\wt{G}_{i, j} =  \lambda_{\textup{param}}^{-1}\
\Bracket{ \bracket{\wt{\bTheta}_{\textup{param}}}_{i, j}^{-1}- (\wh{\bSigma}^{\textup{param}}_n)_{i, j}}.
\end{align*}
The goal of this step is to ensure 
that $\bracket{\wt{\bTheta}_{\textup{param}}, \wt{\bG}}$ 
satisfies (\ref{eq:glasso_grad_optim}). 
As illustrated in the main context, 
we have the pair $\bracket{\wt{\bTheta}_{\textup{param}}, \wt{\bG}}$
coincides with the solution of \eqref{eq:glasso_grad_optim}
once $\abs{\wt{G}_{i,j}} < 1$ for the entry 
$(i, j) \in S^c$. The following step 
focuses on showing $\abs{\wt{G}_{i,j}} < 1$.

%

\noindent
\textbf{Stage II: construct $\wt{\bG}$.}
For the entry $(i, j)\in S$, we set 
$\wt{G}_{i,j} = \sign(\wt{\bTheta}_{\textup{param}})_{i, j}$. 
For the entry $(i, j)$ that is outside of the support 
set $S$, 
we set 
$\wt{G}_{i, j}$ as 
\begin{align*}
\wt{G}_{i, j} = \lambda_{\textup{param}}^{-1}\
\Bracket{ \bracket{\wt{\bTheta}_{\textup{param}}}_{i, j}^{-1}- (\wh{\bSigma}_n^{\textup{param}})_{i, j}}.
\end{align*}
The goal of this step is to ensure 
that $\bracket{\wt{\bTheta}_{\textup{param}}, \wt{\bG}}$ 
satisfies (\ref{eq:glasso_grad_optim}). 

\noindent 
\textbf{Stage III: verify $\wt{\bG}$ to be the 
sub-differential of $\offdiagnorm{\wt{\bTheta}_{\textup{param}}}$.}
In the following analysis, we verify that 
$|\wt{G}_{i, j}| < 1$, which 
yields the upper-bound on
$\|\wt{\bTheta}_{\textup{param}} - \bTheta^{\natural}\|_{\infty}$ 
as a byproduct. 
We first need the necessary lemmas 
from \cite{ravikumar2011high}.

\vspace{-3mm}
\begin{lemma}[Lemma~$6$ in \cite{ravikumar2011high}]
\label{lemma:ravi1}
Suppose that 
$
r \defequal 2\kappa_{\bGamma}\
\bracket{\infnorm{\wh{\bSigma}_n^{\textup{param}} - \bSigma^{\natural}} + \lambda_{\textup{param}}} \leq 
\frac{1\vcap \bracket{\kappa_{\bSigma}^2 \kappa_{\bGamma}}^{-1}}{3\kappa_{\bSigma} \textup{deg}} , 
$
then $\infnorm{\wt{\bTheta}_{\textup{param}} - \bTheta^{\natural}} \leq r$.
\end{lemma}

\begin{lemma}[Lemma~$5$ in \cite{ravikumar2011high}]
\label{lemma:ravi2}
Provided that 
we have $\infnorm{\wt{\bTheta}_{\textup{param}} - \bTheta^{\natural}} \leq \bracket{3\kappa_{\bSigma} \textup{deg}}^{-1}$, 
then 
\begin{align*}
\infnorm{\wt{\bTheta}_{\textup{param}}^{-1} -\bTheta^{\natural -1} + 
\bTheta^{\natural -1}\bracket{\wt{\bTheta}_{\textup{param}} - \bTheta^{\natural}} \bTheta^{\natural -1} } 
\leq  \frac{3}{2}\textup{deg}\cdot \
\kappa_{\bSigma}^3\infnorm{\wt{\bTheta}_{\textup{param}} - \bTheta^{\natural}}^2.
\vspace{-4mm}
\end{align*} 	
\end{lemma}

\begin{lemma}[Lemma~$4$ in \cite{ravikumar2011high}]
\label{lemma:ravi3}
If we have  
\begin{align*}
\infnorm{\wh{\bSigma}_n^{\textup{param}} - \bSigma^{\natural}}\vcup \infnorm{\wt{\bTheta}_{\textup{param}}^{-1} -\bTheta^{\natural -1} + 
\bTheta^{\natural -1}\bracket{\wt{\bTheta}_{\textup{param}} - \bTheta^{\natural}} \bTheta^{\natural -1} } 
\leq  \theta \lambda_{\textup{param}}/8, 
\end{align*}
we conclude that $|\wt{G}_{i, j}| < 1$.
\end{lemma}

Now, setting $\lambda_{\textup{param}} = 8 \tau_{\infty}/\theta$,
first 
we verify the conditions in Lemma~\ref{lemma:ravi1}. 
We have 
\begin{align*}
r \stackrel{\cirone}{\leq}  2\bracket{1 + 8\theta^{-1}}\kappa_{\bGamma}\tau_{\infty} 
\stackrel{\cirtwo}{\leq} \bracket{3\kappa_{\bSigma} \textup{deg}}^{-1}
\bracket{1\vcap \bracket{\kappa_{\bSigma}^2 \kappa_{\bGamma}}^{-1}},
\end{align*}
where in $\cirone$ we use 
$\infnorm{\wh{\bSigma}_n^{\textup{param}} - \bSigma^{\natural}} \leq \tau_{\infty}$
from Lemma~\ref{thm:cov_mat_inf_norm}, 
and in $\cirtwo$ we use 
the assumptions of $\tau_{\infty}$ in Theorem~$1$. 
Then we conclude that 
\begin{align*}
\infnorm{\wt{\bTheta}_{\textup{param}} - \bTheta^{\natural}} \leq 
2\kappa_{\bGamma}\
\bracket{\infnorm{\wh{\bTheta}_n - \bTheta^{\natural}} + \frac{8\tau_{\infty}}{\theta}} \leq 
\bracket{3\kappa_{\bSigma} \textup{deg}}^{-1}. 
\end{align*}
\noindent
Invoking Lemma~\ref{lemma:ravi2}, we have 
\begin{align*}
\infnorm{\wt{\bTheta}_{\textup{param}}^{-1} -\bTheta^{\natural -1} + 
\bTheta^{\natural -1}\bracket{\wt{\bTheta}_{\textup{param}} - \bTheta^{\natural}} \bTheta^{\natural -1} } 
\leq \frac{3}{2}\textup{deg}\times \
\kappa_{\bSigma}^3\tau_{\infty}^2
\stackrel{\cirthree}{\leq} \tau_{\infty}, 
\end{align*}
where $\cirthree$ is due to the 
requirement of $\tau_{\infty}$ in Theorem~$1$. In the 
end, we verify the condition in Lemma~\ref{lemma:ravi3},  
\begin{align*}
\theta \lambda_{\textup{param}}/8 = \tau_{\infty} 
\geq \infnorm{\wh{\bSigma}_n^{\textup{param}} - \bSigma^{\natural}}\vcup \infnorm{\wt{\bTheta}_{\textup{param}}^{-1} -\bTheta^{\natural -1} + 
\bTheta^{\natural -1}\bracket{\wt{\bTheta}_{\textup{param}} - \bTheta^{\natural}} \bTheta^{\natural -1} }. 
\end{align*}
which concludes that $\wt{\bTheta}_{\textup{param}}$ is identical to the solution $\wh{\bTheta}_{\textup{param}}$.
\par \noindent
\textbf{Step IV: bound the error.}
For the entries $(i,j)\in S^c$ outside the support set 
$S$, we have $(\wh{\bTheta}_{\textup{param}})_{i, j}$ to be zero due to the
construction method of $\wt{\bTheta}_{\textup{param}}$. 
For the entries $(i,j)\in S$ inside the support set, 
we first upper-bound the 
element-wise deviation 
$\infnorm{\wh{\bTheta}_{\textup{param}} - \bTheta^{\natural}}$
as
\begin{align}
\label{eq:m_glasso_inf_bd_appendix}
\infnorm{\wh{\bTheta}_{\textup{param}} - \bTheta^{\natural}} =
\infnorm{\wt{\bTheta}_{\textup{param}} - \bTheta^{\natural}} 
\stackrel{\cirone}{\leq} 
2\kappa_{\bGamma}(1+8\theta^{-1})\tau_{\infty}, 
\end{align}
where $\cirone$ is due to Lemma~\ref{lemma:ravi1} and has been 
verified in Step II.
Then we prove $\sign(\wh{\bTheta}_{\textup{param}})$
reaches the ground truth 
once $\min_{(i, j)\in S^c}\abs{\bTheta^{\natural}_{i,j}}\geq 2\kappa_{\bGamma}(1+8\theta^{-1})\tau_{\infty}$. 
W.l.o.g. we assume 
$\bTheta^{\natural}_{i,j}  > 0$ for 
$(i, j)\in S$. Then we have 
\begin{align*}
(\wh{\bTheta}_{\textup{param}})_{i,j} \geq 
\abs{\bTheta^{\natural}_{i, j}} - 
\abs{\wh{\bTheta}_{\textup{param}} - \bTheta^{\natural}} 
\geq  \abs{\bTheta^{\natural}_{i, j}} - \
\infnorm{\wh{\bTheta}_{\textup{param}} - \bTheta^{\natural}} 
\stackrel{\cirtwo}{>}  0, 
\end{align*}
and complete the proof, where 
$\cirtwo$ is due to the assumption on
$\min_{(i, j)\in S^c}\abs{\bTheta^{\natural}_{i,j}}$ and 
\eqref{eq:m_glasso_inf_bd_appendix}.


%

\section{Analysis of the CDF Estimation}
\subsection{Notations}
For the conciseness of notation, we drop the subscript $i$
in the marginal CDFs $F_i(\cdot)$, $\wh{F}_i(\cdot)$, and entry $\wh{X}^{(s)}_i$,
and 
denote them as $F(\cdot)$, $\wh{F}(\cdot)$, and $\wh{X}^{(s)}$, respectively.
We define 
the approximate characteristics function $\wh{\phi}_{\wh{w}}(t)$ 
as 
\[
\wh{\phi}_{\wh{w}}(t) \defequal \
\exp\bracket{-\frac{\sigma^2 t^2}{2(d-p)}}.
\]
Additionally, 
we construct the function $\wt{F}(\cdot)$ as 
\[
\wt{F}(x) = \frac{1}{2} - \frac{1}{\pi} \int_{0}^{\infty} 
\frac{1}{t} \Im\Bracket{\frac{\phi_{\wh{w}}(-t) \wh{\phi}_{\wh{X}}(t)}{\abs{\phi_{\wh{w}}(t)}^2 \vcup \gamma t^a} e^{-\ii tx}} dt, 
\]
where the characteristic function $\phi_{\wh{w}}(\cdot)$ denotes the ground-truth characteristics 
function of the noise $\wh{w}$, and 
the function $\wh{\phi}_{\wh{X}}$ is defined as $n^{-1}\sum_{s=1}^n e^{-\ii t\wh{X}^{(s)}}$.

Let the term $\Delta_x^{(s)}$ be 
\begin{align}
\label{eq:deltaxs_def}
\Delta_x^{(s)} = \frac{1}{\pi}\int_{0}^{\infty}
\frac{1}{t}\Im\bracket{\frac{\wh{\phi}_{\wh{w}}(-t)}{\abs{\wh{\phi}_{\wh{w}}(t)}^2 \vcup \gamma t^a} \exp\bracket{\ii t \wh{X}^{(s)}-\ii tx}} dt - 
\Expc\Bracket{\frac{1}{\pi} \int_{0}^{\infty}
\frac{1}{t}\Im\Bracket{\frac{\phi_{\wh{X}}(t)}{\phi_{\wh{w}}(t)} e^{-\ii tx} } dt}.
\end{align}
Our goal is to obtain the  
uniform convergence rate of the CDF estimator 
$\norm{\wh{F} - F}{\infty}$, which is written as   
\[
\norm{\wh{F} - F}{\infty} = 
\sup_{x\in [0, 1]} \abs{\wh{F}(x) - F(x)} = 
\sup_{x\in [0, 1]} \frac{1}{n}\abs{\sum_{s=1}^n \Delta_{x}^{(s)}}
\defequal \Delta_x.
\]

Before proceed, we define the event $\calE_w$ as 
\begin{align}
\label{eq:event_E_def}
\calE_w \defequal \set{\abs{\Expc(\wh{w}_i)^2 - \frac{\sigma^2}{d-p}} \leq  \frac{\sigma^2}{5(d-p)},~~\forall~1\leq i \leq p}, 
\end{align}
where $\wh{w}_i$ denotes the $i$th entry of the noise 
$\wh{\bW}$.

The parameter $\gamma$ is set as $c_0  \log(np)\bracket{\frac{\sigma^2}{d-p}}^{{a}/{4}}$ and 
the number $a > 1$ is some fixed positive constant.

\subsection{Proof of Thm.~\ref{thm:cdf_tail_prob}}
\label{thm_proof:cdf_tail_prob}

\begin{proof} 
We decompose the probability $\sup_x \Delta_x \geq  
\Expc\sup_x \Delta_x+ t$ as 
\[
\Prob\bracket{
\abs{\sup_x \Delta_x - 
\Expc\sup_x \Delta_x}\geq t} \leq 
\Expc \Ind(\br{\calE}_w) + 
\Expc \Ind\Bracket{ \bracket{
\abs{\sup_x \Delta_x - 
\Expc\sup_x \Delta_x}\geq t }\bigcap \calE_w}. 
\]
The first term $\Expc \Ind(\br{\calE}_w)$ is investigated in 
Lemma~\ref{lemma:var_var}; while the second term 
is bounded with the Talagrand inequality 
(cf. Thm.~$2.6$ in \cite{koltchinskii2011oracle}), which is stated as 
\[
\Expc \Ind\Bracket{ \bracket{
\abs{\sup_x \Delta_x - 
\Expc\sup_x \Delta_x}\geq t }\bigcap \calE_w}
\lsim \Expc\Bracket{\exp\bracket{-\frac{nt}{KC_U} \log\bracket{1 + \frac{ntC_U}{V}}}\Ind(\calE_w)}, 
\]
where $C_U$ is the uniform bound for $\Delta^{(s)}_x$, i.e., 
$|\Delta^{(s)}_x|\leq C_U$ for all $x$ and $s$, and
the variance $V$ satisfies 
\[
V\geq n \sup_{x\in \RR}\Expc \abs{\Delta_x}^2 + 16 C_U \cdot \bracket{\Expc \sup_{x} \abs{\Delta_x}}.
\]
Setting $t$ as $V\log n/n$, we conclude 
\[
\sup_x \Delta_x \leq 
\Expc\sup_x \Delta_x + \frac{cV\log n}{n}   
\]
holds with probability exceeding $1-O(n^{-c})$.
The following context focuses on computing the 
values of $C_U$, $\Expc \sup_x \Delta_x$, and $\sup_x \Expc \abs{\Delta_x}^2$
conditional on event $\calE_w$ in \eqref{eq:event_E_def}. 

\textbf{Step I}. We show there exists some positive constant
$C_U$ such that $|\Delta^{(s)}_x|\leq C_U$ for all 
$x$ and $s$, $1\leq s \leq n$ (cf. Lemma~\ref{lemma:delta_x_constant}); 

\textbf{Step II}. We prove that 
$\Expc \sup_x \Delta_x \leq \Expc \Delta_x + {c_0}/{\sqrt{n}}$, 
where $c_0$ is some positive constant. 
With the symmetrization inequalities (cf. Lemma~$11.4$ in \cite{boucheron2013concentration}), we obtain 
\begin{align}
\label{eq:supdeltax_symmetry}	
\Expc \sup_{x} \bracket{\Delta_x - \Expc \Delta_x} \leq  
\Expc \sup_{x} \frac{2}{n}\sum_{s=1}^n \varepsilon_s \Delta^{(s)}_x,   
\end{align}
where $\set{\varepsilon_s}$ are the 
Rademacher RVs, i.e., $\Prob(\varepsilon_s =\pm 1) = 1/2$, 
$1\leq s \leq n$.
With regarding the empirical process 
$n^{-1}\bracket{\sum_s \varepsilon_s \Delta^{(s)}_{x}}$, we have
\[
\Expc e^{\lambda n^{-1}\bracket{\sum_s \varepsilon_s
\bracket{ \Delta^{(s)}_{x_1} -\Delta^{(s)}_{x_2} }}} 
\stackrel{\cirone}{=} \prod_{s=1}^n 
\Expc \exp\Bracket{\frac{\lambda\varepsilon_s}{n} \bracket{ \Delta^{(s)}_{x_1} -\Delta^{(s)}_{x_2} }} \stackrel{\cirtwo}{\leq}
\prod_{s=1}^n \Expc \exp\Bracket{\frac{\lambda^2}{2n^2}\bracket{\Delta^{(s)}_{x_1} - \Delta^{(s)}_{x_2} }^2},
\]
where in $\cirone$ we exploit the independence across the samples, 
in $\cirtwo$ we use the fact that $\varepsilon_s$ is a Rademacher RV and 
the Hoeffding's lemma \cite{boucheron2013concentration} (Lemma~$2.2$).

Invoking the Dudley's entropy integral (cf. Corol~$13.2$ in 
\cite{boucheron2013concentration}), which is also listed as 
Thm.~\ref{thm:dudley_integral} for the sake of self-containing, we conclude  
\begin{align}
\label{eq:dudley_integral_bound}
\Expc\sup_{x, \varepsilon_s} 
\frac{1}{n}\sum_{s=1}^n 
\varepsilon_s \Delta_x^{(s)}
\lsim~&\frac{1}{\sqrt{n}}\int_{0}^{2c_0} \sqrt{\calH([0, 1], \delta)}d\delta 
\stackrel{\cirthree}{\lsim}
\frac{1}{\sqrt{n}}
\int_{0}^{2c_0} \bracket{1 + \sqrt{\log\frac{c_0}{t}}}d\delta
\asymp \frac{1}{\sqrt{n}},  
\end{align}
where $\calH([0, 1], \delta)$ denotes
the $\delta$-entropy number \cite{boucheron2013concentration},  
in $\cirthree$ we upper-bound the $\delta$-entropy number 
$\calH([0, 1], \delta)$ as $c_0 + c_1\log\bracket{b/\delta}$ (cf. Example $5.24$ in \cite{wainwright_2019}). 
The proof is then completed by combining \eqref{eq:supdeltax_symmetry} and \eqref{eq:dudley_integral_bound}.

\textbf{Step III}. 
We set the variance $V$ as
\[
V =~& n \sup_{x\in \RR}\Expc \abs{\Delta_x}^2 + \frac{c_0}{\sqrt{n}}
+ c_1 \sqrt{\Expc \Delta^2_x} \stackrel{\cirfour}{\geq}
n\sup_{x\in \RR}\Expc \abs{\Delta_x}^2
+ \frac{c_0}{\sqrt{n}}
+ c_1\Expc\Delta_x \\
\stackrel{\cirfive}{\geq}~& 
n\sup_{x\in \RR}\Expc \abs{\Delta_x}^2
+ c_0 \Expc \sup_x \Delta_x, 
\] 
where in $\cirfour$ we use the fact $\Expc \Delta_x \leq \sqrt{\Expc \Delta_x^2}$, 
and in $\cirfive$ we use the results in Step II such that 
$\Expc \sup_x \Delta_x \leq \Expc \Delta_x + c_0/\sqrt{n}$.
Invoking Lemma~\ref{lemma:var_var} and Lemma~\ref{lemma:expcdeltasqr} will complete 
the proof. 

\end{proof}

\subsection{Supporting Lemmas}

\begin{lemma}
\label{lemma:var_var}	
The event $\calE_w$ in \eqref{eq:event_E_def} holds 
with probability exceeding $1-2pe^{-c_0 p}$, 
where $c_0$ is some fixed positive constants. 
\end{lemma}

\begin{proof}
Perform SVD for $\bA$ as $\bA = \bU \bS\bV^{\rmt}$,
we can rewrite the product 
$\bracket{\bA^{\rmt}\bA}^{-1}\bA^{\rmt}$ as 
\[
\bracket{\bA^{\rmt}\bA}^{-1}\bA^{\rmt} = 
\bV\bS^{-2}\bV^{\rmt} \bV \bS^{\rmt}\bU^{\rmt}= 
\bV\bS^{-1}\bU^{\rmt}. 
\]
Then the $i$th entry of $\wh{\bW}$ can be 
written as $\ba_i^{\rmt}\bU^{\rmt}\bW$, where $\ba_i$  as 
\[
\ba_i = \Bracket{\lambda_1^{-1}V_{i1}~\dots~\lambda_p^{-1}V_{ip}}, 
\]
and $\lambda_i$ is the $i$th singular value of $\bS$, and 
$V_{ij}$ is the $(i, j)$th entry 
of the matrix $\bV$. Since each entry $A_{ij}$ is iid standard 
normal RV, i.e., $A_{ij}\sim \normdist(0, 1)$, 
we conclude that its eigenvalues $\set{\lambda_i}_{1\leq i \leq p}$
are independent from its eigenvectors 
$\bV$, which is uniformly distributed on 
a Haar measure \cite{tulino2004random}. 
Hence we can rewrite $\ba_i$ as a product as 
$\bS^{-1}\bg/{\norm{\bg_i}{2}}$, where 
$\bg_i$ is 
a Gaussian distributed RV with zero mean 
and unit variance, namely, 
$\bg_i \sim \normdist(0, \bI_{p\times p})$.

First, we define the event $\calE_{g}$ as 
\[
\calE_g \defequal \set{\exists~1\leq i \leq p,~\St~\abs{\norm{\bg_i}{2}^2 - p} \geq p/4}. 
\]
Then we bound 
$\Expc\Ind(\br{\calE}_w)$ as 
\[
\Expc\Ind\bracket{\br{\calE}_w} \leq \Expc\Ind\bracket{\br{\calE}_g}
+ \Expc\Ind\bracket{\br{\calE}_w\bigcap \calE_g}, 
\]
where $\br{(\cdot)}$ denotes the complement of the event. 
Due to the fact that $\bg$ is a Gaussian RV such that 
$\bg \sim \normdist(0,\bI)$, we invoke Lemma~\ref{lemma:chi_square_dev} 
and bound the first term as $\Expc\Ind(\calE_g) \leq 2 p e^{-c_0 p}$. 
While for the second term, we first compute the expectation 
$\Expc \norm{\bS^{-1}\bg_i \bU^{\rmt}\bW}{2}^2$ as 
\[
\Expc \norm{\bS^{-1}\bg_i \bU^{\rmt}\bW}{2}^2 = 
\sigma^2 \Expc \norm{\bS^{-1}\bg_i}{2}^2 = \sigma^2 \fnorm{\bS^{-1}}^2.
\]
Invoking the Marcenko-Pastur law in \cite{tulino2004random}
(cf. Thm.~$2.35$, which is also listed as Lemma~\ref{lemma:marcenko} for the self-containing of paper), we have 
\[
\fnorm{\bS^{-1}}^2 =~&
\sum_{i=1}^p \lambda_{i}^{-2} {\rightarrow} 
\frac{p}{d}  \int_{L(\tau)}^{U(\tau)} t^{-1}
\Bracket{\bracket{1 - \tau^{-1}}_{+} \Ind(t) + 
\frac{\sqrt{[t-L(\tau)][U(\tau)-t]}}{2\pi \tau t}}  dt \\
=~& \frac{p}{d}\times \frac{1}{1-\tau} = \frac{p}{d-p}, 
\]
where $\tau$ is defined as $p/d$, 
function $L(\cdot)$ is defined as $\bracket{1 - \sqrt{\cdot}}^2$, 
and function $U(\cdot)$ is defined as $\bracket{1 + \sqrt{\cdot}}^2$.  
Hence we have $\Expc \norm{\bS^{-1}\bg \bU^{\rmt}\bw }{2}^2 = 
\Expc \fnorm{\bS^{-1}}^2 \sigma^2 = p\sigma^2/(d-p)$ and will 
show $\Expc\Ind\bracket{\br{\calE}_w\bigcap \calE_g}$ to be zero. 
This is because 
\[
\Expc \Bracket{\frac{\norm{\bS^{-1}\bg_i \bU^{\rmt}\bW}{2}^2}{\norm{\bg_i}{2}^2}
 \Ind(\calE_g)} \stackrel{\cirone}{\geq} 
\frac{4}{3p}\Expc \norm{\bS^{-1}\bg_i \bU^{\rmt}\bW}{2}^2=
\frac{4\sigma^2}{3(d-p)}, 
\] 
where $\cirone$ is due to the definition of $\calE_g$. Similarly 
we can show 
\[
\Expc \Bracket{\frac{\norm{\bS^{-1}\bg_i \bU^{\rmt}\bW}{2}^2}{\norm{\bg_i}{2}^2}
 \Ind(\calE_g)} \leq
\frac{4}{5p}\Expc \norm{\bS^{-1}\bg_i \bU^{\rmt}\bW}{2}^2=
\frac{4\sigma^2}{5(d-p)}, 
\]
which suggests that $\calE_w$ will always hold, and complete the proof. 
\end{proof}

\begin{lemma}
\label{lemma:varepsilon_E}	
Conditional on the event $\calE_w$ in \eqref{eq:event_E_def}, 
we have 
\[
\abs{\wh{\phi}_{\wh{w}}(t) - \phi_{\wh{w}}(t)} \lsim
\frac{\log^2(np) \sigma^2 t^2}{d-p}
e^{-\frac{{\sigma}^2 t^2}{2(d-p)}},
\] 
where $\wh{\phi}_{\wh{w}}(t)$ and $\phi_{\wh{w}}(t)$ denotes the 
estimated characteristic function and the 
ground-truth characteristic function of $\wh{w}_i$, respectively.
\end{lemma}

\begin{proof}
Notice that the
 characteristic function of the Gaussian $\normdist(0, \Var)$ is 
 written as 
\[
\phi_{\Var}(t) = \exp\bracket{-\frac{(\Var) t^2}{2}}. 
\]
Then we conclude that 
\[
\abs{\wh{\phi}_{\wh{w}}(t) - \phi_{\wh{w}}(t)} =~& 
\abs{\exp\bracket{-\frac{\sigma^2 t^2}{2(d-p)}} - 
\exp\bracket{-\frac{\wh{\sigma}^2 t^2}{2}} } \\
\stackrel{\cirone}{\leq} ~&
\frac{\abs{\sigma^2/(d-p) - \wh{\sigma}^2}t^2}{2}\cdot 
\bracket{\exp\bracket{-\frac{\wh{\sigma}^2 t^2}{2}}
+ \exp\bracket{-\frac{\sigma^2 t^2}{2(d-p)}}} \\
\stackrel{\cirtwo}{\lsim}~& \frac{\log^2(np) \sigma^2 t^2}{d-p}
e^{-\frac{{\sigma}^2 t^2}{2(d-p)}}, 
\] 
where $\wh{\sigma}^2$ is defined as the ground-truth variance of the 
noise $\wh{w}$, 
$\cirone$ is because of the relation $\abs{e^{-z_1} - e^{-z_2}} \leq \abs{z_1 - z_2}{\abs{e^{-z_1} + e^{-z_2}}}/{2}$ for arbitrary $z_1$ and $z_2$,
$\cirtwo$ is because of event $\calE_w$.
\end{proof}

\begin{lemma}
\label{lemma:delta_x_constant}
Setting $\gamma \asymp  \log(np)\bracket{\frac{\sigma^2}{d-p}}^{{a}/{4}}$, we conclude 
$|\Delta^{(s)}_x| \lsim 1$ for all $s$, $1\leq s \leq n$, where 
$\Delta^{(s)}_x$ is defined in \eqref{eq:deltaxs_def}. 
\end{lemma}

\begin{proof}

We verify that $\Delta^{(s)}_x$ to be bounded by some constant, which is 
written as 
\begin{align}
\label{eq:delta_x_constant_main}
\Delta^{(s)}_{x} \leq 
\abs{\frac{1}{\pi}\int_{0}^{\infty}
\frac{1}{t}\Im\bracket{\frac{\wh{\phi}_{\wh{w}}(-t)}{\abs{\wh{\phi}_{\wh{w}}(t)}^2 \vcup \gamma t^a} \exp\bracket{\ii t\bracket{\wh{X}^{(s)}-x}}} dt} + \underbrace{F(x)}_{\leq 1}. 
\end{align}
Defining the term $t_{\bot}$  as 
$\bracket{2\gamma}^{-1/a}$, 
we split the whole region $(0,\infty)$ as three disjoint 
sub-regions $\calR_1 = (0, t_{\bot})$ 
and $\calR_2 =  (t_{\bot}, \infty)$. Then we perform the decomposition as
\begin{align}
\label{eq:delta_x_constant_decompose_integral}
&\int_{0}^{\infty}
\frac{1}{t}\Im\bracket{\frac{\wh{\phi}_{\wh{w}}(-t)}{\abs{\wh{\phi}_{\wh{w}}(t)}^2 \vcup \gamma t^a} \exp\bracket{\ii t\wh{X}^{(s)}-\ii tx}} dt \notag \\
=~& 
\underbrace{\int_{\calR_1 }
\frac{1}{t}\Im\bracket{\frac{\wh{\phi}_{\wh{w}}(-t)}{\abs{\wh{\phi}_{\wh{w}}(t)}^2 \vcup \gamma t^a} \exp\bracket{\ii t\wh{X}^{(s)}-\ii tx}} dt}_{T_1} \notag \\
+~& \underbrace{\int_{\calR_2}
\frac{1}{t}\Im\bracket{\frac{\wh{\phi}_{\wh{w}}(-t)}{\abs{\wh{\phi}_{\wh{w}}(t)}^2 \vcup \gamma t^a} \exp\bracket{\ii t\wh{X}^{(s)}-\ii tx}} dt}_{T_2},
\end{align}
Then we separately bound the terms $T_1$ and $T_2$. 

\textbf{Step I}.
We can bound term $T_1$ as 
\begin{align}
\label{eq:delta_x_constant_t1_bound}
T_1 = ~&
\int_{\calR_1}
\frac{1}{t}\Im\bracket{
{\wh{\phi}_{\wh{w}}(-t)}
\Im\Bracket{\bracket{ \exp\bracket{\ii t\wh{X}^{(s)}-\ii tx}}
\Bracket{1 + \sum_{k=1}^{\infty}\bracket{1 - \bracket{\abs{\wh{\phi}_{\wh{w}}(t)}^2 \vcup \gamma t^a}}^k
} }}dt \notag \\
\stackrel{\cirone}{\leq}~&
\int_{\calR_1} \frac{\sin\bracket{t\bracket{\wh{X}^{(s)}-x}}}{t}dt + 
\int_{\calR_1}
\frac{1}{t} \sum_{k=1}^{\infty}\bracket{\frac{\sigma^2 t^2}{d-p}}^k dt\stackrel{\cirtwo}{\lsim}1 + \sum_{k=1}^{\infty}
\frac{1}{2k}
\bracket{\frac{\sigma^2 t_{\bot}^2}{d-p}}^k \stackrel{\cirthree}{\lsim} 1,
\end{align}
where in $\cirone$ we use the fact 
$\abs{{\phi}_{\wh{w}}(t)}^2 \vcup \delta t^a \geq 1 - \frac{c_1 \sigma^2 t^2}{d-p}$
for $t\in \calR_1$, 
in $\cirtwo$ we use the fact $\sup_{\tau > 0}\abs{\int_0^{\tau} \sin(u)/u du}\leq 3$, and $\cirthree$ is because $\sigma^2 t_{\bot}^2 \asymp \sqrt{d-p}(\log(np))^{-\frac{2}{a}}\leq 
\frac{d-p}{2}$
and hence $\sum_{k=1}^{\infty}
\frac{1}{2k}
\bracket{\frac{\sigma^2 t_{\bot}^2}{d-p}}^k \leq  \sum_{k=1}^{\infty}
\bracket{\frac{\sigma^2 t_{\bot}^2}{d-p}}^k\lsim 1$. 


\textbf{Step II}. 
For term $T_3$, we have 
\begin{align}
\label{eq:delta_x_constant_t2_bound}
T_2 \leq \int_{\calR_2} \frac{\wh{\phi}_{\wh{w}}(-t)}{t\sqrt{\abs{\wh{\phi}_{\wh{w}}(-t)}^2 \vcup \delta t^a}}dt
\stackrel{\cirfour}{\leq} \int_{\calR_2}\frac{1}{\sqrt{\gamma} t^{1+\frac{a}{2}} }dt = \
\frac{1}{a \sqrt{\gamma} t_{\bot}^{a/2}} =\frac{1}{\sqrt{2}a}
\asymp 1,  
\end{align}
where $\cirfour$ is because $\wh{\phi}_{\wh{w}}(\cdot) \leq 1$, 
and complete the proof by 
summaring \eqref{eq:delta_x_constant_main}, 
\eqref{eq:delta_x_constant_decompose_integral}
and \eqref{eq:delta_x_constant_t2_bound}.
	
\end{proof}

\begin{lemma}
\label{lemma:expcdeltasqr}
Under the setting of Thm.~\ref{thm:cdf_tail_prob}, we have  
\[
\Expc \bracket{\Delta_{x}}^2 \lsim 
\frac{\log^{{2}/{a}}(np) \sigma}{\sqrt{d-p}}+
\frac{(\log(np))^2}{(d-p)^{\frac{a}{4}}}+
\bracket{\frac{\sigma^2}{d-p}}^{\frac{2\alpha + 1}{4}} +
\frac{1}{n},
\] 
when setting 
$\gamma \asymp  \log(np)\bracket{\frac{\sigma^2}{d-p}}^{{a}/{4}}$
\end{lemma}

\begin{proof}
The proof largely follows \cite{phuong2020deconvolution}. 
However, extra measurements are
required to estimate the characteristic function $\wh{\phi}_{\wh{w}}(\cdot)$ 
in \cite{phuong2020deconvolution}, 
which leads to a simple form of the error $|\wh{\phi}_{\wh{w}}(t) - \phi_{\wh{w}}(t)|$ only depending on the number of extra measurements. 
While our setting does not need these additional measurements and the 
error $|\wh{\phi}_{\wh{w}}(t) - \phi_{\wh{w}}(t)|$ varies with 
$t$. 

With the decomposition 
\[
\Expc \abs{\wh{F}(x) - F(x)}^2 \leq 2 \Expc \abs{\wh{F}(x) - \wt{F}(x)}^2 + 
2 \Expc\abs{\wt{F}(x) - F(x)}^2, 
\]
we complete the proof by invoking Lemma~\ref{lemma:T1_upper_bound}
and Lemma~\ref{lemma:T2_upper_bound}. 

\end{proof}

\begin{lemma}
\label{lemma:T1_upper_bound}
Under the setting of Thm.~\ref{thm:cdf_tail_prob}, 
we can upper-bound the deviation $\Expc \abs{\wh{F}(x) - \wt{F}(x)}$ as
\vspace{-2mm}  
\[
\Expc \abs{\wh{F}(x) - \wt{F}(x)}^2 \lsim \
\frac{\log^{{2}/{a}}(np) \sigma}{\sqrt{d-p}} ,
\]
when setting 
$\gamma \asymp  \log(np)\bracket{\frac{\sigma^2}{d-p}}^{{a}/{4}}$.
 
\end{lemma}

\begin{proof}
First we expand the term $\wh{F}(x) - \wt{F}(x)$ as 
\[
\wh{F}(x) - \wt{F}(x) =~& 
\frac{1}{\pi}
\int_{0}^{\infty}
\frac{1}{t} \Im\Bracket{
\bracket{
\frac{\wh{\phi}_{\wh{w}}(-t) }{\abs{\wh{\phi}_{\wh{w}}(t)}^2 \vcup \gamma t^a} - 
\frac{\phi_{\wh{w}}(-t)}{\abs{{\phi}_{\wh{w}}(t)}^2 \vcup \gamma t^a}
} \wh{\phi}_{\wh{X}}(t) e^{-\ii tx}} dt \\
\stackrel{\cirone}{=}~&
\frac{1}{n\pi}
\int_{0}^{\infty} \frac{D(t)}{t}
\sum_{s=1}^n 
\sin\bracket{t\bracket{\wh{X}^{(s)} - x}}dt 
\]
where in $\cirone$ we define $D(t)$ as 
\[
D(t) \defequal 
\frac{\wh{\phi}_{\wh{w}}(-t) }{|\wh{\phi}_{\wh{w}}(t)|^2 \vcup \gamma t^a} - 
\frac{\phi_{\wh{w}}(-t)}{\abs{{\phi}_{\wh{w}}(t)}^2 \vcup \gamma t^a}. 
\]
According to \cite{phuong2020deconvolution}, it satisfies the relation 
\begin{align}
\label{eq:dt_up_bound}
\abs{D(t)} \leq \frac{2\varepsilon_E(t)}{\sqrt{|{\wh{\phi}}_{\wh{w}}(t)|^2 \vcup \gamma t^a } \cdot \sqrt{
\abs{{\phi}_{\wh{w}}(t)}^2 \vcup \gamma t^a}} + 
\frac{\varepsilon_E(t)}{
{\abs{{\phi}_{\wh{w}}(t)}^2 \vcup \gamma t^a}}, 
\end{align}
where $\varepsilon(t)$ is defined as $\abs{{\wh{\phi}}_{\wh{w}}(t) - {\phi}_{\wh{w}}(t)}$, which is upper bounded by Lemma~\ref{lemma:varepsilon_E}.

Define the terms $I_1$ and $I_2$ as 
\[
I_1 \defequal~& \Expc \bracket{\frac{1}{n\pi}
 \int_0^{t_{\bot}}
\abs{ \frac{D(t)}{t}
\sum_{s=1}^n 
\sin\bracket{t\bracket{\wh{X}^{(s)} - x}}}dt}^2, \\
I_2 \defequal~& \Expc \bracket{\frac{1}{n\pi}
 \int_{t_{\bot}}^{\infty}
\abs{ \frac{D(t)}{t}
\sum_{s=1}^n 
\sin\bracket{t\bracket{\wh{X}^{(s)} - x}}}dt}^2, 
\]
respectively, 
where $t_{\bot}$ is defined as $((d-p)/\sigma^2)^{1/4} \log^{1/a}(np)$. 
We upper-bound the term $\Expc \abs{\wh{F}(x) - \wt{F}(x)}^2$
as 
\begin{align}
\label{eq:lemma:T1_upper_bound_decompose}
\Expc \abs{\wh{F}(x) - \wt{F}(x)}^2  \leq 
2I_1 + 2I_2.
\end{align}

\textbf{Stage I}.
We bound the term $I_1$ as 
\[
I_1 \stackrel{\cirtwo}{\leq}~& \
\bracket{\frac{1}{\pi}
\int_{0}^{t_{\bot}} \Bracket{\abs{\frac{D(t)}{t}} \times 
n^{-1}\bracket{\sum_{s=1}^n \abs{t\bracket{\wh{X}^{(s)} - x}}} }
dt}^2 
\stackrel{\cirthree}{\lsim} \bracket{\int_{0}^{t_{\bot}}\abs{D(t)}dt}^2, 
\]
where in $\cirtwo$ we use the fact $\sin(\cdot) \leq \abs{\cdot}$, 
in $\cirthree$ we use the fact 
$|\wh{X}^{(s)} - x|\leq 1$ for all $s$,~$1\leq s \leq n$. 

Notice in the region $\calR_1$, we can lower bound the 
function $\phi_{\wh{w}}(\cdot) \geq c_1$ as
\[
\abs{{\phi}_{\wh{w}}(t)} \stackrel{\cirfour}{\geq} \abs{\phi_{\wh{w}}(t_{\bot})} = O\Bracket{\exp\bracket{-\frac{c_0 \sigma^2}{d-p}\cdot \frac{\sqrt{d-p}}{\sigma} }} = O(1), 
\]
where in $\cirfour$ we use the fact such that 
$\abs{\phi(\cdot)_{\wh{w}}}$ is non-increasing. 
Then we invoke  \eqref{eq:dt_up_bound} and bound $D(t)$ as $
\abs{D(t)} \lsim \varepsilon_E(t)$, 
since $\abs{\phi_{\wh{w}}(t)}^2 \vcup \gamma t^a\geq \abs{\phi_{\wh{w}}(t)}^2 \gsim 1$. 
Hence term $I_1$ is upper-bounded as 
\begin{align}
\label{eq:lemma:T1_upper_bound_decompose_I1}
I_1 \leq~&
\Bracket{\int_{0}^{t_{\bot}} \frac{\sigma^2 \log^2(np) t^2}{d-p}e^{-\frac{\sigma^2t^2}{2(d-p)}}dt}^2 \
\stackrel{\cirfive}{\leq} 
\log^4(np)\bracket{\int_{0}^{\log^{1/a}(np)(\frac{\sigma^2}{d-p})^{1/4}} \xi^2 e^{-\xi^2/2}d\xi }^2, \notag \\
\leq ~& \frac{\sigma^2\log^{4+\frac{2}{a}}(np)}{d-p} \bracket{\int_{0}^{\infty} e^{-\xi^2/2}d\xi }^2 \asymp 
 \frac{\sigma^2\log^{4+\frac{2}{a}}(np)}{d-p},
 \end{align}
where in $\cirfive$ we use the substitution 
$\xi = \sigma t/\sqrt{d-p}$.

\textbf{Stage II}. 
We define the function $\Lambda(t)$ as 
\[
\Lambda(t) = \frac{1}{n}\sum_{s=1}^n 
\sin\Bracket{t\bracket{\wh{X}^{(s)} - x}} - 
\Im\bracket{\phi_{\wh{X}}(t)e^{-\ii tx}}, 
\]
and bound the term $I_2$ as 
\[
I_2 = 
\Expc\abs{\int_{t_{\bot}}^{\infty} 
\frac{D(t)}{t}\Bracket{
\Lambda(t) + \Im\bracket{\phi_{\wh{X}}(t) e^{-\ii tx}}
}}^2 \lsim I_{2,1} + I_{2,2}, 
\]
where $I_{2,1}$ and $I_{2,2}$ are defined as 
\[
I_{2,1} &\defequal \Expc\abs{\int_{t_{\bot}}^{\infty} \frac{D(t)}{t} 
\Im\bracket{\phi_{\wh{X}}(t) e^{-\ii tx}}
}^2, \\
I_{2, 2} &\defequal \Expc\abs{\int_{t_{\bot}}^{\infty} \frac{D(t) \Lambda(t)}{t}}^2.
\]
Notice that within region $\calR_2$, 
we can upper-bound $\abs{D(t)}$ as 
$\abs{D(t)}\lsim \frac{\varepsilon_E(t)}{\gamma t^a}$ and hence 
\[
I_{2,1} \leq~& 
\abs{\int_{t_{\bot}}^{\infty} 
\frac{\abs{\phi_{\wh{X}}(t)} \sqrt{\Expc |D(t)|^2}}{t}
dt}^2 \stackrel{\cirsix}{\leq} \
\abs{\int_{t_{\bot}}^{\infty} 
\frac{\sqrt{\Expc |D(t)|^2}}{t}
dt}^2 
\stackrel{\cirseven}{\leq}
\bracket{\int_{t_{\bot}}^{\infty}\frac{\varepsilon_E(t)}{\gamma t^{1+a}}dt }^2 \\
=~& \frac{\log^4(np)\sigma^4}{(d-p)^2\gamma^2}
\bracket{\int_{t_{\bot}}^{\infty}\exp\bracket{-\frac{\sigma^2 t^2}{2(d-p)}} t^{1-a} dt}^2 \\
=~& \frac{\log^4(np)\sigma^{2a}}{(d-p)^a\gamma^2}
\bracket{\int_{\log^{1/a}(np)(\frac{\sigma^2}{d-p})^{1/4}}^{\infty} \xi^{1-a} e^{-\xi^2/2}d\xi }^2 \\
\stackrel{\cireight}{\leq}~&  \frac{\log^4(np)\sigma^{2a}}{(d-p)^a\gamma^2}
\log^{\frac{2}{a}-2}(np)\bracket{\frac{\sigma^2}{d-p}}^{\frac{1-a}{2}}
\bracket{\int_{0}^{\infty}e^{-\xi^2/2}d\xi }^2 \\
\asymp ~& \frac{\log^{{2}/{a}}(np) \sigma}{\sqrt{d-p}}, 
\]
where in $\cirsix$ we use the fact $|\phi_{\wh{X}}(\cdot)| \leq 1$, 
in $\cirseven$ we use the bound $|D(t)| \leq \varepsilon_E/(\gamma t^a)$, 
and in $\cireight$ we use the assumption $a > 1$. 

Afterwards, we bound term $I_{2,2}$ as  
\[
I_{2, 2} \leq~& 
2\Expc\abs{
\int_{t_{\bot}}^{\infty} \int_{t_{\bot}}^{\infty} \
\frac{D(u)D(v)\Lambda(u)\Lambda(v)}{u v}
du dv} \\
\leq~& 
2\int_{t_{\bot}}^{\infty} \int_{t_{\bot}}^{\infty}
\frac{\sqrt{\Expc |D(u)|^2} \sqrt{\Expc |D(v)|^2} \cdot \Expc\bracket{\Lambda(u)\Lambda(v)}}{uv}du dv.  
\]
According to Lemma~$5.1$ in \cite{phuong2020deconvolution}, we can bound the term
$\Expc\bracket{\Lambda(u) \Lambda(v)}$ as 
\[
\Expc\bracket{\Lambda(u) \Lambda(v)}
=~&\frac{1}{2n}
\Bracket{\Re\Bracket{e^{\ii (v-u)t} \phi_{\wh{X}}(u-v)}
- \Re\Bracket{e^{-\ii (u+v)x} \phi_{\wh{X}}(u+v)} - 
2\Im\Bracket{e^{-\ii ux}\phi_{\wh{X}}(u) }\cdot
\Im\Bracket{e^{-\ii vx} \phi_{\wh{X}}(s)} 
} \\
\leq~& \frac{1}{2n}
\Bracket{\abs{\phi_{\wh{X}}(u+v)} + \abs{\phi_{\wh{X}}(u-v)} + 2\abs{\phi_{\wh{X}}(u)}\abs{\phi_{\wh{X}}(v)}} \stackrel{\cirnine}{\leq} \frac{2}{n}, 
\]	
where in $\cirnine$ we use the fact 
$|\phi_{\wh{X}}(\cdot)| \leq 1$. 

Following the same strategy as above, we can upper-bound $\abs{D(t)}$ as 
$\abs{D(t)}\lsim \frac{\varepsilon_E(t)}{\gamma t^a}$ and hence
\[
\sqrt{\Expc |D(u)|^2} \sqrt{\Expc |D(v)|^2}
\leq \frac{\varepsilon_E(u)\varepsilon_E(v)}{\gamma^2 u^a v^{a}}.
\]
Combing the above then yields the bound
\[
I_{2,2}\lsim 
\frac{1}{n\gamma^2} \int_{t_{\bot}}^{\infty}\int_{t_{\bot}}^{\infty}  \frac{\varepsilon_E(u)\varepsilon_E(v)}{u^{1+a} v^{1+a}}du dv = 
\frac{\log^{4/a}(np) \sigma^{2+a}}{n(d-p)}.
\]
To sum up, we have  
\[
I_2 & = I_{2, 1} + I_{2, 2}
\lsim \
\frac{\log^{{2}/{a}}(np) \sigma^{1+\frac{a}{2}}}{\sqrt{d-p}} + \
\frac{\log^{4/a}(np) \sigma^{2}}{n(d-p)} \asymp 
\frac{\log^{{2}/{a}}(np) \sigma }{\sqrt{d-p}}, 
\]	
and complete the proof by combining it with
\eqref{eq:lemma:T1_upper_bound_decompose} and 
\eqref{eq:lemma:T1_upper_bound_decompose_I1}. 

\end{proof}



\begin{lemma}
\label{lemma:T2_upper_bound}
Under the setting of Thm.~\ref{thm:cdf_tail_prob}, 
we can upper-bound the deviation 
$\abs{\wt{F}(x) - F(x)}$ as 
\[
\Expc\abs{\wt{F}(x) - F(x)}^2 \lsim \
\frac{(\log(np))^2}{(d-p)^{\frac{a}{4}}}+
\bracket{\frac{\sigma^2}{d-p}}^{\frac{2\alpha + 1}{4}} +
\frac{1}{n},
\]
when setting $\gamma \asymp  \log(np)\bracket{\frac{\sigma^2}{d-p}}^{{a}/{4}}$.
\end{lemma}

\begin{proof}
We decompose the deviation  
$\Expc|\wt{F}(x) - F(x)|^2$ as 
the the bias and variance, which are 
defined respectively as 
\[
\textup{Bias} \defequal & ~\abs{\Expc \wt{F}(x) - F(x)}^2, \\
\textup{Variance} \defequal & ~\Expc \abs{\wt{F}(x) - \Expc \wt{F}(x)}^2.
\]
The following context separately bound the bias and 
variance.  

\paragraph{Bounding bias.}
We rewrite the difference 
$\Expc \wt{F}(x) - F(x)$ as 
\[
\Expc \wt{F}(x) - F(x) = \
\frac{1}{\pi}\int_{0}^{\infty}
\frac{1}{t}
\Im\Bracket{\bracket{\frac{\abs{{\phi}_{\wh{w}}(t)}^2 \phi_X(t)}{\abs{{\phi}_{\wh{w}}(t)}^2 \vcup \gamma t^a} - \phi_X(t)} e^{-\ii tx}}dt,  
\]
which yields 
\[
\textup{Bias} \leq~&  \frac{1}{\pi^2}
\Bracket{\int_{0}^{\infty}
\frac{\abs{\phi_X(t)}}{t}
\Im\Bracket{\bracket{\frac{\abs{{\phi}_{\wh{w}}(t)}^2 }{\abs{{\phi}_{\wh{w}}(t)}^2 \vcup \gamma t^a} - 1} e^{-\ii tx}}dt}^2 \\
\leq~& \frac{1}{\pi^2} 
\Bracket{\int_{0}^{\infty}
\frac{\abs{\phi_X(t)}}{t}
\abs{1-\frac{\abs{{\phi}_{\wh{w}}(t)}^2 }{\abs{{\phi}_{\wh{w}}(t)}^2 \vcup \gamma t^a}} dt}^2 \\
\lsim~& \
\Bracket{\int_{0}^{t_{\square}}
\frac{\abs{\phi_X(t)}}{t}
\abs{1-\frac{\abs{{\phi}_{\wh{w}}(t)}^2 }{\abs{{\phi}_{\wh{w}}(t)}^2 \vcup \gamma t^a}} dt + \
\int_{t_{\square}}^{\infty}
\frac{\abs{\phi_X(t)}}{t}
\abs{1-\frac{\abs{{\phi}_{\wh{w}}(t)}^2 }{\abs{{\phi}_{\wh{w}}(t)}^2 \vcup \gamma t^a}} dt}^2
\leq J_1^2 + J_2^2, 
\]
where $t_{\square}$ is defined as $c_0\bracket{\frac{d-p}{\sigma^2}}^{1/8}$, and terms $J_1$ and $J_2$ are defined as 
\[
J_1 \defequal & \int_{0}^{t_{\square}}
\frac{\abs{\phi_X(t)}}{t}
\abs{1-\frac{\abs{{\phi}_{\wh{w}}(t)}^2 }{\abs{{\phi}_{\wh{w}}(t)}^2 \vcup \gamma t^a}} dt, \\
J_2 \defequal & \int_{t_{\square}}^{\infty}
\frac{\abs{\phi_X(t)}}{t}
\abs{1-\frac{\abs{{\phi}_{\wh{w}}(t)}^2 }{\abs{{\phi}_{\wh{w}}(t)}^2 \vcup \gamma t^a}} dt, 
\]
respectively.
For the term $J_1$, we have 
\[
1-\frac{\abs{{\phi}_{\wh{w}}(t)}^2 }{\abs{{\phi}_{\wh{w}}(t)}^2 \vcup \gamma t^a} \
\leq \frac{\bracket{\abs{{\phi}_{\wh{w}}(t)}^2 \vcup \gamma t^a} - \abs{{\phi}_{\wh{w}}(t)}^2}{\abs{{\phi}_{\wh{w}}(t)}^2 \vcup \gamma t^a} \
\stackrel{\cirone}{\leq} 
\frac{\gamma t^a}{\abs{{\phi}_{\wh{w}}(t)}^2 \vcup \gamma t^a} 
\stackrel{\cirtwo}{\leq} 
\frac{\gamma t^a}{c_R}, 
\]
where in $\cirone$ we use the relation 
$\abs{{\phi}_{\wh{w}}(t)}^2 \vcup \gamma t^a - \abs{{\phi}_{\wh{w}}(t)}^2 \leq \gamma t^a$; in $\cirtwo$ we use the fact 
$\abs{{\phi}_{\wh{w}}(t)}^2 \vcup \gamma t^a \geq c_R$ in the 
regime $[0, t_{\square})$, which can be easily verified. Then we obtain 
\begin{align}
\label{eq:T2_upper_bound_J1}
J_1 \leq \int_{0}^{t_{\square}} 
\frac{\abs{\phi_X(t)}}{t}
\frac{\gamma t^a}{c_R} dt = 
\frac{\gamma}{c_R} \int_{0}^{t_{\square}} t^{a-1}dt 
= \frac{\gamma}{c_R}\frac{t_{\square}^a}{a} \asymp
\frac{\sigma^{\frac{a}{4}}\log(np)}{(d-p)^{\frac{a}{8}}a}.
\end{align}
Afterwards, we bound term $J_2$ as 
\begin{align}
\label{eq:T2_upper_bound_J2}
J_2 \leq 
\int_{t_{\square}}^{\infty} \frac{\abs{\phi_X(t)}}{t}dt \
\stackrel{\cirthree}{\lsim} 
t_{\square}^{-(2\alpha+1)/16} 
\asymp \bracket{\frac{\sigma^2}{d-p}}^{\frac{2\alpha + 1}{8}}, 
\end{align}
where in $\cirthree$ we use the Lemma~$6$ from \cite{phuong2020deconvolution} since 
$\phi_X(\cdot)$ satisfies the Assumption \ref{assump:charx}. 
Combining \eqref{eq:T2_upper_bound_J1} and \eqref{eq:T2_upper_bound_J2} 
then yields  
\begin{align}
\label{eq:T2_upper_bound_bias}
\textup{Bias} \lsim 
\frac{(\log(np))^2}{(d-p)^{\frac{a}{4}}}+
\bracket{\frac{\sigma^2}{d-p}}^{\frac{2\alpha + 1}{4}}.
\end{align}

\paragraph{Bounding variance.}
We bound the $\Var \wt{F}(x)$ as 
\[
& \Var \wt{F}(x) \\
=~& \Var
\Bracket{\frac{1}{\pi} \
\int_{0}^{\infty} 
\frac{1}{t}
\Im\bracket{\frac{{\phi}_{\wh{w}}(-t) }{\abs{{\phi}_{\wh{w}}(t)}^2 \vcup \gamma t^a} \wh{\phi}_{\wh{X}}(t) e^{-\ii tx}}dt} \\
\stackrel{\cirfour}{\leq}~&
\frac{1}{n\pi^2}
\Expc\Bracket{\int_{0}^{\infty} 
\frac{1}{t}\Im\bracket{\frac{{\phi}_{\wh{w}}(-t) }{\abs{{\phi}_{\wh{w}}(t)}^2 \vcup \gamma t^a} e^{\ii t(\wh{X}-x)}}dt}^2 \leq \frac{2}{n\pi^2}\bracket{K_1 + K_2}, 
\]
where in $\cirfour$ we use the bound 
$\Var\bracket{\cdot} \leq \Expc (\cdot)^2$, and the terms 
$K_1$ and $K_2$ are defined as 
\[
K_1 \defequal &~\Expc\Bracket{\int_{0}^{t_{\bot}} 
\frac{1}{t}\Im\bracket{\frac{{\phi}_{\wh{w}}(-t) }{\abs{{\phi}_{\wh{w}}(t)}^2 \vcup \gamma t^a} e^{\ii t(\wh{X}-x)}}dt}^2, \\
K_2 \defequal &~\Expc\Bracket{\int_{t_{\bot}}^{\infty} 
\frac{1}{t}\Im\bracket{\frac{\phi_{\wh{w}}(-t) }{\abs{\phi_{\wh{w}}(t)}^2 \vcup \gamma t^a} e^{\ii t(\wh{X}-x)}}dt}^2, 
\]
and $t_{\bot}$ is defined as $\bracket{2\gamma}^{-{1}/{a}}$. 
First, we bound $K_1$ as 
\begin{align}
\label{eq:T2_upper_bound_K1}
K_1 =~& \Expc \abs{\int_{0}^{t_{\bot}} 
\frac{1}{t}\Im\bracket{
\phi_{\wh{w}}(-t) e^{\ii t(\wh{X}-x)} \bracket{1 + 
\sum_{k=1}^{\infty}\bracket{1 - \bracket{\abs{{\phi}_{\wh{w}}(t)}^2\vcup \gamma t^a}}^k}
}dt
}^2 \notag \\
\lsim ~& 
\Expc \abs{\int_{0}^{t_{\bot}} 
\frac{1}{t}\Im\bracket{
\phi_{\wh{w}}(-t) e^{\ii t(\wh{X}-x)}
}dt
}^2 + 
\Expc \abs{\int_{0}^{t_{\bot}} 
\frac{1}{t}\sum_{k=1}^{\infty}
\bracket{1 - \bracket{\abs{{\phi}_{\wh{w}}(t)}^2\vcup \gamma t^a}}^k dt
}^2 \notag \\
\stackrel{\cirfive}{\leq}~& \
\Expc \sup_{\tau > 0} \abs{\int_{0}^{\tau} \frac{\sin t}{t}dt}^2+ \
 \Expc \abs{\int_{0}^{t_{\bot}} 
\frac{1}{t}\sum_{k=1}^{\infty}
\bracket{1 - \abs{{\phi}_{\wh{w}}(t)}^2}^k dt
}^2 \notag \\ 
\leq~& 3 + 
\Expc\Bracket{\int_{0}^{t_{\bot}} \frac{1}{t} \sum_{k=1}^{\infty}\bracket{\frac{\sigma^2t^2}{d-p}}^k }^2
\stackrel{\cirsix}{\lsim} 1, 
\end{align}
where in $\cirfive$ we use the fact $\abs{{\phi}_{\wh{w}}(t)}^2 \vcup \delta t^a \geq 1 - \frac{c_1 \sigma^2 t^2}{d-p}$
for $t\in (0, t_{\bot})$, and in 
$\cirsix$ we use $\sigma^2 t^2_{\bot} \asymp \sqrt{d-p}(\log(np))^{-\frac{2}{a}}\leq 
\frac{d-p}{2}$.

Then we expand the term $K_2$ as a product 
of two terms reading as 
\begin{align}
\label{eq:T2_upper_bound_K2}
K_2 =~&\
\Expc \Bracket{
\int_{t_{\bot}}^{\infty}\int_{t_{\bot}}^{\infty} 
\frac{1}{uv} \Im\bracket{ 
\frac{\phi_{\wh{w}}(-u)\phi_{\wh{w}}(-v)e^{\ii (u+v)(\wh{X}-x)}}{
\Bracket{\abs{\phi_{\wh{w}}(u)}^2 \vcup \gamma u^a} \Bracket{\abs{\phi_{\wh{w}}(v)}^2 \vcup \gamma v^a}
}}
du dv } \notag \\
\leq ~& \Expc \Bracket{
\int_{t_{\bot}}^{\infty} \int_{t_{\bot}}^{\infty}
\frac{{\phi}_{\wh{w}}(-u) {\phi}_{\wh{w}}(-v)}{uv}\cdot 
\frac{1}{\abs{\phi_{\wh{w}}(u)}^2 \vcup \gamma u^a} \cdot 
\frac{1}{\abs{\phi_{\wh{w}}(v)}^2 \vcup \gamma v^a}dudv} \notag \\
\leq~& \frac{1}{\gamma^2}\Expc\Bracket{
\int_{t_{\bot}}^{\infty} \frac{1}{t^{a+1}}dt
}^2 = \
\frac{1}{a^2 \gamma^2 t^{2a}_{\bot}} \asymp 1. 
\end{align}
Combining \eqref{eq:T2_upper_bound_K1} and \eqref{eq:T2_upper_bound_K2}
generates 
\begin{align}
\label{eq:T2_upper_bound_var}
\Var \wt{F}(x)\lsim n^{-1}. 
\end{align}
And the proof is completed by 
combining \eqref{eq:T2_upper_bound_bias} and 
\eqref{eq:T2_upper_bound_var}.

\end{proof}

\section{Analysis of the Graphical Structure Model}

Denote the empirical covariance matrix $\bSigma_n$ as 
\[
\bSigma^{\textup{non-param}}_n \defequal \frac{1}{n}\sum_{s=1}^n 
\bh(\bX^{(s)}) \bh(\bX^{(s)})^{\rmt} - 
\bracket{\frac{1}{n}\sum_{s=1}^n 
\bh(\bX^{(s)})} \bracket{\frac{1}{n}\sum_{s=1}^n 
\bh(\bX^{(s)})}^{\rmt}, 
\]
where $h(\cdot)$ denotes the oracle estimator of the 
transform functions in Def.~\ref{def:nonparanormal}.  
We first analyze the estimation error of the covariance 
matrix $\bSigma^{\textup{non-param}}_n$ in terms of the infinity norm $\ell_{\infty}$.  
\begin{theorem}
\label{thm:graph_error_infnorm}
Under the Assumption~\ref{assump:charx} and Assumption~\ref{assump:lipschitz}, we
have  
\[
\infnorm{\bSigma^{\textup{non-param}}_n - \wh{\bSigma}^{\textup{non-param}}_n} \lsim 
\sqrt{\log n\vcup \log(d-p)} \bracket{\frac{\sqrt{\log n}}{n^{1/4}}\vcup 
\frac{\sqrt{\log(d-p)}}{(d-p)^{\beta/4}}},
\]
with probability exceeding 
$1 - o(1)$, where $\delta_{n, d, p}$ is set as \eqref{eq:deltandp_def},  $\beta$ is defined as 
$\frac{1}{2}\vcap \frac{a}{4} \vcap \frac{2\alpha + 1}{4}$, and 
$\wh{\bSigma}^{\textup{param}}_n$ is defined in 
\eqref{eq:nonparam_emp_cov_estim}. 
\end{theorem}

The proof largely follows \cite{liu2009nonparanormal}. However, we cannot 
directly access the samples $\set{\bX^{(s)}}_{1\leq s \leq n}$ and have to use 
the perturbed samples $\{\wh{\bX}^{(s)}\}_{1\leq s \leq n}$ instead.
This will lead to additional errors in estimating the covariance matrix
and how to bound these errors constitutes the technical bottleneck.

\subsection{Notations}
We assume that the correct estimation of $m_i = 0$ and 
$v_i = 1$ w.l.o.g.
Let $h_i(x) = \Phi^{-1}\bracket{F_i(x)}$,
where $(\cdot)^{-1}$ denotes the inverse of the function.
For the conciseness of the notation, we define 
$\hsamp{s}{i}$, 
$\tsamp{s}{i}$, and $\samp{s}{i}$ as
\[
\hsamp{s}{i} &=\wh{h}_i(\wh{X}^{(s)}_i);  \\
\tsamp{s}{i} &= \wh{h}_i(X^{(s)}_i); \\
\samp{s}{i} &= h_i(X^{(s)}_i).
\]

The $(i,j)$th entries of the corresponding covariance matrices $\wh{\bSigma}_n^{\textup{non-param}}$, $\wt{\bSigma}_n^{\textup{non-param}}$,
$\bSigma_n^{\textup{non-param}}$ are written as 
\[
\bracket{\wh{\bSigma}_n^{\textup{non-param}}}_{i, j} &= \frac{1}{n}\sum_{s=1}^n 
\hsamp{s}{i}\hsamp{s}{j} - \wh{\mu}_i \wh{\mu}_j; \\
\bracket{\wt{\bSigma}_n^{\textup{non-param}}}_{i, j} &= \frac{1}{n}\sum_{s=1}^n 
\tsamp{s}{i}\tsamp{s}{j}-\wt{\mu}_i \wt{\mu}_j; \\
\bracket{\bSigma_n^{\textup{non-param}}}_{i, j} &=  \frac{1}{n}\sum_{s=1}^n 
\samp{s}{i}\samp{s}{j}- {\mu}_i {\mu}_j.
\] 

Moreover, we define two regions $\calR_E$ and $\calR_M$ as 
{\small 
\begin{align}
\label{eq:rerm_def}
\calR_E & \defequal \bigg[-c_U\sqrt{\log (n\vcup (d-p))}, -c_L\sqrt{\log n\vcup \log(d-p)}
\bigg) \bigcup \bigg(c_L\sqrt{\log (n\vcup (d-p))}, c_U\sqrt{\log n\vcup \log(d-p)}
\bigg]; \notag \\
\calR_M & \defequal  \bigg[-c_L\sqrt{\log (n\vcup (d-p))}, c_L\sqrt{\log (n\vcup (d-p))}\bigg].
\end{align}
}
%

\subsection{Proof of Thm.~\ref{thm:graph_error_infnorm}}

\begin{proof}
We bound the deviation between ${\bSigma}_n^{\textup{non-param}}$
and $\wh{\bSigma}_n^{\textup{non-param}}$ as 
\[
\infnorm{\bSigma_n^{\textup{non-param}} - \wh{\bSigma}_n^{\textup{non-param}}} \leq 
\underbrace{\infnorm{ \bSigma_n^{\textup{non-param}}- \wt{\bSigma}_n^{\textup{non-param}}}}_{\defequal T_1} +
\underbrace{\infnorm{\wt{\bSigma}_n^{\textup{non-param}} - \wh{\bSigma}_n^{\textup{non-param}}}}_{\defequal T_2}.
\]	

\textbf{Step I}. For the first term $T_1$, we invoke the 
triangle inequality and have 
\[
\infnorm{\wt{\bSigma}^{\textup{non-param}} - \bSigma^{\natural}} \leq \max_{i,j} \frac{1}{n}\abs{\sum_{s=1}^n \tsamp{s}{i}\tsamp{s}{j} - \samp{s}{i}\samp{s}{j}}  + 
\infnorm{\mu_i\mu_j - \wt{\mu}_i \wt{\mu}_j}.  
\]
Following a similar strategy that is used in \cite{liu2009nonparanormal}, 
we focus on the 
first term as the second term is of higher order. 

We bound the value of $\max_{i,j}n^{-1}\abs{\sum_{s=1}^n \tsamp{s}{i}\tsamp{s}{j} - \samp{s}{i}\samp{s}{j}}$ as 
\[
& \Prob\bracket{\max_{i,j}n^{-1}\abs{\sum_{s=1}^n \tsamp{s}{i}\tsamp{s}{j} - \samp{s}{i}\samp{s}{j}}
\geq \vartheta} \\
\leq~&  p^2 \Expc\Bracket{\bracket{n^{-1}\abs{\sum_{s=1}^n \tsamp{s}{i}\tsamp{s}{j} - \samp{s}{i}\samp{s}{j}}
\geq \vartheta}\Ind\bracket{\samp{s}{i}\in \calR_E\bigcup \calR_M, \forall~1\leq s\leq n,1\leq  i \leq p}
} \\
+~& np \Expc\Ind\bracket{\samp{s}{i}\not\in  \calR_E\bigcap \calR_M,~\exists~1\leq s\leq n,1\leq  i \leq p}.
\]
Following a classical procedure as in 
\cite{boucheron2013concentration}, 
we can show the second probability $\Expc\Ind\bracket{\samp{s}{i}\not\in  \calR_E\bigcap \calR_M}$ is no greater than 
$e^{-c_0\bracket{n\vcap (d-p)}}$. 
For the conciseness of notation, we define the deviation $\delta^{(s)}_{i,j}$ as 
\[
\delta^{(s)}_{i,j} = \tsamp{s}{i}\tsamp{s}{j} - \samp{s}{i}
\samp{s}{j}. 
\]
Then the summary 
$n^{-1}\bracket{\sum_{s=1}^n \tsamp{s}{i}\tsamp{s}{j} - \samp{s}{i}\samp{s}{j}}\Ind\bracket{\samp{s}{i}\in \calR_E\bigcup \calR_M, \forall~1\leq s\leq n,1\leq  i \leq p}$ 
can be decomposed as 
\[
\frac{1}{n}\sum_{s=1}^n \delta^{(s)}_{i,j} = \
\underbrace{\frac{1}{n}\sum_{s=1}^n \delta^{(s)}_{i,j} 
\Ind[\calC^{(s)}_1(i, j)]}_{\defequal T_{1, 1}}+ \
\underbrace{\frac{1}{n}\sum_{s=1}^n \delta^{(s)}_{i,j}
\Ind[\calC^{(s)}_2(i, j)]}_{\defequal T_{1, 2}}+ \
\underbrace{\frac{2}{n}\sum_{s=1}^n \delta^{(s)}_{i,j}
\Ind[\calC^{(s)}_3(i, j)]}_{\defequal T_{1,3}}, 
\]
where the events $\calC^{(s)}_1(i, j)$, 
$\calC^{(s)}_2(i, j)$, and $\calC^{(s)}_3(i, j)$ are defined 
as 
\[
\calC^{(s)}_1(i, j) &\defequal \set{\samp{s}{i} \in \calR_E, \samp{s}{j} \in \calR_E}; \\
\calC^{(s)}_2(i, j) &\defequal \set{\samp{s}{i} \in \calR_M, \samp{s}{j} \in \calR_M}; \\
\calC^{(s)}_3(i, j) &\defequal \set{\samp{s}{i} \in \calR_E, \samp{s}{j} \in \calR_M}, 
\]
respectively,
where the definitions of $\calR_E$ and $\calR_M$ can be found
in \eqref{eq:rerm_def}. 

In the following, 
we will separately bound the three terms and show 
$n^{-1}\abs{\sum_{s=1}^n \delta^{(s)}_{i,j}} \lsim \vartheta \defequal \vartheta_1 \vcup \vartheta_2$, where 
the quantities $\vartheta_1$ and $\vartheta_2$ are
defined in \eqref{eq:vartheta1_def} and 
\eqref{eq:vartheta2_def}, respectively. 
The analysis of the first term $T_{1,1}$ and second term $T_{1,2}$
is deferred to Lemma~\ref{lemma:graph_error_c1} and 
Lemma~\ref{lemma:graph_error_c2}, respectively; 
while that of the third term $T_{1,3}$ is omitted due to their
similarities of Lemma~\ref{lemma:graph_error_c1} and 
Lemma~\ref{lemma:graph_error_c2}. 

\textbf{Step II}. 
The second term $T_2$ is upper-bounded 
in Lemma~\ref{lemma:emp_cov_perturb}. 
The analysis is in the 
same spirit as the above procedure but requires some 
modifications. 
\end{proof}

\begin{lemma}
\label{lemma:graph_error_c1}
We have 
\begin{align}
\label{eq:vartheta1_def}
\frac{1}{n}\abs{\sum_{s=1}^n \delta^{(s)}_{i,j} 
\Ind[\calC^{(s)}_1(i, j)]} \leq  
2 c_0\Bracket{n^{-c_1} \vcup (d-p)^{-c_1}} 
\bracket{\log n \vcup \log(d-p)}^{\frac{3}{2}} \defequal \vartheta_1, 
\end{align}
with probability exceeding $1 - c_2 n^{-c_3}\vcap (d-p)^{-c_4} - c_5 n^{-c_6}$,  
where the parameters $c_i$ are some
fixed constant, $0 \leq i \leq 6$. 
\end{lemma}

\begin{proof}

Invoking the union bound, we obtain 
\[
& \Prob\bracket{
\frac{1}{n}\abs{\sum_{s=1}^n \delta^{(s)}_{i,j} 
\Ind[\calC^{(s)}_1(i, j)]} \gsim 
2 c_0\Bracket{n^{-c_1} \vcup (d-p)^{-c_1}} 
\bracket{\log n \vcup \log(d-p)}^{\frac{3}{2}}} \\
\stackrel{\cirone}{\leq}~&
\underbrace{\Expc \Ind\bracket{\max_{s} \abs{\delta^{(s)}_{i,j}} \gsim 
\bracket{\log n \vcup \log(d-p)}}\Ind[\calC^{(s)}_1(i, j)] }_{P_1}
+ \underbrace{\Prob\bracket{
{\frac{1}{n}\sum_{s= 1}^n 
\Ind[\calC^{(s)}_1(i, j)]} \geq 
\vartheta }}_{P_2}, 
\]
where $\vartheta$ is defined as 
\begin{align}
\label{eq:graph_c1_vartheta_set}
\vartheta = 2 c_0\Bracket{n^{-c_1} \vcup (d-p)^{-c_1}} 
\sqrt{\log n \vcup \log(d-p)}, 
\end{align}
and in $\cirone$ we use the union bound.

With the relation 
\[
\abs{\delta^{(s)}_{i,j}}\leq 
\abs{\bracket{\tsamp{s}{i} - \samp{s}{i}}\bracket{\tsamp{s}{j} - \samp{s}{j}}}
+ \abs{\samp{s}{i}\bracket{\tsamp{s}{j} - \samp{s}{j}}}
+ \abs{\samp{s}{j}\bracket{\tsamp{s}{i} - \samp{s}{i}}}, 
\] 
we can upper bound the probability $P_1$ as 
\[
P_1 \leq np\Expc\Ind\bracket{\samp{s}{i}\leq c_U\sqrt{\log p \vcup \log (d-p)} \Ind[\calC^{(s)}_1(i, j)], 
~\tsamp{s}{i} \lsim c_U\sqrt{\log n \vcup \log (d-p)}}
\stackrel{\cirtwo}{\lsim} n^{-c} \vcap (d-p)^{-c}, 
\]
where $\cirtwo$ is due to  Lemma~\ref{lemma:graph_log_bound}. 

While for probability $P_2$, we have 
%
\[
& \Prob\bracket{
{\sum_{s= 1}^n 
\Ind[\calC^{(s)}_1(i, j)]} \geq 
n \vartheta } \leq 
\Prob\bracket{
\frac{1}{n}\sum_{s= 1}^n 
\Ind[\samp{s}{i} \in \calR_E]\geq \vartheta }  \\
=~& \Prob\bracket{\frac{1}{n}\sum_{s= 1}^n 
\Bracket{\Ind[\samp{s}{i} \in \calR_E] - \Expc\Ind\bracket{\samp{s}{i} \in \calR_E}} \geq \
\vartheta - \frac{1}{n}\Bracket{\sum_{s=1}^n \Expc\Ind\bracket{\samp{s}{i}\in \calR_E} }} \\
\stackrel{\cirthree}{\leq}~& 
\exp\bracket{-\frac{n}{2}\Bracket{\vartheta - \Prob\bracket{\samp{s}{i} \in \calR_E}}^2}, 
\]
where in $\cirthree$ we use the Hoeffding's inequality (cf. Thm.~$2.8$ 
in \cite{boucheron2013concentration}). 
Notice that the  probability 
$\Prob\bracket{\samp{s}{i}\in \calR_E}$ can be 
bounded as  
\[
\Prob\bracket{\samp{s}{i}\in \calR_E} =~&
\frac{2}{\sqrt{2\pi}}\int_{c_L\sqrt{\log n \vcup \log(d-p)}}^{c_U\sqrt{\log n \vcup \log(d-p)}}
e^{-t^2/2}dt \lsim \sqrt{\frac{2}{\pi}}e^{-\frac{c_L}{2} \bracket{\log n \vcup \log(d-p)}}
\sqrt{\log n \vcup \log(d-p)} \\
\leq~&  c_0\Bracket{n^{-c_1} \vcup (d-p)^{-c_1}} 
\sqrt{\log n \vcup \log(d-p)}, 
\]
where $0 < c_1 < 1/2$ is some fixed positive constant. 
Recalling the value of $\vartheta$ in \eqref{eq:graph_c1_vartheta_set}, 
we have 
\[
\Prob\bracket{
{\sum_{s= 1}^n 
\Ind[\calC^{(s)}_1(i, j)]} \geq 
n \vartheta }  \leq 
\exp\bracket{-c_0 n^{1-2c_1}\log n} \ll  n^{-c}, 
\]
and complete the proof. 

\end{proof}

\begin{lemma}
\label{lemma:graph_error_c2}
We have 
\begin{align}
\label{eq:vartheta2_def}
\frac{1}{n}\abs{\sum_{s=1}^n {\delta^{(s)}_{i, j} 
\Ind\bracket{C^{(s)}_2(i, j)}
}}\lsim \frac{\sqrt{\log n}}{n^{1/4}}\vcup 
\frac{\sqrt{\log(d-p)}}{(d-p)^{\beta/4}} \defequal \vartheta_2, 
\end{align}
with probability exceeding the 
probability $1-p^2\exp\bracket{-\frac{c_0\sqrt{n}}{\log^2 n} - \frac{c_1n\log^4(np)}{\log(d-p)(d-p)^{\beta/2}}}-p^2n^{-c_3}-2p^3e^{-c_4 p} - 4n^{-c_5}p^{-c_6}$,
where $\beta$ is defined as  $\frac{1}{2}\vcap \frac{a}{4} \vcap \frac{2\alpha + 1}{4}$.
\end{lemma}
\begin{proof}

For each term $\delta^{(s)}_{i,j}$, we can decompose it as 
\[
\delta^{(s)}_{i,j}=
\bracket{\tsamp{s}{i} - \samp{s}{i}}\bracket{\tsamp{s}{j} - \samp{s}{j}} 
+ \samp{s}{i}\bracket{\tsamp{s}{j} - \samp{s}{j}}
+ \samp{s}{j}\bracket{\tsamp{s}{i} - \samp{s}{i}}.
\] 
Then we can decompose the summary as 
\[
\frac{1}{n}\abs{\sum_{s=1}^n {\delta^{(s)}_{i, j} 
\Ind\bracket{C^{(s)}_2(i, j)}
}}\
\leq~& \frac{1}{n}\sum_{s=1}^n 
\abs{\bracket{\hsamp{s}{i} - \tsamp{s}{i}}\bracket{\hsamp{s}{j} - \tsamp{s}{j}}\Ind\bracket{C^{(s)}_2(i, j)} } \\
+~& \frac{1}{n}\sum_{s=1}^n
\abs{\tsamp{s}{i}\bracket{\hsamp{s}{j} - \tsamp{s}{j}}\Ind\bracket{C^{(s)}_2(i, j)}} \\
+~& \frac{1}{n}\sum_{s=1}^n
\abs{\tsamp{s}{j}\bracket{\hsamp{s}{i} - \tsamp{s}{i}}\Ind\bracket{C^{(s)}_2(i, j)}}.
\]
Our next goal is to investigate the behavior of 
$\sup_{i}\abs{\tsamp{s}{i} - \hsamp{s}{i}}\Ind\Bracket{C_2^{(s)}(i, j)}$.
Define the an event $\calE_{F}(\cdot)$ as 
\[
\calE_{F}(i) = \set{
\delta_{n, d, p} \leq \wh{F}_i \leq 1 - \delta_{n, d, p}} ,~~1\leq i \leq n. 
\]
We have 
\[
&\Prob\bracket{\frac{1}{n}\abs{\sum_{s=1}^n {\delta^{(s)}_{i, j} 
\Ind\bracket{C^{(s)}_2(i, j)}
}}\geq \vartheta_2} \leq 
\Prob\bracket{
\sup_{i}\abs{\tsamp{s}{i} - \hsamp{s}{i}}\Ind\Bracket{C_2^{(s)}(i, j)} \geq 
\vartheta_2} \\
\leq~& 
p^2\underbrace{\Expc \Ind\bracket{\Ind\bracket{C_2^{(s)}(i, j)}\br{\calE}_F(i)}}_{T_1} + 
p^2\underbrace{\Expc\Ind\bracket{\bracket{\abs{\tsamp{s}{i} - \hsamp{s}{i}}\Ind\bracket{C_2^{(s)}(i, j)} \geq 
\vartheta_2 }\bigcup \calE_F(i)}}_{T_2}. 
\]

Easily we can verify that $\delta_{n, d, p}$
satisfy the relation
\[
2\delta_{n, d, p} \leq 1 - \Phi\bracket{c_L\sqrt{\log n \vcup \log(d-p)}}
- \sqrt{\varepsilon_x}, 
\]
where $\varepsilon_x$ is defined in \eqref{eq:varepsilonx_def}. 
Invoking Lemma~\ref{lemma:graph_truncate_prob}, we can 
bound term $T_1$ as 
\[
T_1 \leq~& 2 \exp\bracket{-2n\bracket{1 - \delta_{n, d, p} - \Phi\bracket{c_L\sqrt{\log n \vcup \log(d-p)}}-\sqrt{\varepsilon_x}}^2}
\leq 2\exp(-2n\delta^2_{n, d,p}) \\
\lsim ~&
2\exp\bracket{-\frac{c_0\sqrt{n}}{\log^2 n} - \frac{c_1n\log^4(np)}{\log(d-p)(d-p)^{\beta/2}}}.  
\]
Conditional on event $\calE_F(i)$, we study the term $T_1$ by investigating
the difference
$\tsamp{s}{i} - \hsamp{s}{i}$. 
We assume $\wh{F}_i^{\textup{tr}}({X}^{(s)}_i) \geq F_i({X}^{(s)}_i)$ w.l.o.g.
According to the mean value theorem, we have  
\[
\tsamp{s}{i} - \samp{s}{i} = 
\Phi^{-1}\bracket{\wh{F}_i^{\textup{tr}}({X}^{(s)}_i)} -
\Phi^{-1}\bracket{F_i({X}^{(s)}_i)} = \
\bracket{\Phi^{-1}}^{'}(\xi)\abs{\wh{F}_i^{\textup{tr}}({X}^{(s)}_i) - F_i({X}^{(s)}_i)},
\]
where $\xi$ is some point such 
that $F_i({X}^{(s)}_i) \leq \xi \leq \wh{F}_i^{\textup{tr}}({X}^{(s)}_i)$.
Due to the fact that 
$X^{(s)}_i \in \calR_M$ and conditional on event $\calE_F(i)$, we 
conclude $\delta_n \leq \xi \leq 1 -\delta_n$
and hence
\[
\abs{\bracket{\Phi^{-1}(\xi)}^{'}} 
\leq~& \abs{\bracket{\Phi^{-1}(1 - \delta_{n, d, p})}^{'}}
\vcup \abs{\bracket{\Phi^{-1}(\delta_{n, d, p})}^{'}} = 
\abs{\bracket{\Phi^{-1}(1-\delta_{n, d, p})}^{'}}\\
= ~&\frac{1}{\phi\bracket{\Phi^{-1}(1-\delta_{n, d, p})}}\leq 
\exp\Bracket{\frac{1}{2}\bracket{\Phi^{-1}(1-\delta_{n, d, p})}^2}
\stackrel{\cirone}{\leq} \frac{1}{\delta_{n, d, p}}, 
\]
where  $\cirone$ is due to Lemma~\ref{lemma:gauss_cdf_gradient} 
and the fact that $\delta_{n,d,p} \rightarrow 0$. 
Set $\vartheta_2$ 
such that 
\[
\vartheta_2 \delta_{n, d, p} \geq 
(\log n)\varepsilon_x + \frac{c_0}{\sqrt{n}}
+ c_1 \sqrt{\varepsilon_x}, 
\]
where $\varepsilon_x$ is defined in \eqref{eq:varepsilonx_def}. 
We invoke Thm.~\ref{thm:cdf_tail_prob} and conclude 
\[
T_1 \leq\Prob\bracket{
\abs{\wh{F}_i^{\textup{tr}}({X}^{(s)}_i) - F_i({X}^{(s)}_i)}
\geq {\delta_{n, d, p}\vartheta_2}} 
\leq n^{-c_3}+2pe^{-c_4 p} + 4n^{-c_5}p^{-c_6}.  
\]
Recalling the definition of $\delta_{n, d, p}$
in \eqref{eq:deltandp_def}, we conclude 
$\vartheta_2$ to be approximately 
\[
\vartheta_2 \asymp \frac{\sqrt{\log n}}{n^{1/4}}\vcup 
\frac{\sqrt{\log(d-p)}}{(d-p)^{\beta/4}}, 
\]
where $\beta$ is defined 
as $\frac{1}{2}\vcap \frac{a}{4} \vcap \frac{2\alpha + 1}{4}$.

\end{proof}

\begin{lemma}
\label{lemma:emp_cov_perturb}
Under the Assumption~\ref{assump:lipschitz},
we have 
\begin{align}
\label{eq:vartheta3_def}
\infnorm{\wh{\bSigma}_n^{\textup{non-param}} - \wt{\bSigma}_n^{\textup{non-param}}} \lsim \sqrt{\log n\vcup \log(d-p)} \bracket{\frac{\sqrt{\log n}}{n^{1/4}}\vcup 
\frac{\sqrt{\log(d-p)}}{(d-p)^{\beta/4}}} + 
\frac{L\sigma\bracket{\log n \vcup \log(d-p)}}{
c_0 \log^2(np) (d-p)^{1/4}} \defequal \vartheta_3, 
\end{align}
with probability exceeding 
$1- e^{-c_0 (n\vcap (d-p))} - e^{-c_1n}$. 
\end{lemma}

\begin{proof}
The proof conditions on the event 
\[
\abs{\tsamp{s}{i}}\lsim c_U\sqrt{\log n \vcup \log (d-p)}, 
\]
holds for all $s$ and $i$ with probability exceeding 
$1-e^{-c_0 \bracket{n \vcap (d-p)}}$.

Following the same argument as in Thm.~\ref{thm:graph_error_infnorm}, our analysis 
focus on the error 
$n^{-1}\sum_{s=1}^n \bracket{\hsamp{s}{i}\hsamp{s}{j}-\samp{s}{i}\samp{s}{j}}$. 
Adopting the decomposition such that 
\[
\hsamp{s}{i}\hsamp{s}{j} - \tsamp{s}{i}\tsamp{s}{j}=
\tsamp{s}{i}\bracket{\hsamp{s}{j} - \tsamp{s}{j}}
+ \tsamp{s}{j}\bracket{\hsamp{s}{i} - \tsamp{s}{i}}
+
\bracket{\hsamp{s}{i} - \tsamp{s}{i}}\bracket{\hsamp{s}{j} - \tsamp{s}{j}}, 
\] 
we have 
\[
\frac{1}{n}\sum_{s=1}^n \bracket{\hsamp{s}{i}\hsamp{s}{j}-\samp{s}{i}\samp{s}{j}} =~& \
\underbrace{\frac{1}{n}\sum_{s=1}^n 
\tsamp{s}{i}\bracket{\hsamp{s}{j} - \tsamp{s}{j}}
}_{T_1} +\underbrace{\frac{1}{n}
\sum_{s=1}^n \tsamp{s}{j}\bracket{\hsamp{s}{i} - \tsamp{s}{i}}}_{T_2} \\
+~&\underbrace{\frac{1}{n}
\sum_{s=1}^n  \bracket{\hsamp{s}{i} - \tsamp{s}{i}}\bracket{\hsamp{s}{j} - \tsamp{s}{j}}}_{T_3}.
\]
We only need to analyze the behavior of 
the terms $T_1$ and $T_2$, since the term $T_3$ is of higher order. 
Conditional on the event discussed in Thm.~\ref{thm:cdf_tail_prob}, 
we have 
\[
\abs{\hsamp{s}{i} - \tsamp{s}{i}}
= ~&\abs{\bracket{\Phi^{-1}(\xi)}^{'}} \abs{\hftr{s}{i} - \ftr{s}{i}}\leq 
\frac{1}{\delta_{n, d, p}} \abs{\hftr{s}{i} - \ftr{s}{i}}\\
\leq~& \frac{1}{\delta_{n, d, p}}
\Bracket{\abs{\hftr{s}{i} - F(\wh{X}^{(s)}_i)} 
+ \abs{F(\wh{X}^{(s)}_i) - F(X^{(s)}_i)}
+ \abs{F(X^{(s)}_i) - \wh{F}^{\textup{tr}}(X^{(s)}_i)}} \\
\stackrel{\cirone}{\leq}~& \frac{2}{\delta_{n, d, p}}
\Bracket{(\log n)\varepsilon_x + \frac{c_0}{\sqrt{n}} + c_1\sqrt{\varepsilon_x}} + \frac{L\abs{\wh{w}^{(s)}_i}}{\delta_{n, d, p}},
\] 
where $\cirone$ is because of the Lipschitz property in Assumption \ref{assump:lipschitz}.
Then we obtain 
\[
T_1\lsim~& \underbrace{\frac{2\sqrt{\log n \vcup \log(d-p)}
}{\delta_{n, d, p}}
\Bracket{(\log n)\varepsilon_x + \frac{c_0}{\sqrt{n}} + c_1\sqrt{\varepsilon_x}}}_{T_{1,1}} + 
\underbrace{\frac{L\sqrt{\log n \vcup \log(d-p)}}{n\delta_{n, d, p}}
\sum_{s=1}^n |\wh{w}^{(s)}_i|}_{T_{1, 2}} \\
\stackrel{\cirtwo}{\lsim}~& 
2\sqrt{\log n\vcup \log(d-p)} \bracket{\frac{\sqrt{\log n}}{n^{1/4}}\vcup 
\frac{\sqrt{\log(d-p)}}{(d-p)^{\beta/4}}}
+ \frac{L\sqrt{\log n \vcup \log(d-p)}}{\delta_{n, d, p}}
\sqrt{\frac{\sum_{s=1}^n |\wh{w}^{(s)}_i|^2}{n}}, 
\]
where in $\cirtwo$ we plug in the definition of 
$\varepsilon_x$ in \eqref{eq:varepsilonx_def}, that is, 
\[
T_{1, 1}\asymp
2\sqrt{\log n\vcup \log(d-p)} \bracket{\frac{\sqrt{\log n}}{n^{1/4}}\vcup 
\frac{\sqrt{\log(d-p)}}{(d-p)^{\beta/4}}}.
\]
Since the $\wh{w}^{(s)}_i$ is approximately Gaussian 
distribution with mean zero and $\sigma^2/(d-p)$ variance, 
we have $(d-p)/\sigma^2 \sum_{s=1}^n |\wh{w}^{(s)}_i|^2$ 
be $\xi^2$-RV with freedom $n$, which means 
$\sum_{s=1}^n |\wh{w}^{(s)}_i|^2 \leq 2n\sigma^2/(d-p)$ holds
with probability at least $1-e^{-cn}$. 
To sum up, we obtain 
\[
T_1 \lsim~& \sqrt{\log n\vcup \log(d-p)} \bracket{\frac{\sqrt{\log n}}{n^{1/4}}\vcup 
\frac{\sqrt{\log(d-p)}}{(d-p)^{\beta/4}}} + \
\frac{L\sigma \sqrt{\log n \vcup \log(d-p)}}{\delta_{n, d, p}\sqrt{d-p}} \\
\lsim ~& \sqrt{\log n\vcup \log(d-p)} \bracket{\frac{\sqrt{\log n}}{n^{1/4}}\vcup 
\frac{\sqrt{\log(d-p)}}{(d-p)^{\beta/4}}} + 
\frac{L\sigma\bracket{\log n \vcup \log(d-p)}}{
c_0 \log^2(np) (d-p)^{1/4}}, 
\]
and complete the proof.

\end{proof}

\subsection{Supporting Lemmas}
\begin{lemma}
\label{lemma:graph_log_bound}
For all possible $i$ ($1\leq i \leq n$), we conclude
\[
\sup_{t\in \calR_E} \abs{
\Phi^{-1}\bracket{\wh{F}^{\textup{tr}}_i(t)} - \
\Phi^{-1}\bracket{F_i(t)}
} \leq 2c_U \sqrt{\log n \vcup \log(d-p)}, 
\]  
where $\calR_E$ is defined in \eqref{eq:rerm_def}.
\end{lemma}

\begin{proof}  
  
We conclude that 
\[
 \abs{
\Phi^{-1}\bracket{\wh{F}^{\textup{tr}}_i(t)} - \
\Phi^{-1}\bracket{F_i(t)}
} 
\leq 
\abs{\Phi^{-1}\bracket{F_i(t)}} +
\abs{\Phi^{-1}\bracket{\wh{F}^{\textup{tr}}_i(t)} }.
\]  
For the first term, we have 
\[
\abs{\Phi^{-1}\bracket{F_i(t)}} \leq c_U\sqrt{\log n \vcup \log(d-p)}, 
\]
according to the definitions of $\calR_E$. 
Meanwhile for the second term 
$\abs{\Phi^{-1}\bracket{\wh{F}^{\textup{tr}}_i(t)} }$, 
we have 
\[
\Phi^{-1}\bracket{\wh{F}^{\textup{tr}}_i(t)} 
\stackrel{\cirone}{\leq} 
\Phi^{-1}\bracket{1-\delta_{n, d, p}} 
\stackrel{\cirtwo}{\leq} \sqrt{2 \log \frac{1}{1- \delta_{n, d, p}}} 
\stackrel{\cirthree}{\leq} c_U \sqrt{\log n \vcup \log (d-p)}, 
\]
where in $\cirone$ we exploit the fact 
$\wh{F}^{\textup{tr}}_i(t) \leq 1 -\delta_{n, d, p}$, 
in $\cirtwo$ we invoke Lemma~$11$ in
 \cite{liu2009nonparanormal} (cf. 
 Lemma~\ref{lemma:gauss_cdf_gradient}), 
and in $\cirthree$ we use the fact 
that $\delta_{n, d, p}\geq \frac{1}{n} \vcup \frac{1}{d-p}$.

\end{proof}

\begin{lemma}
\label{lemma:graph_truncate_prob}
Provided that $\delta_{n, d, p} \leq 1 - \Phi\bracket{c_L\sqrt{\log n \vcup \log(d-p)}}-\sqrt{\varepsilon_x}$, 
we can bound the probability 
\[
\Prob\bracket{
\wh{F}_i(g_i(c_L\sqrt{\log n \vcup \log(d-p)})) \geq 1 - \delta_{n, d, p}} &\leq  
\exp\bracket{-2n\bracket{1 - \delta_{n, d, p} - \Phi\bracket{c_L\sqrt{\log n \vcup \log(d-p)}}-\sqrt{\varepsilon_x}}^2}; \\
\Prob\bracket{
\wh{F}_i(g_i(-c_L\sqrt{\log n \vcup \log(d-p)})) \leq \delta_{n, d, p}} &\leq  
\exp\bracket{-2n\bracket{1 - \delta_{n, d, p} - \Phi\bracket{c_L\sqrt{\log n \vcup \log(d-p)}}-\sqrt{\varepsilon_x}}^2}, 
\]
where $\varepsilon_x$ is defined in \eqref{eq:varepsilonx_def}.
\end{lemma}

\begin{proof}
We have 

\par  \vspace{-4mm}
{\small 
\[
&\Prob\bracket{
\wh{F}_i(g_i(c_L\sqrt{\log n \vcup \log(d-p)})) \geq 1 - \delta_{n, d, p}} \\
=~& 
\Prob\bracket{
\wh{F}_i(g_i(c_L\sqrt{\log n \vcup \log(d-p)})) - \Expc \wh{F}_i(g_i(c_L\sqrt{\log n \vcup \log(d-p)})) \geq 1 - \delta_{n, d, p} - \Expc \wh{F}_i(g_i(c_L\sqrt{\log n \vcup \log(d-p)}))} \\
\leq~& \exp\Bracket{-2n\bracket{1-\delta_{n, d, p} -\Expc \wh{F}_i(g_i(c_L\sqrt{\log n \vcup \log(d-p)}))}^2}.
\]
\vspace{-4mm}
}

\par \noindent  

The term is bounded as 

\par  \vspace{-2mm}
{\small 
\[
& 1-\delta_{n, d, p} - \Expc \wh{F}_i\bracket{g_i(c_L\sqrt{\log n \vcup \log(d-p)})}\\
\geq~& \abs{1 - \delta_{n, d, p} - F_i\bracket{g_i(c_L\sqrt{\log n \vcup \log(d-p)})}}
-
\abs{F_i\bracket{g_i(c_L\sqrt{\log n \vcup \log(d-p)})} - \Expc \wh{F}_i(g_i(c_L\sqrt{\log n \vcup \log(d-p)}))} \\
=~& \abs{1 - \delta_{n, d, p} - \Phi\bracket{c_L\sqrt{\log n \vcup \log(d-p)}}} - 
\sqrt{\varepsilon_x}. 
\]\vspace{-2mm}
}

\par \noindent 
Similarly, we can prove 
\[
\Prob\bracket{
\wh{F}_i(g_i(-c_L\sqrt{\log n \vcup \log(d-p)})) \leq \delta_{n, d, p}} 
\leq
\exp\bracket{-2n\bracket{1 - \delta_{n, d, p} - \Phi\bracket{c_L\sqrt{\log n \vcup \log(d-p)}}-\sqrt{\varepsilon_x}}^2}.
\]
	
\end{proof}

\section{Additional Simulations}
This section presents additional simulation results with 
synthetic data. Here we adopt different underlying structures of 
$\bTheta^{\natural}$ and study the performance of our algorithm.

\subsection{Synthetic data with parametric method}
We adopt the same graph construction in the main context. 

\begin{figure}[!h]

\centering
\mbox{
\includegraphics[width = 2.5in]{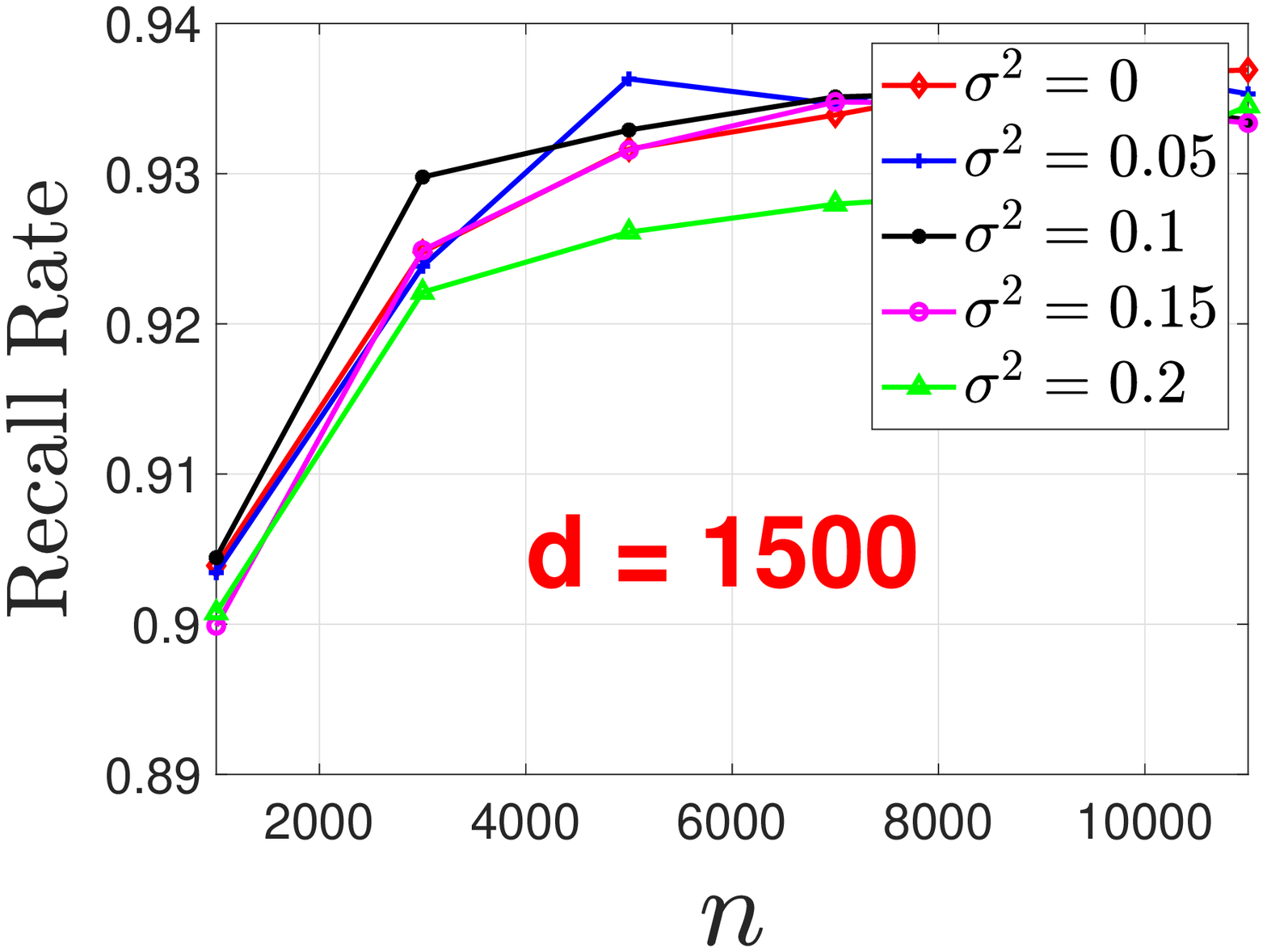}
\includegraphics[width = 2.5in]{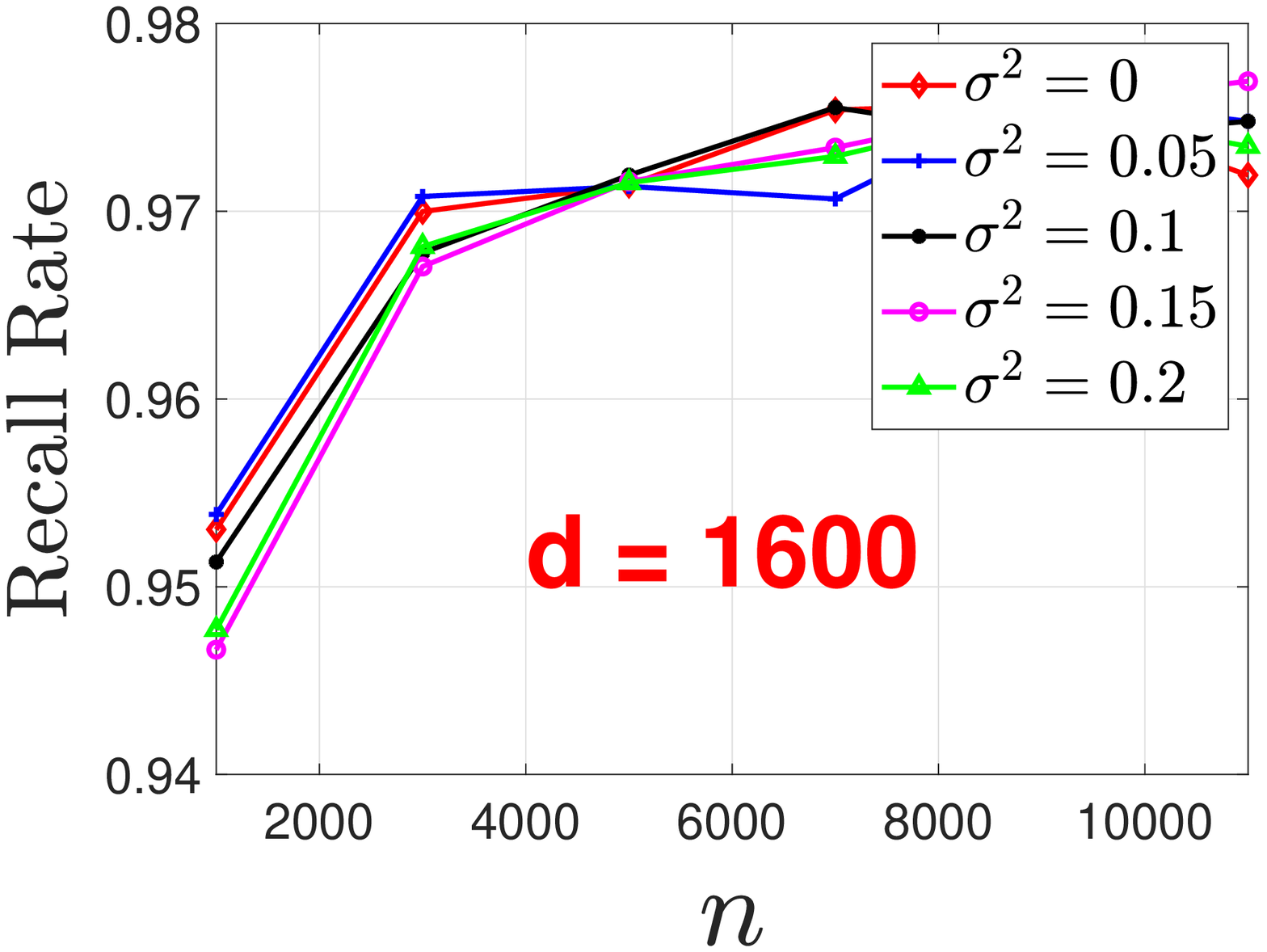}
}
\mbox{
\includegraphics[width = 2.5in]{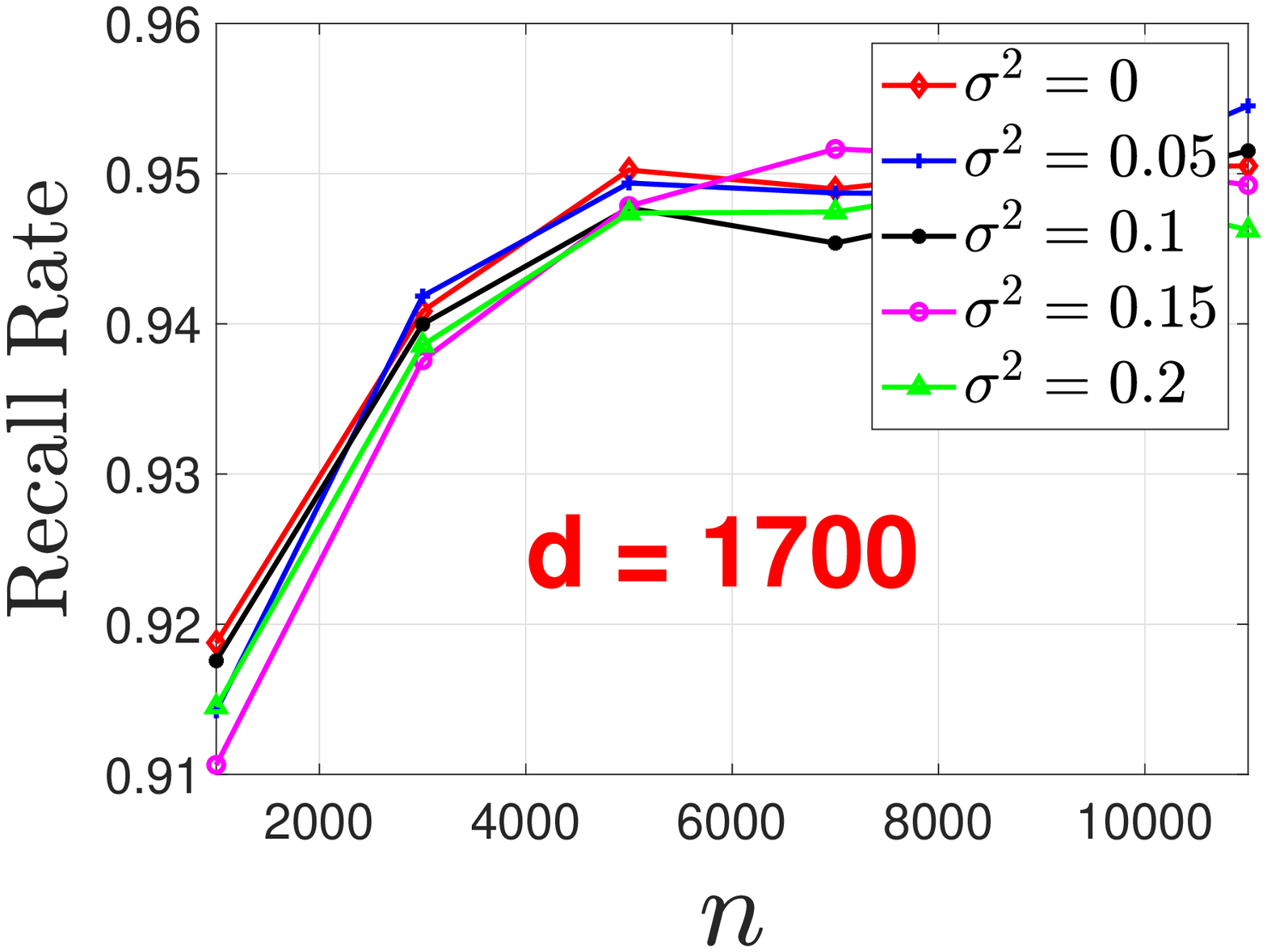}
\includegraphics[width = 2.5in]{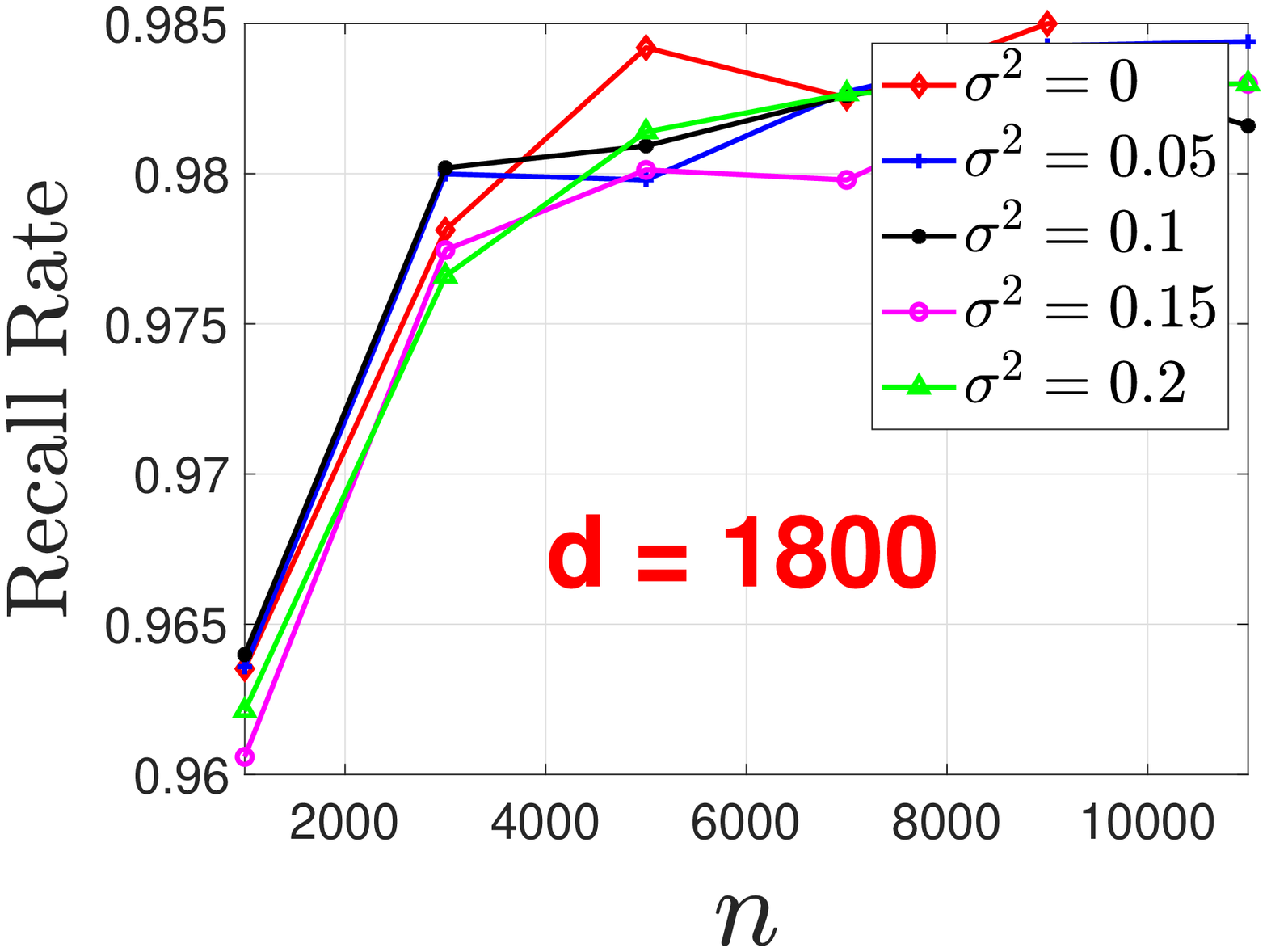}
}
\caption{We study the  
impact of sample size $n$ 
on the recall rate. 
The signal dimension $p$ is fixed as $2000$.}
\label{fig:synthetic_band_recall}
\end{figure}

\begin{figure}[!h]

\centering
\mbox{
\includegraphics[width = 2.5in]{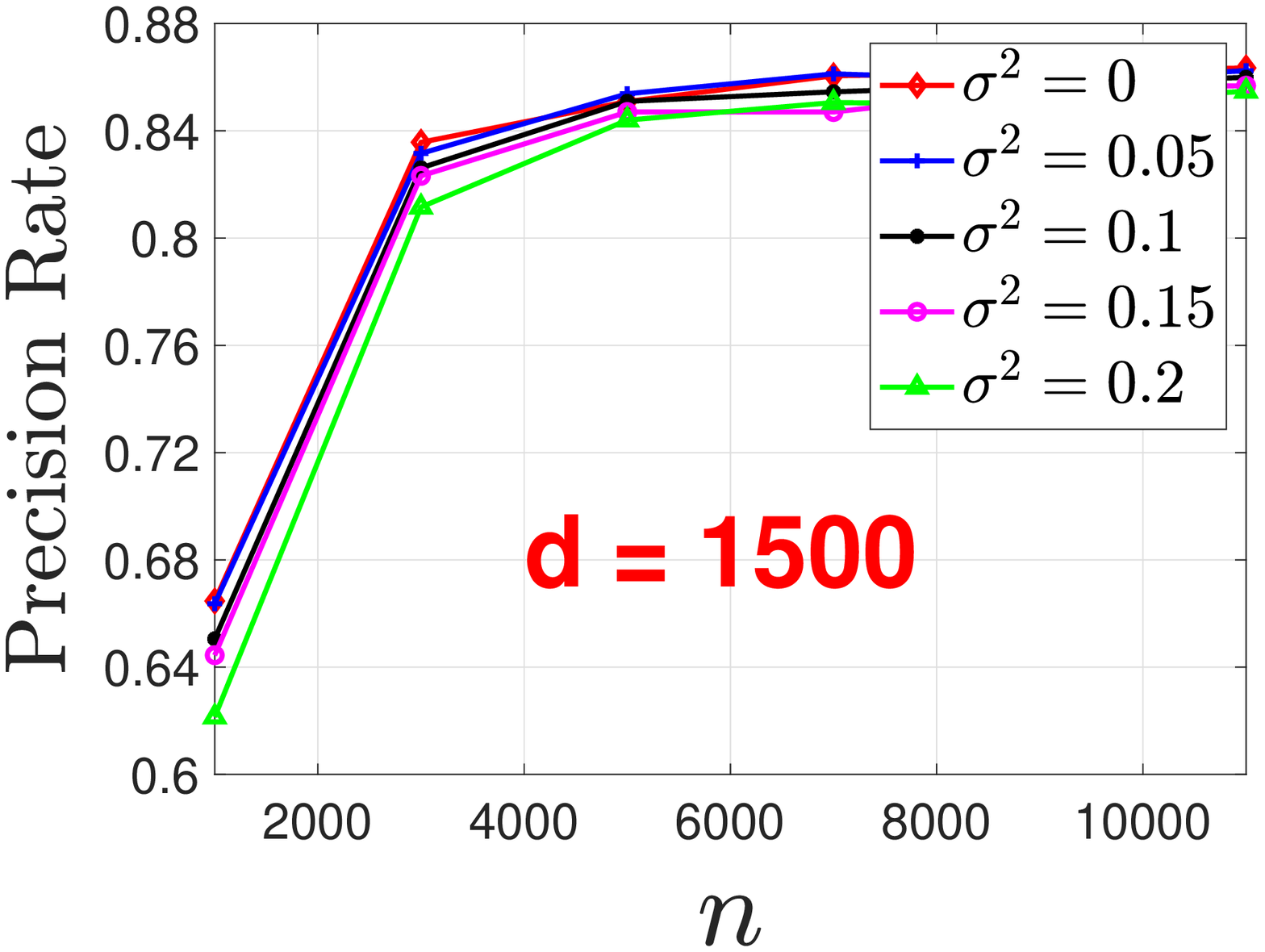}
\includegraphics[width = 2.5in]{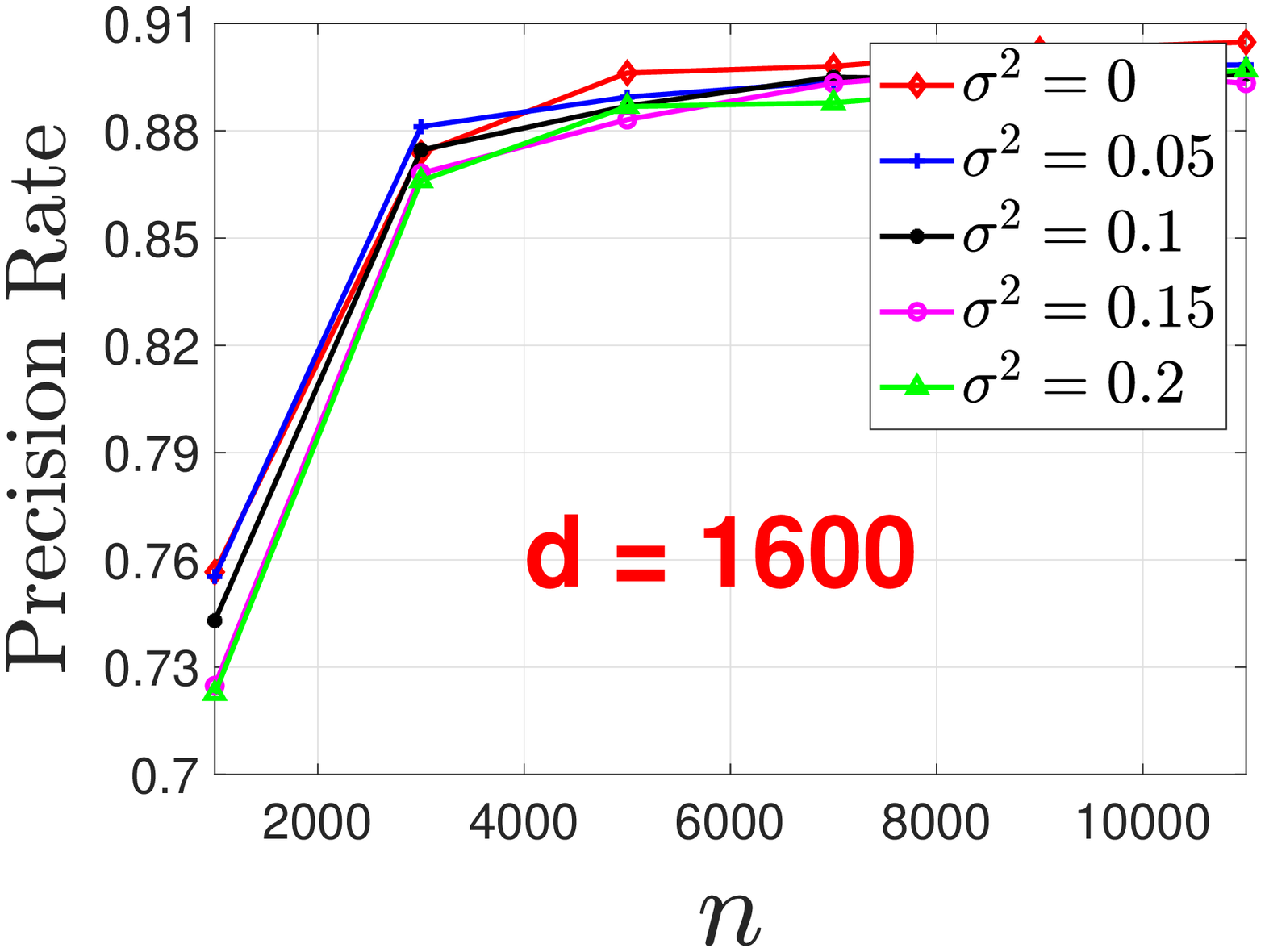}
}
\mbox{
\includegraphics[width = 2.5in]{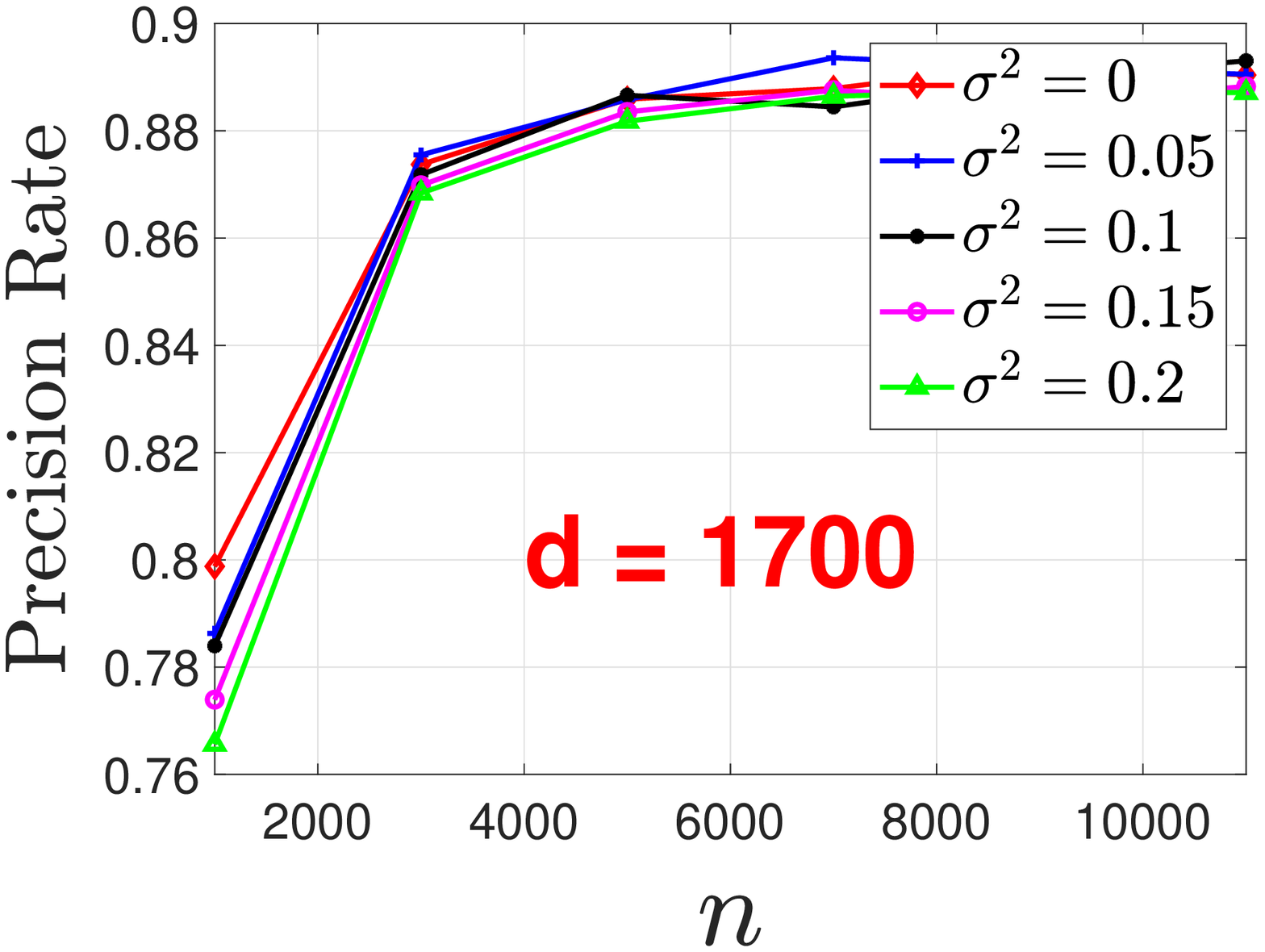}
\includegraphics[width = 2.5in]{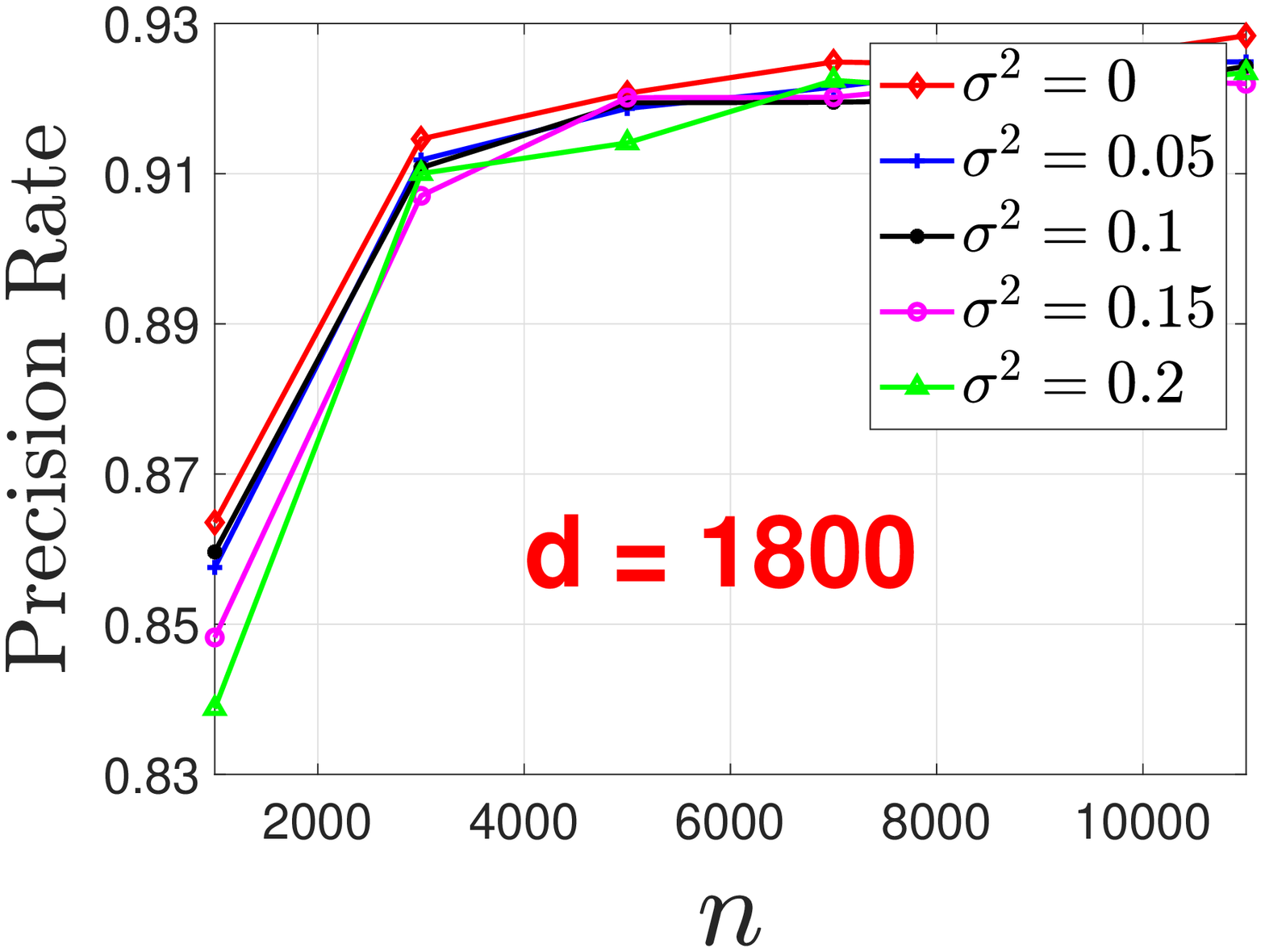}
}
\caption{We study the  
impact of sample size $n$ 
on the precision rate. 
The signal dimension $p$ is fixed as $2000$.}
\label{fig:synthetic_band_precision}
\end{figure}

\subsection{Synthetic data with non-parametric method}
This subsection uses the same setting as that in the main
context but evaluates the 
performance of our estimator via the recall rate and the precision rate
of the 
edge selection, which is shown in 
Tab.~\ref{tab:nonparam_synthetic_recall_prec}.

\begin{table}[ht]
\centering
{\footnotesize
\begin{tabular}{@{}ccccccccc@{}}\toprule
& \multicolumn{2}{c}{\makecell{\textbf{Uniform}}} & \phantom{ab}& \multicolumn{2}{c}{\makecell{\textbf{Exponential}}} &
\phantom{ab} & \multicolumn{2}{c}{\makecell{\textbf{Gauss Mixture}}}\\
\cmidrule{2-3} \cmidrule{5-6} \cmidrule{8-9}
$n$ & \makecell{Recall Rate} & \makecell{Precision Rate} && \makecell{Recall Rate} & \makecell{Precision Rate}  && \makecell{Recall Rate} & \makecell{Precision Rate} \\ \midrule
$d=200$\\  
$100$ & $0.9315$ & $0.9257$ && $0.9839$ & $0.9029$ && $0.9732$ & $0.9114$\\
$115$ & $0.9342$ & $0.9527$ && $0.9866$ & $0.9043$ && $0.9758$ & $0.9331$\\
$130$ & $0.9369$ & $0.9555$ && $0.9906$ & $0.9323$ && $0.9826$ & $0.9363$\\
$145$ & $0.9450$ & $0.9655$ && $0.9946$ & $0.9375$ && $0.9866$ & $0.9534$\\ 
$160$ & $0.9490$ & $0.9756$ && $0.9946$ & $0.9527$ && $0.9866$ & $0.9638$\\
$175$ & $0.9490$ & $0.9793$ && $0.9946$ & $0.9504$ && $0.9879$ & $0.9712$\\
\\ $d = 300$ \\ 
$100$ & $0.9584$ & $0.9347$ && $0.9852$ & $0.9074$ && $0.9852$ & $0.9176$\\
$115$ & $0.9651$ & $0.9400$ && $0.9799$ & $0.9070$ && $0.9906$ & $0.9191$\\
$130$ & $0.9812$ & $0.9545$ && $0.9919$ & $0.9276$ && $0.9933$ & $0.9290$\\
$145$ & $0.9812$ & $0.9549$ && $0.9933$ & $0.9333$ && $0.9933$ & $0.9298$\\
$160$ & $0.9852$ & $0.9597$ && $0.9973$ & $0.9411$ && $0.9946$ & $0.9376$\\
$175$ & $0.9839$ & $0.9656$ && $0.9973$ & $0.9458$ && $0.9960$ & $0.9513$\\
\\ $d = 500$ \\ 
$100$ & $0.9745$ & $0.9455$ && $0.9812$ & $0.9097$ && $0.9879$ & $0.9126$\\
$115$ & $0.9758$ & $0.9758$ && $0.9866$ & $0.9320$ && $0.9893$ & $0.9320$\\
$130$ & $0.9758$ & $0.9813$ && $0.9919$ & $0.9467$ && $0.9906$ & $0.9343$\\
$145$ & $0.9826$ & $0.9826$ && $0.9946$ & $0.9470$ && $0.9933$ & $0.9477$\\
$160$ & $0.9879$ & $0.9879$ && $0.9946$ & $0.9502$ && $0.9946$ & $0.9537$\\
$175$ & $0.9893$ & $0.9893$ && $0.9973$ & $0.9619$ && $0.9960$ & $0.9604$\\
\midrule 
\makecell{Baseline} & $0.3356$ & $0.1834$  && $0.3765$ & $0.0497$ && $0.8188$ & $0.8538$  \\
\bottomrule
\end{tabular}}
\caption{
Recall rate and precision rate. 
The signal length $p$ is  
fixed as $100$ and $\sigma^2$ is fixed as 
$1.0$. The sample number corresponding to the baseline 
is set as $175$, the maximum sample number.
}
\label{tab:nonparam_synthetic_recall_prec}
\end{table}

We construct the sparse precision matrix $\bTheta$ as 
\[
\Theta_{i,j} = \left\{
\begin{aligned}
&\rho_1,~&& \textup{if}~i = j; \\
& \rho_2,~&& \textup{if}~\abs{i-j} = 1; \\
&0.~&& \textup{otherwise}, 
\end{aligned}
\right.
\]
which is previously adopted in \cite{ravikumar2011high}.
The corresponding edge set of is denoted as 
$E$. 
We fix the signal length $p$ as $100$ and 
evaluate the performance with the following three types of 
marginal distribution for the random vector $\bX$: 
\begin{itemize}
\item 
uniform distribution within the 
region $[0, 1]$;
\item 
exponential distribution, i.e., 
$e^{-z}$ for $z\geq 0$;
\item 
 Gaussian mixture, 
i.e., $0.25\sum_{i=1}^4 \normdist(\mu_i, 10^{-2})$, where $\mu_i \in \set{\pm 0.25, \pm 0.5}$.
\end{itemize} 

We set $\rho_1$ and $\rho_2$ as $1$ and $0.4$, respectively. 
For the baseline in \ref{tab:nonparam_synthetic}, we assume the underlying distribution of 
$\bX$ to be jointly Gaussian. With direct observations, 
we learn the 
graphical structure with graphical lasso. The recall rate and precision 
rate is shown in \ref{tab:nonparam_synthetic}. 
For the uniform and exponential distribution, our algorithm 
has a significant improvement in terms of both the recall rate and precision rate. While for mixture Gaussian, the
improvement is modest. This phenomenon may justify the approximation of 
mixture Gaussian with Gaussian distribution, which is widely 
used in the field of coding theory, machine learning, etc. 

\begin{table}[ht]
\caption{
Recall rate and precision rate. 
The signal length $p$ is  
fixed as $100$ and $\sigma^2$ is fixed as 
$1.0$. The sample number corresponding to the baseline 
is set as $175$, the maximum sample number.
}
\vspace{2mm}
\label{tab:nonparam_synthetic}
\centering
{\footnotesize
\begin{tabular}{@{}ccccccccc@{}}\toprule
& \multicolumn{2}{c}{\makecell{\textbf{Uniform}}} & \phantom{ab}& \multicolumn{2}{c}{\makecell{\textbf{Exponential}}} &
\phantom{ab} & \multicolumn{2}{c}{\makecell{\textbf{Gauss Mixture}}}\\
\cmidrule{2-3} \cmidrule{5-6} \cmidrule{8-9}
$n$ & \makecell{Recall Rate} & \makecell{Precision Rate} && \makecell{Recall Rate} & \makecell{Precision Rate}  && \makecell{Recall Rate} & \makecell{Precision Rate} \\ \midrule
$d=200$\\  
$100$ & $0.9315$ & $0.9257$ && $0.9839$ & $0.9029$ && $0.9732$ & $0.9114$\\
$115$ & $0.9342$ & $0.9527$ && $0.9866$ & $0.9043$ && $0.9758$ & $0.9331$\\
$130$ & $0.9369$ & $0.9555$ && $0.9906$ & $0.9323$ && $0.9826$ & $0.9363$\\
$145$ & $0.9450$ & $0.9655$ && $0.9946$ & $0.9375$ && $0.9866$ & $0.9534$\\ 
$160$ & $0.9490$ & $0.9756$ && $0.9946$ & $0.9527$ && $0.9866$ & $0.9638$\\
$175$ & $0.9490$ & $0.9793$ && $0.9946$ & $0.9504$ && $0.9879$ & $0.9712$\\
\\ $d = 300$ \\ 
$100$ & $0.9584$ & $0.9347$ && $0.9852$ & $0.9074$ && $0.9852$ & $0.9176$\\
$115$ & $0.9651$ & $0.9400$ && $0.9799$ & $0.9070$ && $0.9906$ & $0.9191$\\
$130$ & $0.9812$ & $0.9545$ && $0.9919$ & $0.9276$ && $0.9933$ & $0.9290$\\
$145$ & $0.9812$ & $0.9549$ && $0.9933$ & $0.9333$ && $0.9933$ & $0.9298$\\
$160$ & $0.9852$ & $0.9597$ && $0.9973$ & $0.9411$ && $0.9946$ & $0.9376$\\
$175$ & $0.9839$ & $0.9656$ && $0.9973$ & $0.9458$ && $0.9960$ & $0.9513$\\
\\ $d = 500$ \\ 
$100$ & $0.9745$ & $0.9455$ && $0.9812$ & $0.9097$ && $0.9879$ & $0.9126$\\
$115$ & $0.9758$ & $0.9758$ && $0.9866$ & $0.9320$ && $0.9893$ & $0.9320$\\
$130$ & $0.9758$ & $0.9813$ && $0.9919$ & $0.9467$ && $0.9906$ & $0.9343$\\
$145$ & $0.9826$ & $0.9826$ && $0.9946$ & $0.9470$ && $0.9933$ & $0.9477$\\
$160$ & $0.9879$ & $0.9879$ && $0.9946$ & $0.9502$ && $0.9946$ & $0.9537$\\
$175$ & $0.9893$ & $0.9893$ && $0.9973$ & $0.9619$ && $0.9960$ & $0.9604$\\
\midrule 
\makecell{Baseline} & $0.3356$ & $0.1834$  && $0.3765$ & $0.0497$ && $0.8188$ & $0.8538$  \\
\bottomrule
\end{tabular}}
\end{table}


\subsection{Real-World Data}
We now consider the real-world databases, which 
consists of $5$ databases:
Carolina Breast Cancer (GSE$148426$)  with $2497$ samples
(patients) \cite{Bhattacharya2020.08.14.250902}, 
Lung Cancer (GSE$137140$) with $3924$ samples \cite{asakura2020mirna}, 
Ovarian Cancer (GSE$106817$) with $4046$ samples \cite{yokoi2018integrated},
Colorectal Cancer (GSE$115513$) with $1513$ samples  \cite{slattery2016microrna}, 
and Esophageal Squamous Cell Carcinoma (GSE$122497$)  with $5531$ samples \cite{sudo2019development}. 
Each database is divided into two categories, i.e., \textbf{Healthy group} and \textbf{Unhealthy group},
where the measurements 
are given as the concentration of miRNAs. 
The miRNAs are known to have dependency among each other,
i.e., a non-diagonal precision matrix, and hence 
there is an underlying graphical model describing these
dependency structure based on the associated precision 
matrix.

{ 
\begin{table*}[!h]
\centering
\caption{
Recall rate of edge selection on real-world databases, 
namely, GSE$148426$ \cite{Bhattacharya2020.08.14.250902}, 
GSE$137140$ \cite{asakura2020mirna}, 
GSE$106817$ \cite{yokoi2018integrated},
GSE$115513$ \cite{slattery2016microrna}, 
and GSE$122497$ \cite{sudo2019development}.
The precision matrix $\bTheta$ 
 learned by direct observations is assumed to be the ground-truth. 
}
{\footnotesize
\label{tab:real_data_recall}
\begin{tabular}{@{}cccccrlccccc@{}}\toprule
 & \multicolumn{5}{c}{Healthy group} & \phantom{abc}& \multicolumn{5}{c}{Unhealthy group} 
\\
\cmidrule{2-6} \cmidrule{8-12} 
\textbf{$d/p$} & \makecell{ GSE \\$148426$} & 
\makecell{GSE\\$137140$} & \makecell{GSE\\$106817$} & 
\makecell{GSE\\$115513$} & \makecell{GSE\\$122497$} &&\makecell{ GSE \\$148426$} & 
\makecell{GSE\\$137140$} & \makecell{GSE\\$106817$} & 
\makecell{GSE\\$115513$} & \makecell{GSE\\$122497$} \\ \midrule
$2$  & $0.9494$ & $0.8892$ & $0.9$ & $0.9856$ & $0.7659$
&&   $0.9424$ & $0.9692$ & $1$ & $0.9856$ & $0.8379$
\\
$5$ & $1$ & $0.9950$ & $1$ & $1$ & $0.9220$
&&  $1$ & $0.9692$ & $1$ & $1$ & $0.9632$
\\
$10$ & $1$ & $1$ & $1$ & $1$ & $0.9707$
&&   $1$ & $1$ & $1$ & $1$ & $0.9963$
\\
$12$ & $1$ & $1$ & $1$ & $1$ & $0.9805$
&&   $1$ & $1$ & $1$ & $1$ & $1$
\\
$15$ & $1$ & $1$ & $1$ & $1$ & $0.9902$
&&   $1$ & $1$ & $1$ & $1$ & $1$
\\ 
$20$ & $1$ & $1$ & $1$ & $1$ & $1$
&&   $1$ & $1$ & $1$ & $1$ & $1$
\\
\bottomrule
\end{tabular}
}
\end{table*}
}

{ 
\begin{table}[!h]
\centering
\caption{
Precision rate of edge selection on real-world databases, 
namely, GSE$148426$ \cite{Bhattacharya2020.08.14.250902}, 
GSE$137140$ \cite{asakura2020mirna}, 
GSE$106817$ \cite{yokoi2018integrated},
GSE$115513$ \cite{slattery2016microrna}, 
and GSE$122497$ \cite{sudo2019development}.
The precision matrix $\bTheta$ 
 learned by direct observations is assumed to be the ground-truth. 
}
{\footnotesize
\label{tab:real_data_prec}
\begin{tabular}{@{}cccccrlccccc@{}}\toprule
 & \multicolumn{5}{c}{Healthy group} & \phantom{abcde}& \multicolumn{5}{c}{Unhealthy group} 
\\
\cmidrule{2-6} \cmidrule{8-12} 
\textbf{$d/p$} & \makecell{ GSE \\$148426$} & 
\makecell{GSE\\$137140$} & \makecell{GSE\\$106817$} & 
\makecell{GSE\\$115513$} & \makecell{GSE\\$122497$} &&\makecell{ GSE \\$148426$} & 
\makecell{GSE\\$137140$} & \makecell{GSE\\$106817$} & 
\makecell{GSE\\$115513$} & \makecell{GSE\\$122497$} \\ \midrule
$2$  & $1$ & $1$ & $1$ & $0.9548$ & $1$
&&   $0.9704$ & $0.9692$ & $1$ & $0.9581$ & $1$
\\
$5$ & $0.9080$ & $0.9900$ & $1$ & $0.9176$ & $1$
&&  $0.9205$ & $0.9692$ & $1$ & $0.9207$ & $0.9924$
\\
$10$ & $0.9080$ & $0.9341$ & $0.9091$ & $0.9207$ & $0.9900$
&&   $0.9205$ & $0.9420$ & $1$ & $0.9268$ & $0.9854$
\\
$12$ & $0.9080$ & $0.9475$ & $0.9259$ & $0.9237$ & $0.9901$
&&   $0.9329$ & $0.9420$ & $1$ & $0.9268$ & $0.9645$
\\
$15$ & $0.9080$ & $0.9566$ & $0.9434$ & $0.9299$ & $0.9621$
&&   $0.9456$ & $0.9420$ & $1$ & $0.9299$ & $0.9645$
\\ 
$20$ & $0.9518$ & $0.9613$ & $0.9615$ & $0.9393$ & $0.9535$
&&   $0.9586$ & $0.9420$ & $1$ & $0.9457$ & $0.9784$
\\
\bottomrule
\end{tabular}
}
\end{table}
}

\vspace{-1mm}
\noindent
\textbf{Preprocess} 
For each database, there are multiple types of 
data, namely, healthy vs non-healthy, benign cancer vs 
non-benign cancer, etc. Based on the labels, 
we first divide the whole dataset into 
two types of data and
separately preprocess them.
For each type, we split the 
data into the training and testing sets
$\set{\calD^i_{\textup{train}}, \calD^i_{\textup{test}}}$, $i\in \set{1, 2}$. 
The sensing matrix $\bA$ is assumed to be 
$A_{ij}\iid \normdist(0, 1)$ and the measurement 
noise $\bw$ is set to be zero.  

\noindent
\textbf{Training}
Using the training set $\calD^i_{\textup{train}}$,
which are directly observed from the desired signal (miRNAs),  
we 
first select the penalty coefficients in graphical lasso 
via cross-validation. Then estimate 
the precision matrix via the graphical Lasso 
and use it as the baseline $\bTheta_i$, $i\in \set{1, 2}$
as they are obtained through direct observation. We 
denote
$\bTheta_1$ for $\calD^1$ 
while $\bTheta_2$ for $\calD^2$. 
Then we mask the training set with the synthetic
sensing matrix $\bA$ and create indirect observation 
of data. Last we estimate 
the precision matrix $\wh{\bTheta}_i$, from the 
indirect observations, i.e., $\bA\calD^i_{\textup{train}}$
using  M-gLasso, $i\in \set{1, 2}$.

\noindent
\textbf{Testing}
Notice that the ground-truth precision matrix $\bTheta^{\natural}_i$, $i\in \set{1, 2}$ 
cannot be obtained in the real-world
applications, even 
with the direct observations. Hence a direct comparison  
 between $\bTheta^{\natural}_i$ and $\hat{\bTheta}_i$ cannot 
 be performed. 
To evaluate the performance of the algorithm, 
we take an indirect way by using \emph{quadratic discriminant analysis}
\cite{hastie2009elements} to perform classification in the 
 testing set 
$\set{\calD^1_{\textup{test}} \bigcup \calD^2_{\textup{test}}}$,
with the estimated matrix $\wh{\bTheta}_i$, $i\in \set{1, 2}$. Then we compare with the 
classification accuracy when using $\bTheta_i$, $i\in \set{1, 2}$. 
The summary of the classification
rates are shown in 
Tab.~\ref{tab:real_data}. 

\noindent
\textbf{Discussion}
From the tables we conclude that our 
estimated precision matrix 
$\wh{\bTheta}_i$ achieves almost the 
same classification accuracy 
with the baseline 
$\bTheta_i$ when $d = p$, 
and is only a slightly worse when 
$d = 0.5 p$, i.e., the dimension of the projection
space under the indirect observation is half of the 
dimension of the signal space.
In some special cases, we even see some 
improvements with indirect observations. One possible 
reason  is that the features are mixed
by sensing matrix $\bA$, which lead to better quantities
for the classification.

 \section{Useful Facts about Probability Inequalities and Random Matrices}
For the self-containing of this paper, we list some useful facts
about probability inequalities and random matrices in this section. 

\begin{lemma}[\citep{wainwright_2019} (Example 2.11, P29)]
\label{lemma:chi_square_dev}
For a $\chi^2$-RV $Z$ with $\ell$ degrees of freedom, we have 
\begin{align*}
\Prob\bracket{|Z - \ell| \geq t} \leq 
2\exp\bracket{-\bracket{\frac{t^2}{8\ell} \vcap \frac{t}{8}}},
~~\forall~t\geq 0. 
\end{align*}
\end{lemma}

\begin{theorem}[Theorem $2.35$ in \cite{tulino2004random}]
\label{lemma:marcenko}
For a $d\times p$ matrix $\bA$ whose entries 
are independent zero-mean real RVs with variance 
$d^{-1}$ and fourth moment of order $O(d^{-2})$, 
we have the empirical distribution of the eigenvalues of
$\bA^{\rmt}\bA$ converge to the distribution with 
density 
\[
f_{\tau}(x)= \Bracket{0 \vcup \bracket{1-\tau^{-1}}}
\Ind(x) + \frac{\sqrt{\Bracket{\bracket{\bracket{1+\sqrt{\tau}}^2 - x} \vcup 0} \times \Bracket{\bracket{x - \bracket{1-\sqrt{\tau}}^2} \vcup 0}}}{2\pi \tau x},
\]
as 
$d$ and $p$ approaches to infinity with $p/d\rightarrow \tau$.
\end{theorem}

\begin{theorem}[Corol.~$13.2$ in \cite{boucheron2013concentration}]
\label{thm:dudley_integral}
Consider a totally bounded pseudo-metric space $\bracket{\calT, \textup{dist}(\cdot, \cdot)}$. 
Provided a collection of RVs $\set{Z_u}_{u\in \calT}$ satisfying 
\[
\Expc e^{\lambda\bracket{Z_{u} - Z_{v}}} \leq \frac{\nu \lambda^2 \dist^2(u, v)}{2}, ~~\forall~\lambda \geq 0, 
\]
we have
\[
\Expc \sup_{u} \bracket{ Z_{u} - Z_{u_0}} \leq 
12 \sqrt{\nu} \int_{0}^{\textup{diam}(T)/2} 
\sqrt{\calH\bracket{\delta, \calT}}d\delta, 
\]
where $\calH\bracket{t, \calT}$ denotes the $\delta$-covering 
entropy with pseudo-metric $\dist(\cdot, \cdot)$, and $\textup{diam}(\cdot) \defequal \sup\dist(u, u_0)$ is the diameter of the 
set $\calT$. 
\end{theorem}

\begin{lemma}[Lemma~$11$ in \cite{liu2009nonparanormal}]
\label{lemma:gauss_cdf_gradient}
Denote the distribution function and density 
function of the standard normal RV as 
$\Phi(\cdot)$ and $\phi(\cdot)$, respectively. 
We have 
\[
\bracket{\Phi^{-1}}^{'}(\eta) = \
\frac{1}{\phi\bracket{\Phi^{-1}(\eta)}}. 
\]
Furthermore, we have 
\[
\Phi^{-1}(\eta) \leq \sqrt{2\log\frac{1}{1-\eta}}, 
\]
for $\eta \geq 0.99$.	
\end{lemma}

\begin{theorem}[Wick's theorem, Thm.~$1.28$ (P$11$) in \cite{janson1997gaussian}]
\label{thm:wick}
Considering
the centered jointly normal variables
$\xi_1, \xi_2, \cdots, \xi_n$, we 
conclude 
\[
\Expc\bracket{\xi_1 \xi_2 \cdots \xi_n}
= \sum_{\substack{\textup{all possible disjoint}\\ \textup{pairs } (i_k, j_k)
\textup{ of } [n] }}\prod_{k}\Expc\bracket{\xi_{i_k}\xi_{j_k}}.
\] 
\end{theorem}
Notice that the variables $\set{\xi_i}_{1\leq i \leq n}$
are not necessarily different nor independent. To illustrate 
this theorem, we consider two special cases. 
First we let $\xi_1 = \xi_2 = \xi_3 = \xi_4 = \xi \sim \normdist(0, 1)$,
then we have 
\[
\Expc\xi^4 = \Expc\bracket{\xi_1\xi_2 \xi_3\xi_4}
= \Expc(\xi_1\xi_2)\Expc(\xi_3\xi_4) + \
\Expc(\xi_1\xi_3)\Expc(\xi_2\xi_4) + \
\Expc(\xi_1\xi_4)\Expc(\xi_2\xi_3) = 
3.
\]
Second we consider the case where $n$ is odd. 
Since we cannot partition $\set{1, 2, \cdots, n}$ into disjoint 
pairs $(\xi_{i_k}, \xi_{j_k})$,
we always have $\Expc(\xi_1 \cdots \xi_n) = 0$.

\end{document}